\documentclass[11pt]{article}
\pdfoutput=1
\usepackage{authblk}

\def\pb{\pagebreak}  

\usepackage[dvipsnames]{xcolor}

\def\beq{\begin{equation} }\def\eeq{\end{equation} }\def\ep{\varepsilon}\def\1{\mathbf{1}}

\usepackage[framemethod=default]{mdframed}
\usepackage{caption}
\usepackage{indentfirst}
\usepackage{bm, mathrsfs, graphics,float,amssymb,amsmath,subeqnarray,setspace,graphicx,amsthm,epstopdf,subfigure, enumerate, color}
\usepackage[utf8]{inputenc}
\usepackage[colorlinks,
            linkcolor=red,
            anchorcolor=blue,
            citecolor=blue
            ]{hyperref}
\usepackage{natbib}

\usepackage{fullpage}
\numberwithin{equation}{section}

\newtheorem{lemma}{Lemma}
\newtheorem{theorem}{Theorem}
\newtheorem{proposition}{Proposition}
\newtheorem{definition}{Definition}
\newtheorem{corollary}[theorem]{Corollary}
\newtheorem{remark}{Remark}

\ifx\assumption\undefined
\newtheorem{assumption}{Assumption}
\fi

\newcommand{\EE}{\mathbb{E}}
\newcommand{\RR}{\mathbb{R}}

\newcommand{\argmin}{\mathop{\mathrm{argmin}}}

\newcommand{\minimize}{\mathop{\mathrm{minimize}}}

\usepackage{multirow}
\usepackage{tablefootnote}
\usepackage{colortbl}
\usepackage{hhline}

\usepackage{algorithm}
\usepackage{algorithmic}

\def\ep{\varepsilon}

\newcommand{\x}{{x}}

\newcommand{\E}{\mathbb{E}}

\newcommand{\hf}{\hat{f}}

\renewcommand{\u}{{u}}

\newcommand{\<}{\left\langle}
\renewcommand{\>}{\right\rangle}

\usepackage{mathtools}

\newcommand{\hp}{{\hat{p}}}

\begin{document}
\title{Over Parameterized Two-level Neural Networks Can Learn Near Optimal Feature Representations}

\author[$*$]{Cong Fang}

\affil[$*$]{
Shenzhen Research Institute of Big Data\thanks{This work was done when Cong Fang was an intern at Shenzhen Research Institute of Big Data. The author is now with Princeton University.}}

\author[$\dag$]{
Hanze Dong
}

\affil[$\dag$]{Department of Mathematics\\ HKUST}

\author[$\ddag$]{
Tong Zhang
}
\affil[$\ddag$]{Department of Computer Science and Mathematics\\ HKUST}

\date{}
\maketitle
\begin{abstract}
Recently, over-parameterized neural networks  have been extensively analyzed in the literature \citep{zhang2016understanding,MeiE7665, du2018gradient}. However, the previous studies cannot satisfactorily explain why fully trained neural networks are successful in practice. In this paper, we present a new theoretical framework for analyzing over-parameterized neural networks which we call \emph{neural feature repopulation}. Our analysis can satisfactorily explain the empirical success of two level neural networks that are trained by standard learning algorithms. Our key theoretical result is that in the limit of infinite number of hidden neurons, over-parameterized two-level neural networks  trained via the standard (noisy) gradient descent learns a well-defined feature distribution (population), and the limiting feature distribution is nearly optimal for the underlying learning task under certain conditions. Empirical studies confirm that predictions of our theory are consistent with the results observed in real practice.
\end{abstract}


\section{Introduction}

In classical machine learning, features are often created by domain experts who are familiar with the problem. An example is the sift feature \citep{lowe1999object}, which is a handcrafted feature specifically engineered by human experts for image recognition. However, the situation has changed with deep learning. In particular, it is well-known by practitioners that neural networks (NNs) can learn features from data that are better than what can be  constructed by human experts, when a large amount of data are given to the network \citep{krizhevsky2012imagenet}. 
Specifically, NNs  are composed of (possibly multiple) feature representation layer(s), and a final linear learner. Compared with other machine learning techniques, the feature engineering process for NN is  performed automatically:    features are learned, together with the final linear classifier.
The empirical success of NNs is argued to be largely due to their ability to learn high quality features from the data \citep{lecun2015deep} and there are some attempts that visualize the convolutional kernels to understand the learning process \citep{zeiler2014visualizing}.

However, NNs are difficult to analyze theoretically. Part of the reason is due to the general perception that they are highly nonconvex learning models, which makes them difficult to analyze. Recently, significant progress has been made towards better understanding of the so-called overparameterized NNs. It is observed empirically that overparameterized NNs (that is, wide NNs with many hidden nodes at each layer) are easy to learn \citep{zhang2016understanding}. It was noted that when we consider NNs with an infinite number of neurons, then the network becomes more and more convex, and thus can be analyzed theoretically in the limit \citep{lee1996efficient,bengio2006convex, MeiE7665, chizat2018global, du2018gradient,allen2019convergence}. However, none of the existing results have addressed the question of how does NNs learn feature representations, which is a crucial reason on why they are so successful empirically. In fact, none of the current theory for NNs can fully explain the empirical success of NNs. 

In this paper, we address the open problem by answering the following fundamental questions for two-level NNs: 
\emph{How does  an overparameterized NN learn the features?  And why are the learned features  useful for classification?}

Traditionally, these questions have not been answered satisfactorily. Some researchers took a dim view of revealing the truth of this mystery under the belief that the learning procedure of NN is associated with a highly non-convex optimization objective. Therefore for such problems standard algorithms such as stochastic gradient descent (SGD) are susceptible to  local minima.

This paper solves this problem satisfactorily by developing a new theoretical analysis for over parameterized NNs. We call our new view \emph{neural feature repopulation}. 
Through this new point of view of overparameterization, we prove that (overparametrized or wide) two level NNs can learn features that are nearly optimal in terms of efficient representation via a repopulation process. Specifically, we show that in the limit of infinite number of neurons, two level NNs behave like a convex model that can learn a near optimal population over features using (noisy) Gradient Descent (GD). We also demonstrate that the predictions of our theory match with the observed behavior of NNs in real practice. 

Therefore, the main contribution of this work is the following
\begin{itemize}
\item We propose a new framework for analyzing overparameterized NN called  neural feature repopulation. It shows that NN can learn near optimal population over features under certain conditions. 
An important differentiation of our formulation from earlier work is the separation of linear model and feature learning in our analysis, and the explicit role of regularization in characterizing feature learning. These can be generalized to analyze deep neural networks. 
\item Our theory matches empirical findings, and hence this is, to the best of our knowledge,  the first theory that can satisfactory explain the success of two-level NNs. 
\end{itemize}

\noindent\textbf{Sketch.} The paper is organized as follows. Section~\ref{sec:related} compare our analysis to related work. Section~\ref{sec:contNN} describes the equations for the continuous limit of infinite neurons. Section~\ref{sec:discNN} describes the relationship between continuous formulation and the standard discrete NNs. 
Section~\ref{sec:analysis} presents our analysis of overparameterized NN, and consequences of the neural feature repopulation framework. 
Section~\ref{sec:repop} illustrates an empirical consequence of our theory called ``Feature Repopulation'', which we can study via experiments. 
Section~\ref{sec:empirical} presents our empirical studies to show that the predictions from our theory can be observed in real practice. Therefore our method can fully explain the real behavior of two-level NNs. 
Finally, concluding remarks are given in Section~\ref{sec:conclusion}.  The proofs and additional experiments is shown in Appendix.

For clarity, our paper mainly focuses on presenting the new framework in the limit of infinite neurons.   
Some more elaborate analysis, such as the uniform approximation complexity using finite number of hidden neurons,  have not been fully studied in this work, but can be  carefully studied in the near feature.\\

\noindent\textbf{Notations.} Let $\| \cdot\|$ denote the Euclidean norm of a vector or spectral norm of a square matrix. For function $f(\theta'): \RR^{d+1} \to \RR$ and $i\in [1, \cdots, d+1]$, we use $\nabla_{\theta'_i} f(\theta')$ and $\nabla^2_{\theta'_i} f(\theta')$ to denote  the first-order and  second-order derivative on the $i$-th coordinate, respectively. Further, we use  $\nabla f$, $\nabla \cdot f$, and $\nabla^2f$  to denote the gradient,  divergence, and  Laplace operator for $f$, respectively. Similarity, let $\theta = [\theta'_1, \cdots, \theta'_{d} ]$ and $u=\theta'_{d+1}$,  we define $\nabla_{\theta} f(\theta') = [\nabla_{\theta'_1} f(\theta'), \cdots,\nabla_{\theta'_d} f(\theta') ]$ and $\nabla^2_{\theta}f(\theta') = \sum_{i=1}^d \nabla_{\theta'_i}^2 f(\theta') $. Besides, for set $S$, we let $\dot{S}$ be the interior point of $S$ and $\bar{S}$ be the closure of $S$, and denote $S^c$ to be the  complementary set of $S$.



\section{Related Works}
\label{sec:related}

Due to the empirical success of deep NNs, there have been many attempts to theoretically understand these methods. A major difficult is the nonconvexity of NNs, which means that standard training algorithms such as SGD can lead to suboptimal solutions. 

A number of earlier works \citep{hardt2016identity, freeman2016topology, safran2017spurious,nguyen2017loss,zhou2017critical,du2018power} that  studied the theoretical properties of NN  were based on specific models, and these works  attempted to characterize the geometric landscapes  of the learning objective functions. The result, in general, shows that for any NN that satisfies the so-called \emph{strict saddle property} \citep{ge2015escaping},  noisy gradient descent (NGD) can efficiently find a local minimal \citep{jin2017escape,fang2018spider,fang2019sharp}.

A more modern treatment of NNs is based on the empirical observation that wide NNs with many  neurons are easier to optimize due to their redundancy. Such NNs are called overparameterized, and a number of recent works showed that in the continuous limit, such networks can behave like a convex system under appropriate conditions. This allows desirable theoretical properties to be obtained for such problems. 

This work continues along this line of research. Before introducing our own framework,  we review the two existing point of views for dealing with overparameterized NNs that we believe are most related to this work. In both views, we consider the continuous limit of NNs when the overparameterization goes to infinity.

\subsection{Neural Tangent Kernel View} 
With a specialized scaling and random initialization,  it was shown that it is sufficient to consider NN parameters (in the continuous limit) in an infinitesimal region around the initial value. The NN with parameters restricted in this region can be regarded as a linear model with infinite dimensional feature \citep{jacot2018neural,li2018learning,du2018gradient,arora2019fine,du2019gradient,allen2018learning,allen2019can,allen2019convergence,du2018gradient,zou2018stochastic,oymak2018overparameterized}. It induces a kernel referred to as \emph{Neural Tangent Kernel} (NTK) \citep{jacot2018neural}. Since the system becomes linear, the dynamics of GD with this region can be tracked via  properties of the associated NTK. There are a series of these studies,  in which  they  remarkably prove  polynomial convergence rates to the global optimal \citep{du2018gradient, allen2019convergence}  and  sharper generalization errors with early stopping \citep{arora2019fine,allen2019can}, so that the parameters are confined to an infinitesimal region. It's also worth mentioning that the similar results are  available for deep  neural nets with  complex  topological structures, such as  Recurrent NN \citep{allen2019can},  Residual NN \citep{du2018gradient}.  

While mathematically beautiful, we argue that this point of view is not a correct mathematical model for NNs. In fact, as we have pointed out in the introduction, it is well known empirically that the success of NNs is largely due to the ability of such models to learn useful feature representations. However, the NTK view essentially corresponds to a linear model on the infinite dimensional random feature that defines the underlying NTK. Such features do not need to be learned from the data, and thus such a theory for NNs fail to explain the ability of NNs to learn feature representations. Therefore the consequence of this theory is not supported by empirical findings. 

\subsection{Mean-Field View} 
Recently, a different line of papers \citep{MeiE7665,chizat2018global,sirignano2019mean,rotskoff2018neural,ma2019comparative,mei2019mean, dou2019training, wei2018margin} applied the mean filed view to study the over-parameterized two level NNs. The key idea is to describe the evolution of the distribution of overparameterized NN parameters  as a  Wasserstein gradient flow. This leads to differential equations that can be analyzed mathematically. The representative works are from \cite{MeiE7665} and \cite{chizat2018global}, which showed the convergence of (noisy) GD in the continuous limit to the global optimal solution of the equations when time goes to infinity.  \cite{MeiE7665} have also considered approximate error with finite NN.  A number of extensions have been proposed later on.  For example,  \cite{wei2018margin} showed polynomial convergence rate under certain conditions. \cite{mei2019mean,ma2019comparative} build the relations between  the mean-field view and kernel view.   \cite{ma2019analysis}  studied NNs with  skip connection architecture. 

The advantage of the Mean-Field View over the kernel view is that it considers the full solution of NN until convergence rather than in the tangent space only. This is a more realistic characterization of the NN training process in practice. 
However, similar to the kernel view, the Mean-Field View does not illustrate how NNs learn useful feature representations, and thus cannot be regarded as a full theory. There are also difficulties in generalizing the analysis to deep NNs.

\subsection{Feature Repopulation View}

This paper proposes the Feature Repopulation View of NN learning, which is inspired by the mean field view. We address its shortcomings as follows. We specifically consider feature learning to full explain the fact that NN can learn effective feature representations. Our analysis separately considers the top layer linear model and bottom layer feature learning, and explicitly characterizes the effect of regularization in feature learning. 
While this paper focuses on the analysis of two-level NNs, the theoretical framework of feature repopulation can be directly extended to \emph{deep} NNs (which is not directly possible in the mean field view of NN learning). We will leave the analysis to a subsequent paper. 

Our analysis is the \emph{first} work showing that NN can learn useful features. As we show in our empirical studies, the predictions from this theory can be verified in practice. Therefore this is the first analysis that fully explains the empirical effectiveness of NN learning for two-level NNs. 

The key idea of our view focuses on the ability of NN to learn effective feature representations. For two level NNs, the hidden layer is regarded as features, and the top layer is regarded as a linear classifier using such features.
It is known that universal approximation can be achieved by two level NNs even with random features \citep{cybenko1989approximation}, as long as the number of such random features approach infinity. 
In fact, random features have been successfully employed in machine learning, and the resulting method is often referred to as \emph{Random Kitchen Sinks} \citep{rahimi2009weighted}. 
Although random features can be used, in the feature repopulation view, we argue that NNs can learn a significant better distribution (population) of features which we refer to as ``repopulated features'' which is the result of training. Because of the better features, two level NNs, when fully training, can represent the target function with significantly fewer neurons sampled from the repopulated features (compared to the initial random features). Therefore the feature repopulation view directly shows that fully trained NNs are superior to Random Kitchen Sinks in terms of the feature effectiveness. 

More specifically, in our framework, neural networks are equivalent to linear classifiers with learned features optimized for the learning task. The feature learning process is characterized by regularization, and the learned features can be represented by distributions over hidden nodes. Discrete neural networks can be obtained by sampling hidden nodes from the learned feature distributions. The efficiency of such sampling measures the quality of learned feature representations.
We show that the entire feature learning process is a convex optimization problem with a suitable change of variables. By incorporating proper regularization terms, the objective function is strictly convex, and hence guarantees  a single global optimal solution. Moreover, (noisy) GD converges to this solution. 
Moreover, with weak regularization, two-level NN can learn near optimal feature representation measured by sample efficiency to achieve a certain function approximation error. As a side product, our view can also satisfactorily explain the effectiveness of batch normalization \citep{ioffe15}, which has been regarded as a mysterious trick in the NN literature. 

It was worth noting that the neural tangent kernel view of NN learning can be regarded as a generalization of Random Kitchen Sinks with more sophisticated random features. In particular, the features are not learned. In comparison, the NN repopulation analysis of this paper shows the superiority of features learned by NNs over random features. Our view is more consistent with empirical findings that NN do learn nonrandom useful features, which will also be illustrated in our experiments. We believe NTK is not the correct theory of NN because it is inconsistent with empirical evidents.



\section{Continuous NN}\label{sec:contNN}
Consider the traditional classification task: given input-out data $(x_i,y_i) \in \RR^d \times \{1,-1\}$, drawn i.i.d. from an unknown underlying distribution $D$, our goal  is to find an $f: \RR^d\to \RR$ that minimizes the objective function as follows:
\begin{equation}\label{problem}
\min_{f} Q(f) = J(f)+ R(f),
\end{equation}
and
$$J(f) = \EE_{x,y} \phi(f(x), y), $$
where $\phi(y',y)$ is a loss function which we assume is convex with respect to $y'$ and $R$ is a regularizer that is used to avoid  ill-condition  and control the complexity of the model.  For  empirical risk minimization tasks,  $(x_i,y_i)$ is finite sampled and $D$ follows an uniformly discrete distribution.  

Similar to Kernel methods,  we consider the two-level architecture given below to represent $f$:
\begin{eqnarray}\label{ff}
f(\omega, \rho, x) = \int h'(\theta, x) \omega(\theta) \rho(\theta) d \theta,
\end{eqnarray}
where $h'(\theta,x):\RR^d \times \RR^d \to \RR$ is a known real-valued function, $\omega (\theta): \RR^d \to \RR$ is a real value function of $\theta$, and $\rho (\theta)$ is a probability density over $\theta$.   

For the regularizer, we let:
\begin{eqnarray}
R(\omega, \rho) = \lambda_1 R_1(\omega, \rho) +\lambda_2 R_2(\rho),
\end{eqnarray}
where 
$$ R_1(\omega, \rho) = \int  r_1(\omega(\theta))\rho(\theta)d \theta $$ 
and 
$$ R_2 (\rho) =\int   r_2(\theta)\rho(\theta)d \theta. $$
For the sake of simplicity,  we only consider $\ell_2$ norm regularizer,  i.e. $r_1(\omega) = |\omega |^2$ and $r_2(\theta) = \|\theta \|^2$ in this paper. Similar analysis for $\ell_p$ norm with $p>1$ can also be performed using our technique.  Given an instance $x$,  let $h(\cdot, x)$ be a non-linear mapping which maps the $x$ to the feature space,  and  $\rho(\cdot)$  is the sampling distribution for the features. In the Random Kitchen Sinks \citep{rahimi2009weighted}, commonly we have $h'(\cdot, x) = h(\cdot, x)$ and $\rho(\cdot)$ follows from a random distribution, e.g. Gaussian distribution. In the NN,  $\rho(\cdot)$ is learned jointly with the weight $w(\cdot)$. And  typical example for $h(\cdot,\cdot)$ in the NN  is $h(\theta,x) = \sigma(\theta\cdot x)$, where $\sigma(\cdot)$ is the activation function, e.g.  sigmoid,  tanh, relu.  We can also use vector valued functions for $h(\cdot, \cdot)$.  We study two cases which guarantees the NN  can learn optimal  feature representation:
\begin{enumerate}[\text{Case} 1.]
\item  Work on the original feature $h'(\theta, x) = h(\theta, x)$ when $h(\theta,x)$ is bounded. 
\item Use the normalized feature
\begin{eqnarray}\label{batch norm}
h'(\theta, x) = \frac{h(\theta, x) }{\sqrt{\EE_x  [h(\theta, \x)]^2 }}.
\end{eqnarray}
The  procedure for normalization can be implemented by adding a normalization layer back to the hidden layer. Similar technique has   been considered by the Batch Normalization \citep{ioffe15}  which has been regarded as a mysterious engineer trick in the related literature.
\end{enumerate}

In the following, we show a discrete NN approximates the continuous one when $m\to \infty$ and then drive the evolution rule of $\rho(\theta)$ and $\omega(\theta)$ from the  (noisy) GD algorithm when the step size $\Delta t \to 0$.

\section{Discrete NN}\label{sec:discNN}
\begin{algorithm}[tb]
	\caption{(Noisy) Gradient Descent for Training a Two-level NN:\\ Input the data $\{x_i,y_i\}_{i=1}^n$ and the  step size $\Delta t$. }
	\label{algo:NGD}
		\begin{algorithmic}[1]
	\STATE Initialize the weights $(u_0,\theta_0 )$ {\hfill $\diamond$ Usually $(u_0,\theta_0 )$ follows from a random distribution }
	\STATE {\bf for} $t =0,1,\dots, T-1$ {\bf do}
	\STATE \quad
	Perform forward propagation to compute $\hat{ Q}(u_t, \theta_t)$
	\STATE \quad
		Perform backward propagation to compute $(\nabla_{u} \hat{ Q}(u_t, \theta_t), \nabla_{\theta} \hat{ Q}(u_t, \theta_t))$
	 \STATE \quad
	  Draw an Gaussian noise $\zeta \sim N(0,  \sqrt{2\Delta  t }I_{m})$
	   \STATE \quad
	   $u_{t+1}\leftarrow u_{t} - \Delta t \nabla_{u} \hat{ Q}(u_t, \theta_t) - \sqrt{\lambda_3}\zeta$ {\hfill $\diamond$ (noisy) Gradient Descent }
	  \STATE \quad
	  Draw an Gaussian noise $\xi \sim N(0,  \sqrt{2\Delta  t }I_{md})$
	  	 \STATE \quad
	   $\theta_{t+1}\leftarrow \theta_{t} -\Delta t \nabla_{\theta} \hat{ Q}(u_t, \theta_t)- \sqrt{\lambda_3}\xi  $  {\hfill $\diamond$ (noisy) Gradient Descent }
	  \STATE {\bf end}
	  \STATE Output the weights $(u_{T}, \theta_T)$
	\end{algorithmic}
\end{algorithm}
Suppose we construct a discrete NN of the form:
\begin{eqnarray}\label{NN form}
\hf(u, \theta, x) = \frac{1}{m}\sum_{j=1}^m u^j h'(\theta^j,x)
\label{eq:discrete}
\end{eqnarray}
that would approximate $f(\omega, \rho, x)$, where $\theta^j \in \RR^d$ and $u^j \in \RR$ for all $j\in [m]$. And we have  the regular terms as 
\begin{eqnarray}
\hat{R}_1(u, \theta) = \frac{1}{m}\sum^{m}_{j=1} r_1(u^j), \quad \hat{R}_2( \theta) = \frac{1}{m}\sum^{m}_{j=1}r_2(\theta^j).
\end{eqnarray}
 Consider training  the constructed NN with the objective denoted as
 \begin{eqnarray}
 \hat{Q}(u, \theta) = \EE_{x,y} \phi(\hat{f}(u, \theta,x), y) + \lambda_1 \hat{R}_1(u, \theta) + \lambda_2 \hat{R}_2(\theta) 
 \end{eqnarray}
 and solving it by   the standard (noisy) GD, which is described  as follows:
\begin{enumerate}[D 1.]
\item \label{discrete} Initially, we sample $m$ hidden nodes $\theta_0^j$ with $j\in [m]$ from $\rho_0(\theta)$,  and then  update  $\theta_j^t$ for all $t\geq 0$  and $j\in [m]$ by the (noisy) GD, i.e.
\begin{eqnarray}
 \theta^j_{t+1}\leftarrow \theta^j_{t} -\Delta t \nabla_{\theta^j}[\hat{Q}(u_t,\theta_t)] -\sqrt{\lambda_3} \xi_{t+1}^j,
\end{eqnarray}
where $\Delta t$ is the step size and $\xi_{t+1}^j\sim N(0, \sqrt{2\Delta t  } I_d) $.
\item \label{discrete2} For each node $j$, with hidden layer $h'(\theta^j, x)$, we sample weight $u_0^j$ from $p_0(u| \theta^j)$, where  $p_0(u| \theta)$ is a pre-defined distribution.   Then we  update  $u_t^j$ for all $t\geq 0$  and $j\in [m]$ by the  (noisy) GD,  i.e.
\begin{eqnarray}
u_{t+1}^j\leftarrow u_{t}^j  -\Delta t \nabla_{u^j}[ \hat{Q}(u_t,\theta_t)] - \sqrt{\lambda_3} \zeta^j_{t+1} ,
\end{eqnarray}
where $\zeta^j_{t+1}\sim N(0, \sqrt{2\Delta t })$.
\end{enumerate}
 The whole algorithm  is shown in Algorithm \ref{algo:NGD}.  We first consider the case when $\lambda_3=0$.  Algorithm \ref{algo:NGD}  degenerates to the standard GD Algorithm.We assume that $p_0(u|\theta) = \delta(u= \omega(\theta))$, where $\delta(\cdot)$ is the Dirac delta function.  We drive the evolutionary dynamics for $\rho_t$ and $\omega_t$  and show that the dynamics is consistent with GD  of  the aforementioned discrete NN in the limit case when $\Delta t \to 0$ and $m\to0$. Concretely, by transferring the change of $\theta_t$ to the change of its distribution, we have the  lemma below.
 \begin{lemma}[GD dynamics, Informal]\label{approximate}
 Suppose at time $t\geq 0 $, we have  $\theta^j_t \sim \rho_t$, and let   $u_j^t=\omega_t(\theta^j_t)$ with $j\in [m]$, i.i.d.,  then for any $x$, we have
 \begin{eqnarray}\label{approximate1}
 \lim_{m\to \infty}  \hf(u_t, \theta_t, x)  = f(\omega_t, \rho_t, x).
 \end{eqnarray}
Furthermore,  if $\theta_j^t$ and $u_j^t$ is updated as  D \ref{discrete} and D \ref{discrete2}, respectively.  Let $\Delta t\to 0$ and $m \to \infty$, we have
 \begin{eqnarray}\label{tt1}
\frac{d \rho_t(\theta) }{dt} = - \nabla_\theta\cdot[\rho_t(\theta)g_2(t, \theta, \omega_t(\theta))],
\end{eqnarray}
and 
\begin{eqnarray}\label{tt2}
\frac{d\omega_t(\theta)}{dt} =g_1(t, \theta, \omega_t(\theta)) - \nabla_\theta[\omega_t(\theta)] \cdot g_2(t, \theta, \omega_t(\theta)),
\end{eqnarray}
where $g_1(t,\theta,u)$ and $g_2(t, \theta,u)$  satisfy
\begin{eqnarray}\label{g1}
g_1(t, \theta, u) =  - \EE_{x,y}[\nabla_f\phi(f(\omega_t, \rho_t,x),y)h'(\theta, x)] -\lambda_1  \nabla_u [r_1 (u)]
\end{eqnarray}
and 
\begin{eqnarray}\label{g2}
g_2(t, \theta, u) = - \EE_{x,y}[\nabla_f\phi(f(\omega_t, \rho_t,x),y)u \nabla_\theta h'(\theta, x)] -\lambda_2 \nabla_{\theta} [r_2( \theta)].
\end{eqnarray}
 \end{lemma}

\begin{remark}
For simplicity, Lemma \ref{approximate} and Lemma \ref{lemma:evo} are not presented under the uniform convergence statement. That is, showing GD with a proper step size approximates the partial differential equations (\eqref{tt1} and \eqref{tt2}) for the whole evolutionary process ( $t\geq0$). And  the neurons are i.i.d. sampled only when $t=0$.   A rigorous proof  of uniform convergence requires additional efforts, such as using the so called ``propagation of chaos'' technique. We mention that this careful analysis has already been proposed by  some recent works,  e.g. \cite{chizat2018global} propose a proof for GD in  asymptotic sense ($m\to \infty$). \cite{MeiE7665,mei2019mean} achieve a non-asymptotic  result ($m< \infty$)  under some bounded conditions.  Because our work mainly focuses the analysis on the continuous limit,  we only provide a non-rigorous proof in Appendix \ref{proof Discrete NN}, which provides the readers the basic intuitions for better understanding.  One can  refer \cite{chizat2018global,MeiE7665,mei2019mean} for more details. 
\end{remark}

Lemma \ref{approximate} implies that the evolution of $\rho_t$ and $\omega_t$ is consistent with GD for discrete NN.  Because $\rho_t$ is the  distribution of the learned features,  our  goal  is to argue that $\rho_t$ will converge to a certain distribution useful for the underlying learning task. When $m\to \infty$ and $\Delta t \to 0$, \cite{chizat2018global} proved that GD converges to the global solution if  $h(\theta,x)$ is $\ell$-homogeneous ($\ell\geq1$) on $\theta$\footnote{$f(\theta)$ is said to be $\ell$-homogeneous if for any $c\geq 0$, we have $f(c\theta) = c^\ell f(\theta)$. }.  For  general  active functions, we can  apply noisy Gradient Descent (NGD) ($\lambda_3>0$).  In this scenario,  we consider  the joint distribution of $(\theta_t, \omega_t)$  which is denoted as  $p_t(\theta, u)$.  When $t=0$, we set  $p_0(\theta, u) =  p_0(\theta)p_0(u|\theta)$. And when $t\geq0$,  the following lemma describes the evolutionary dynamics for  $p_t$: 
\begin{lemma}[NGD dynamics, Informal]\label{lemma:evo}
 Suppose at time $t\geq0$, we have  $[\theta^j_t,u^j_t] \sim p_t(\theta, \omega)$with $j\in [m]$, i.i.d., then for any $x$, we have
 \begin{eqnarray}
 \lim_{m\to \infty}  \hf(u_t, \theta_t, x)  = f(\omega_t, \rho_t, x)
 \end{eqnarray}
Furthermore,  if $\theta_j^t$ and $u_j^t$ is updated as  D \ref{discrete} and D \ref{discrete2}, respectively,  
let $\Delta t\to 0$ and $m \to \infty$, we have
\begin{eqnarray}\label{evo p}
   \frac{d p_t(\theta,u) }{dt} = - \nabla_\theta\cdot[p_t(\theta, u)g_2'(t, \theta,u)] -  \nabla_\u[p_t(\theta, u)g_1'(t, \theta,u)]+ \lambda_3 \nabla^2[p_t(\theta,u)],
\end{eqnarray}
where $\rho_t(\theta) =  \int_\RR p_t(\theta, u)du$ and $\omega_t(\theta) = \EE_{[\theta, u]\sim p_t} [u|\theta ] = \frac{\int_\RR u p_t(\theta, u)du }{\rho_t(\theta)}$\footnote{\eqref{evo p} ensures $\rho_t>0$ for all $t\geq0$. Please refer Lemma \ref{p great 0} in Appendix \ref{app: con}.}. 
\end{lemma}
Lemma \ref{lemma:evo}  implies that $p_t(\theta, u)$ becomes a diffusion process due to the injection of the random noise.  And we will show in Section \ref{sec:analysis} that the evolution for $p_t(\theta, u)$ can be regarded as optimizing the NN loss with additional regularizers.   At the same time, we can also write down the evolution of $\rho_t$ and $\omega_t$ as follows.
\begin{lemma}[Evolution of $\rho_t$ and $\omega_t$]\label{evo rho}
Suppose $p_t(\theta, u)$ evolves according to \eqref{evo p}, then 
\begin{eqnarray}\label{tt3}
\frac{d \rho_t(\theta)}{dt} = - \nabla_\theta\cdot[\rho_t(\theta)g_2(t, \theta, \omega_t(\theta))]+ \lambda_3 \nabla^2 \rho_t(\theta),
\end{eqnarray}
and 
\begin{eqnarray}\label{tt4}
\frac{d\omega_t(\theta)}{dt} &=&g_1(t, \theta, \omega_t(\theta)) - \nabla_\theta(\omega_t(\theta)) \cdot g_2(t, \theta, \omega_t(\theta))  + \lambda_3 \nabla_\theta^2[\omega_t(\theta)] + \frac{2\lambda_3}{\rho_t(\theta)}[\nabla\rho_t(\theta)]\cdot [\nabla_{\theta}\omega_t(\theta)]\notag\\
&&-\frac{1}{\rho_t(\theta)}\nabla_{\theta}\cdot\left(\int_\RR p(\theta, u)u [g_2(t, \theta, u)-g_2(t, \theta, \omega_t(\theta))]d u \right).
\end{eqnarray}
\end{lemma}
We can verify that for GD ($\lambda_3=0$),  $p_t(\theta, \omega)$  can be represented as
$p_t(\theta, \omega)= \rho_t(\theta)  \delta (u = \omega_t(\theta) )$ when $t=0$ it holds. In this case, \eqref{tt3} and \eqref{tt4}  degenerate to \eqref{tt1} and \eqref{tt2}, respectively.

\section{Analysis for  Continuous NN}\label{sec:analysis}
In this section, we give our analysis of continuous NN and show that in the continuous limit, the feature distribution can be nearly optimal in terms of sample efficiency for the hidden units after training.   We directly consider optimizing the  NN by  NGD since it is more general.  Similar result for GD with activation functions  satisfying certain homogeneous properties can also be derived using the technique of \cite{chizat2018global}. We  note that our main claim that NN learns near optimal feature distribution is  independent of the underlying optimization technique. Because the objective is strictly convex,  all algorithms which ensure the global optimality  converge to the same solution,  and the final $\rho$ is near optimal for the learning task.

  Consider the objective function  \eqref{problem} with an extra entropy term\footnote{According to the definition, $\int_{\RR^{d+1}} p\mathrm{ln} (p) d\theta d u< +\infty$ only when $p$ is absolutely continuous with respect to the Lebesgue measure, which indicates that $p$ has a probability density function by the  Radon–Nikodym theorem \citep{billingsley2008probability}.} as a regularizer:
\begin{eqnarray}\label{problem2}
\min_{p\in \mathcal{P}^{d+1}} && Q'(p) = J'(p)+ \int_{\RR^{d+1}} \left(\frac{\lambda_1}{2}|u|^2+\frac{\lambda_2}{2}\|\theta\|^2 \right)pd\theta du + \lambda_3\int_{\RR^{d+1}} p\mathrm{ln} (p)d\theta du,\end{eqnarray}
where 
\[
 J'(p) =\EE_{x,y}\phi\left(\int h'(\theta,x)u p(\theta,u)d\theta du  ,y\right)  =\EE_{x,y}\phi\left(\int h'(\theta,x)\omega(\theta) \rho(\theta)d\theta  ,y\right) =J(\omega, \rho).  
\]
The road map is described as follows:  we establish our criterion for measuring the efficiency of feature representation in Section \ref{criterion}. Section \ref{assumptions}  proposes our assumptions for theoretical analysis. In Section \ref{sub convergece}.  we show that the optimization problem \eqref{problem2} is a strictly convex functional with respect to $p$, thus unique minimal solution $p_*$ (in the sense of a.e) is guaranteed. We then give a more general proof showing that NGD converges to the optimal solution. Finally,  Section \ref{sub convergece} reveals the fact that NN learns  optimal features under our criterion when imposing proper hyper-parameters for the regularizers.

\subsection{Optimal Criterion for Continuous Feature Representation}\label{criterion}

Motivated by  ``propagation of chaos'' \citep{mckean1967propagation,sznitman1991topics}, which suggests that the hidden nodes of a two-level discrete NN  can be regarded as independent samples from the feature population of the corresponding continuous NN, in the feature repopulation framework, we explicitly investigate the error caused by discrete sampling. In fact, the smaller the sampling error is, the more effective the underlying feature representation is because one can use a more compact discrete NN with fewer nodes to approximate the underlying target function. More specifically, consider a target function $f(x)$ and
  feature population $\rho(\theta)$ for the continuous NN, where
  \[
  f(x)= \int \ h'(\theta,x) \omega(\theta) \rho(\theta) d\theta .
  \]
  Let $\hat{f}$ be a discrete NN of \eqref{eq:discrete}, with $m$ nodes independently sampled from $\rho(\theta)$, then the mean squared sampling error is given by 
  \begin{equation}
    \E_{\{\theta^j\}_1^m}  \|\hat{f}(x) - f(x)\|_2^2
    = \frac1m \int \|\omega(\theta) h'(\theta,x)-f(x)\|_2^2 \; \rho(\theta) d \theta . \label{eq:discrete-nn-var}
  \end{equation}
Since the feature repopulation view focuses on the ability of two level
NNs to learn effective feature representations, we
need to compare the effectiveness of different feature populations. It is natural to use \eqref{eq:discrete-nn-var} as our criterion. 
Formally, we shall introduce the following definition, which is a simplification (or upper bound) of \eqref{eq:discrete-nn-var}. 
Consider a feature population $\rho(\theta)$.
We can measure the effectiveness of $\rho$ as follows.
\begin{definition}\label{def optimal}
  Given a target function $f(x)$ and
  feature population $\rho(\theta)$.  We assume that the target function $f$  can be represented by a signed measure $\mu$ with  
  \[
  f(x)= \int \ h'(\theta,x) \mu(\theta) d\theta .
  \]
  Let 
    \begin{eqnarray}
\omega(\theta)= 
\begin{cases}
\frac{\mu(\theta)}{\rho(\theta)},&    \rho(\theta)\neq 0,\\
0,&   \rho(\theta)= 0, \\
\end{cases}
\end{eqnarray}

  The efficiency is measured by the variance
  \begin{equation}
  V(\mu, \rho) =  \int \|\omega(\theta)\|_2^2 \rho(\theta)
   d \theta .
  \label{eq:opt}
\end{equation}
  
\end{definition}

\begin{assumption}\label{ass:h1}
Assume that $h'(\theta,x)$ satisfies the following condition:
\begin{equation}\label{def bv}
  \forall \theta , \qquad \E_x [h'(\theta,x)]^2 \leq B_v^2.
\end{equation}

\end{assumption}
If the learned target function $f$ with weight $\omega$ and feature population $\rho$ can be represented
as \eqref{ff}. We consider a discrete function of \eqref{eq:discrete}
which is obtained by sampling $\theta^j$ from $\rho(\theta)$, with
$u^j= \omega(\theta^j)$. 
\[
\E_{\{\theta^j\}_1^m} \; \E_x \|\hat{f}(x)-f(x)\|^2 \leq \frac1m V(\mu,\rho) B_v^2 .
\]

It follows that if $V(\mu,\rho)$ is small, then $f$ can be efficiently represented by a discrete NN with a small number of neurons drawn from the feature population $\rho$. 

In this paper, we only consider the criterion $V(\mu, \rho)$, which measures the quality of feature population $\rho$ given a target representation $\mu$. It is also possible to consider other
optimality conditions. 
 
 \subsection{Assumptions}\label{assumptions}
This section presents several assumptions needed in our theoretical analysis. We first  present the following  assumptions for the loss function.
\begin{assumption}[Properties of the loss function $\phi$]\label{ass:h2} We assume:
\begin{enumerate}[i]
    \item  $\phi(y',y)$ is convex on $y'$.
    \item $\phi(y',y)$ is bounded below, i.e. for all $y'\in \RR$ and $y\in \{1,-1\}$,  $\phi(y',y)\geq B_l$.
    \item  $\phi(y',y)$ has  $L_1$-bounded  and $L_2$-Lipschitz continuous gradient on $y'$, i.e. for all $y\in \{1,-1\}$, $y_1\in \RR$ and $y_2 \in \RR$, we have $|\nabla_{y'}  \phi(y_1,y) |\leq L_1$ and $|\nabla_{y'}  \phi(y_1,y) - \nabla_{y'}  \phi(y_2,y)|\leq L_2 | y_1 -y_2|$. \label{ass:h23}
\end{enumerate}
\end{assumption}
The assumptions for the loss function $\phi$ are standard in the optimization literature \citep{du2018gradient, allen2018learning}.  We also mention that   Assumption \ref{ass:h2} (\ref{ass:h23}) can be further relaxed.   We only require   the smoothness  of $\phi(y', y)$  on the lower level set  $\{ y'|\phi(y',y) \leq C_l, y=\pm 1\}$ where  $ C_l< +\infty$  depends on  the initial value $Q'(p_0)$ and some problem-dependent constants.    The reason is that  $Q(p_t)$ is monotonously non-increasing (refer Lemma \ref{descent} in Appendix \ref{app: con}) and the regularizers can be rewritten as a Kullback–Leibler (KL) divergence plus a constant. That is
$$ R'(p) =  \int_{\RR^{d+1}} \left(\frac{\lambda_1}{2}|u|^2+\frac{\lambda_2}{2}\|\theta\|^2 +\lambda_3\mathrm{ln} (p(\theta,u)) \right)pd\theta du =  \lambda_3\int_{\RR^{d+1}} p(\theta,u)\mathrm{ln}\left(\frac{p(\theta,u)}{\hat{p}(\theta,u)}\right)d\theta du + C_{\hat{p}},$$
where  $C_{\hat{p}}= \frac{\lambda_3 d}{2}\mathrm{ln}(\frac{\lambda_3}{2\pi\lambda_2})+\frac{\lambda_3}{2}\mathrm{ln}(\frac{\lambda_3}{2\pi\lambda_1})$ is a constant and  $\hat{p}(\theta, u) =  \exp(C_{\hat{p}}/\lambda_3)\exp\left[ -\frac{\lambda_1|u|^2}{2\lambda_3}-\frac{\lambda_2\|\theta\|^2}{2\lambda_3}\right] $ is a Gaussian distribution . It follows that we can set  $C_l= Q'(p_0) + C_{\hat{p}}- B_l$ by noting that the  KL divergence is non-negative. Secondly, we propose the  assumptions for the feature activation function.
 \begin{assumption}[Properties of the feature activation function $h'$]\label{ass:h3} Under Assumption \ref{ass:h1}, we further assume:
\begin{enumerate}[i]
    \item  for all $\x$,  $h'(\theta,x)$ is second-order differentiable  on $\theta$.
    
    \item for all $\x$ and $\theta$, we assume  $|h'(\theta, x)|\leq C_1\|\theta \|+C_2$, 
    $\|\nabla_{\theta} h'(\theta,x)\|\leq C_3, $ and $| \nabla^2_{\theta} h'(\theta,x)|\leq C_3$.
\end{enumerate}
\end{assumption}
Assumption \ref{ass:h1} is satisfied with $B_v =1$ if $h'(\theta, x)$ is set as  \eqref{batch norm}. As for the smoothness conditions in Assumptions \ref{ass:h3}, they hold for  many feature functions, e.g. tanh, sigmoid, smoothed relu.

 \begin{assumption}[Initial Value] \label{ass:h4} 
 We assume $Q'(p_0)< +\infty$.
\end{assumption}
 Assumption \ref{ass:h4} holds for common distributions that have bounded second moments and are absolutely continuous with respect to the  Lebesgue measure.  A safe setting of  $p_0$ might be a standard Gaussian distribution.
 
\subsection{Convergence of GD}\label{sub convergece}
It is not hard to observe from \eqref{problem2} that the continuous NN learning is a  convex optimization problem in the infinite dimensional measure space. So by exploiting the convexity, we describe the properties for the solution of \eqref{problem2}  below:
\begin{proposition}[Global Optimal  Solution]\label{proposition: unique}
Suppose Assumption \ref{ass:h2} and \ref{ass:h3} hold,  $Q(p')$ is convex with respect to $p$ and has a unique optimal solution $p_*$, a.e.,  which satisfies:
\begin{eqnarray}\label{opt}
p_*(\theta,u) = \frac{\exp\left( - \frac{\lambda_1 }{2\lambda_3}|u |^2 - \frac{\lambda_2}{2\lambda_3}\| \theta\|^2 - \frac{1}{\lambda_3}\EE_{x,y}[\nabla_{y'}\phi'(f(\omega_*, \rho_*),y)h'(\theta,x)]u   \right)}{C_5},
\end{eqnarray}
where $C_5$ is a finite constant for normalization. Moreover, we have $p_*(\theta, u)>0$.
\end{proposition}

 In the next section,  we study the convergence of  $p_t$ under the evolutionary process \eqref{evo p}.   Finding the global minimizer of \eqref{problem2} is not necessarily easy  as the underlying variable is infinite dimensional, even though \eqref{problem2} is convex. Here, we give a more general result showing that NGD converges to global minimal solution and state it as follows:
\begin{theorem}[Convergence of NGD]\label{conver1}
Under Assumption \ref{ass:h2}, \ref{ass:h3}, and \ref{ass:h4}, and suppose that $p_t$ evolves according to \eqref{evo p},   then $p_t$  converges weakly to $p_*$. Moreover,
\begin{eqnarray}
\lim_{t\to\infty} Q(p_t) = Q(p_*). 
\end{eqnarray}
\end{theorem}
The proof of Theorem \ref{conver1} with an additional proof sketch is given in Appendix \ref{proof: conver}. Note that similar convergence result has  been proposed by several earlier works. However, we consider a more general situation that still requires a separate treatment. We list the differences below.
\begin{itemize}
    \item Our proof is inspired by \cite{MeiE7665}. The proofs share the same proof outline: prove the tightness of $p_t$ first, and then  take a sub-sequence  $p_k$, and  show that $p_k$ converges weakly to the desired distribution.   However,  in \cite{MeiE7665}, they consider a bounded neuron architecture.  That is of the form as $\frac{1}{m}\sum_{i=1}^m\sigma(u_i h(\theta_i,x))$ where $\sigma$ is a certain bounded function. This form is not the standard NN. The boundedness condition  for  $\sigma$ is fundamentally required in their proof, because  from \eqref{opt}  the tail distribution of $p_*$ is Gaussian   indicating that  even in the optimal solution $u h(\theta,x)$ is unbounded. Instead, we study the standard NN in form of \eqref{NN form}, which has not been covered in the existing literature. Our proof involves a non-trivial analysis by carefully studying the tails for the test functions. 
    \item  \cite{chizat2018global}  proposed a convergence result for GD when the activation function satisfies a certain homogeneous property.  \cite{wei2018margin} further extends it and show that the convergence rate could be polynomial when the variables are injected with a certain type of noise.  Our analysis does not rely on the homogeneous  property and thus is more general.  The two proof ideas differ from each other.

 \item In \cite{mei2019mean}, the authors extended some of the results in \cite{MeiE7665} to the unbounded neuron architecture.  They    proved the existence and the uniqueness of the (N)GD dynamics and further provided a sharper bound for  uniform approximate complexity  using a finite number of hidden neurons. However, they did not prove  any convergence result for the (N)GD dynamics.
\end{itemize}

\subsection{Near Optimal Feature Representation}\label{sub near}
This section shows that NN can learn feature populations that are near optimal under our criterion proposed in Section \ref{criterion}.  Let 
\[
\rho_*(\theta) = \int_{\RR} p_*(\theta, u)du, 
\]
which is the feature distribution that we have learned.  So after training, the features are repopulated  from $\rho_*$.   Define
\[
\omega_*(\theta) = \frac{\int_{\RR} up_*(\theta, u)du }{\rho_*(\theta)} ,
\]
which denotes the learned weights in the top layer. Let $\mu_*(\theta) = \rho_*(\theta) \omega_*(\theta)$ be  the learned signed measure  that represents the target function  $f_*$  for the underlying learning task, where we have the representation 
\[
f_*(x) = \int_{\RR^d} h'(\theta, x)\mu_*(\theta) d\theta .
\]
The  following theorem establishes the relationship between $\rho_*$ and $\mu_*$ and thus demonstrates that the learned feature distribution $\rho_*(\cdot)$ 
is nearly proportional to the target measure $|\mu_*(\cdot)|$ when $\lambda_2 \ll \lambda_1$.

\begin{theorem}[Target Feature Population]\label{optimal fea}
Under Assumption \ref{ass:h2} and \ref{ass:h3}, we  have
\begin{eqnarray}\label{ccc0}
   \left| \frac{\mu_*(\theta)}{\rho_*(\theta)}\right|^2 = \frac{\lambda_2}{3\lambda_1}\|\theta \|^2 + B_8,
\end{eqnarray}
where $B_8$ is a finite constant.
\end{theorem} 
Theorem \ref{optimal fea} implies that  the final feature  distribution $\rho_*$  is no longer  an arbitrary random distribution but a specific distribution 
associated with the learned target function of the underlying learning task.
Therefore NN  is fundamentally different from Random Kitchen Sinks \citep{rahimi2009weighted}, in which the features are sampled from a completely random distribution  $\rho_0(\theta)$ unrelated to the underlying learning task.  

It is worth noting that the optimal distribution can be decomposed as $p_*(\theta,u)=p_*(u|\theta) \rho_*(\theta)$.
Although it can be shown that $p_*(u|\theta)$ is a Gaussian distribution with mean $\omega_*(\theta)$ (see \eqref{u_2} and discussions thereafter), the actual form is not important in our theory. In fact, only $\omega_*(\theta)$ and $\rho_*(\theta)$ are of interests in the neural feature repopulation view, and we are particularly interested in the feature population $\rho_*(\theta)$ learned by a two-level continuous NN. 

To explicitly demonstrate that $\rho_*(\theta)$ is more efficient in terms of sampling from the features, we consider the following conditions, where $\rho_*$ is  near optimal  under our criterion. 
\begin{corollary}[Near Optimal Feature Representation]\label{optimal fea2}
Under the assumptions in Theorem \ref{optimal fea},  for any $M>0$, for all  $\theta $  with $ \| \theta\|\leq M$, we have
\begin{eqnarray}\label{ccc}
\frac{|\mu_*(\theta) |}{B_8+ \frac{\lambda_2}{3\lambda_1}M} \leq \rho_*(\theta) \leq  \frac{|\mu_*(\theta) |}{B_8}. 
\end{eqnarray}
This implies that for any given $\lambda_1>0$ and $\theta$, as  $\lambda_2 \to 0$,  we have
\begin{eqnarray}\label{optimal}
\rho_*(\theta) \to \frac{|\mu_*(\theta) |}{C_*},
\end{eqnarray}
where $C_*=\int_{\RR^d}  |\mu_*(\theta) | d\theta$ is a finite  constant normalization factor.
\end{corollary}
Corollary \ref{optimal fea2}, and specifically \eqref{optimal} says that  when
{$\lambda_1 \gg \lambda_2$}, 
$\rho_*(\theta)$ is a near optimal feature representation according to our definition of optimality in \eqref{eq:opt}. This is because
\begin{equation}
    \frac{|\mu_* |}{{C}_*} = \argmin_{\rho} V(\mu_*, \rho) ,
\end{equation}
and the optimal value is
\[
V(\mu_*,|\mu_*|/C_*)= \left[ \int |\mu(\theta)| \; d \theta\right]^2 . 
\]
This means that the $\ell_2$ regularized two-level NN (with { $\lambda_2 \ll \lambda_1$})
tries to solve the $\ell_1$-norm regularization
problem with respect to the target function.
We make additional remarks below based on Corollary \ref{optimal fea2}. 
\begin{itemize}
    \item Corollary \ref{optimal fea2} shows that learned distribution $\rho_*$ can achieve near minimal variance $V$ to represent the target function $f_*$.  So the result of NN training is that     features are resampled from a near-optimal distribution in terms of sample efficiency to approximate $f_*$.
    \item With   Corollary \ref{optimal fea2} in hand,  we can give a direct comparison between NN and Random Kitchen Sinks \citep{rahimi2009weighted}.  NN  outperforms Random Kitchen Sinks with  more efficient features.  One can easily construct examples to demonstrate the efficiency.  For example, let $d=1$, and consider a `hard' target function $\tilde{f}_*$ as the composition of two  narrow boxcar functions, i.e.
    \begin{eqnarray}
\tilde{\mu}_*(\theta)= 
\begin{cases}
\frac{1}{2a},&   |\theta-1| \leq a,\\
\frac{1}{2a},&   |\theta+1| \leq a, \\
0,&  \text{others},
\end{cases}
\end{eqnarray}
where $0<a<1$.   It is not hard to obtain that  for the learned NN, $V(\tilde{\mu}_*, \tilde{\rho}_*) = 1$, where $\tilde{\rho}_* = \argmin_{\rho} V(\tilde{\mu}_*, \rho)$.  However, for all $\rho_N$ that satisfies a Gaussian distribution $N(0, \sigma^2)$, we have $V(\tilde{\mu}_*, \rho_N) \geq \frac{\sqrt{2\pi}\exp(1/2)}{a}$ (we refer the readers to the proof in Appendix \ref{vv pro}). So when $a$ is small,  $f$ can be efficiently approximated by the NN with much fewer hidden units compared to Random Kitchen Sinks that uses the  universal Gaussian distribution to sample the features.
    \item We shall  mention that the assumption {$\lambda_1\gg \lambda_2$}  is relatively weak because the magnitudes of $\lambda_1$ and $\lambda_2$ are significantly smaller than that of the loss $J(\omega, \rho)$.  The assumption is also reasonable because when {$\lambda_1\gg \lambda_2$}, the complexity of the model concentrates on the top layer,  so that $|\omega_*|$ is approximately a constant. 
    On the other hand,  when the  assumption  breaks down, \eqref{ccc0} and \eqref{ccc} still  imply that $\rho_*(\cdot)$ is approximately proportional to $\mu_*(\cdot)$, and thus is more efficient  for sampling. 
    \item Our result also provides a new explanation for the batch normalization (BN) tick \citep{ioffe15}.  Before this work, some previous theoretical justification argued that BN helps the optimization process \citep{santurkar2018does},  while our analysis provides a different view. In fact, we show that BN can improve the efficiency of feature representation for neural networks.    Since BN normalizes the second moment of the feature function with respect to the data distribution $D$ for all $\theta$,  we have $B_v =1$.  With this observation, BN reduces the sampling variance for the learned features, so that it is possible for the underlying NN to represent the target function $\hat{f}$ more efficiently, as in \eqref{def bv}.   From this perspective, to achieve the same effect, we might consider imposing a  $\theta$-dependent regularizer as  an alternative for BN, which might lead to a more efficient algorithm.  
\end{itemize}

\section{Generative Feature Repopulation}\label{sec:repop}

This section considers an empirical consequence of our theory.
From Theorem \ref{optimal fea}, when the number of hidden units goes to infinity, 
the distribution of learned features is near optimal after GD converges. 
Therefore compared with the random distribution in the initial step,
sampling the hidden units from the learned distribution will be much
more efficient.   This can be useful when we want to retrain the
network with a different number of hidden units. In such case, one
generally needs to retrain using the random initialization. However,
our theory shows that it is better to initialize with the learned
feature population, i.e., the repopulated features. 
This leads to the procedure of Figure~\ref{fig:algo} for two level NN
initialization when we vary the number of hidden units. 

\begin{figure}[H]
\begin{mdframed}[style=exampledefault]
	\begin{enumerate}[Step 1.]
		\item  Consider a two-level NN with a massive $m$ hidden nodes, denoted as
\begin{eqnarray}\label{experiment}
f(u,\theta;x) = \frac{1}{m}\sum_{i=1}^m u^j 
\sigma\left(\left\langle  \theta^j,x\right\rangle\right),
\end{eqnarray}		
where $x\in \RR^d$ is the data, 
each $w_j$ is a vector of dimension $d$, and $\sigma$	is activation functions.	
		\item   Initialize the  NN with Gaussian distribution for weights $(\tilde{u}, \tilde{\theta})$. Then train the NN by a standard algorithm, e.g. GD after some iterations and obtain the weights $(\hat{u}, \hat{\theta})$.
		\item Run an efficient  generative model such as GAN or VAE to learn a mapping $g(\cdot)$, so that $g(\tilde{\theta}^j, \ep^j)$ and $\hat{\theta}^j$ have the same distribution. We learn how to map an initial $\tilde{\theta}^j$ to its trained weight $\hat{\theta}^j$, and allow some uncertainty characterized by a Gaussian noise $\ep^j$. 
		\item  Start a new NN of size $m'$.  We generate repopulated weights as follows: generate Gaussian random initialization $\tilde{\theta}$ as before. For each $\tilde{\theta}^j$, we compute $g(\tilde{\theta}^j, \ep^j)$. This is called \emph{repopulation}.
		\item  We consider the function: $f(u;x) = \frac{1}{m}\sum_{i=1}^m u^j \sigma\left(\left\langle \hat{\theta}^j,x\right\rangle \right)$ . We train it with fixed $\hat{\theta}$ and learn the weights $u^j$ only. 
	\end{enumerate}
\end{mdframed}
\caption{NN Initialization with Generative Feature
  Repopulation}
\label{fig:algo}
\end{figure}

$~~$\\We can study the performance of two-level NN when varying $m'$.  We can also compare the results by fixing $\theta$ and learn weights $u$ only in \eqref{experiment} with the original initialization when $m$ changes.  This is analogous to what Random Kitchen Sinks does, but with a better random distribution over features that is optimized for the underlying learning task.  Our hope is that we can use a much smaller $m'$
to achieve a satisfactory accuracy. 

\section{Empirical Results}\label{sec:empirical}

In our experiments, we use some practical examples to demonstrate the efficacy of the learned NN weight distribution, which provides better features than task independent random features. We design the following experiments to validate our theory: (1) employ the procedure described in Figure \ref{fig:algo} to demonstrate the performance gain with feature repopulation (Section \ref{feature_repop}); (2) visualize the learned weight distribution, which differs from task independent random distribution significantly (Section \ref{weight_visual}).  Some additional experiments are provided in the appendices: (3) employ NGD with simulated data, which is similar to SGD (Appendix \ref{simu}); (4) validate the efficiency of learned feature in experiments (Appendix \ref{efficacy_representation});
(5) repopulate inefficient features (large $\lambda_2$) by importance sampling  (Appendix \ref{impor});  (6) compare different initialization strategy when training the whole network (Appendix \ref{feat_init}); (7)  Visualizations of features (Appendix \ref{feat_visu});

\subsection{Performance on Classification Tasks with Repopulated Feature}\label{feature_repop}

We employ Step 3 by Conditional Variational Auto-Encoder \citep{kingma2014semi,sohn2015learning} for learning the mapping $g(\cdot)$. Two classification problems are examined: (1) The MNIST dataset with 60,000 training examples and 10,000 testing examples. Each example is a $28\times28$ image labeled by $0,\cdots,9$. (2) The CIFAR-10 dataset contains 60,000 $28\times28\times3$ color images  in 10 classes. Besides, 50,000 of them are training examples and the others are testing examples. In both datasets, we set $m=100$ and early stopping to avoid the overfitting problem. All  instances are flattened to a 1D tensor as the input of feedforward NN. The activation layer $\sigma(x) = \tanh (x)$. Binary Cross-Entropy Loss is applied in our trials (practically, it is bounded with a small shift). The sampling distribution of random weights  and the generative model prior are both standard normal distributions. We employ a two-level feedforward NN and the Adam optimiser in our experiments, which is the setting used by practitioners.

\begin{figure}[H]
	\begin{center}
		\includegraphics[width=0.4\linewidth]{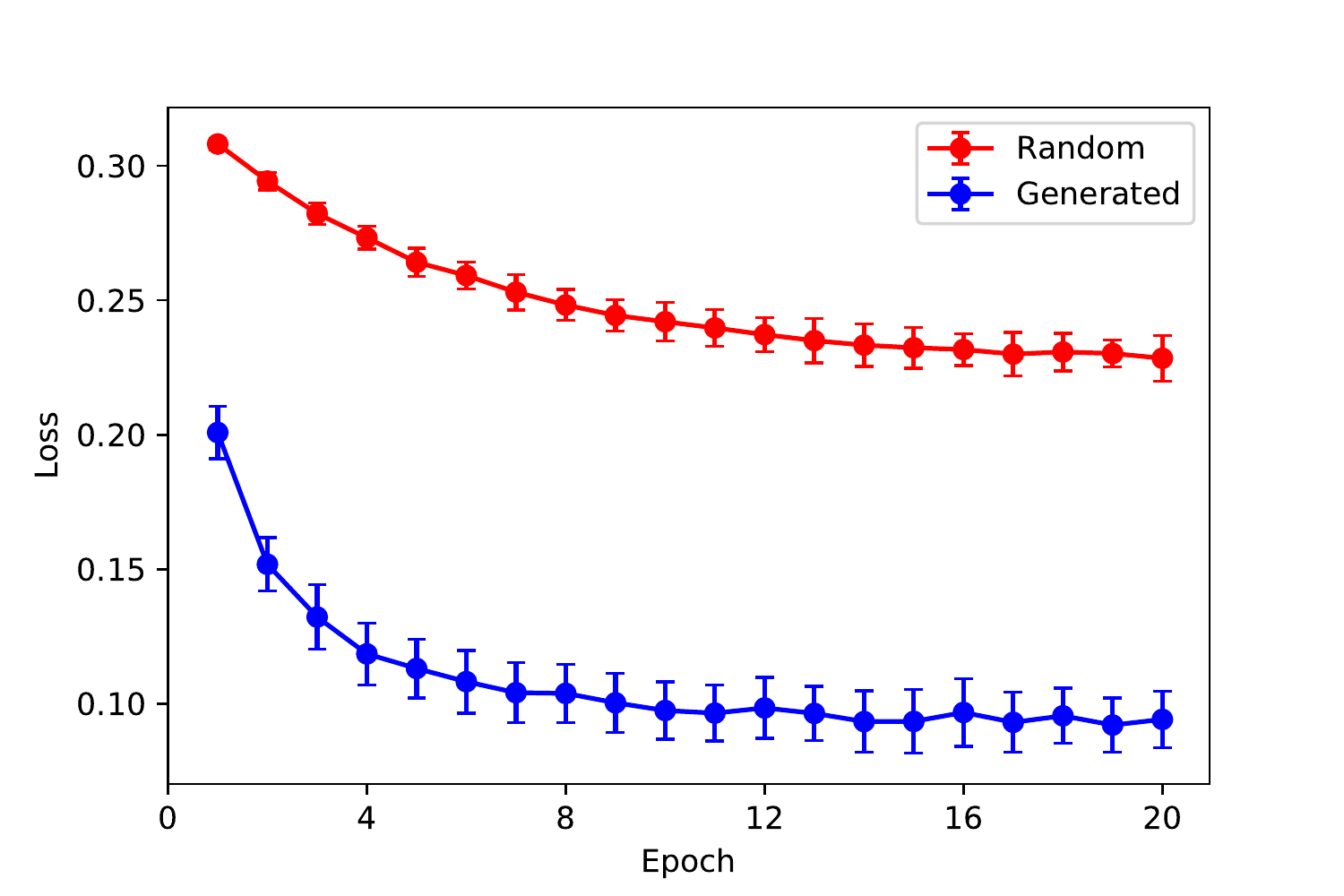}
		\includegraphics[width=0.4\linewidth]{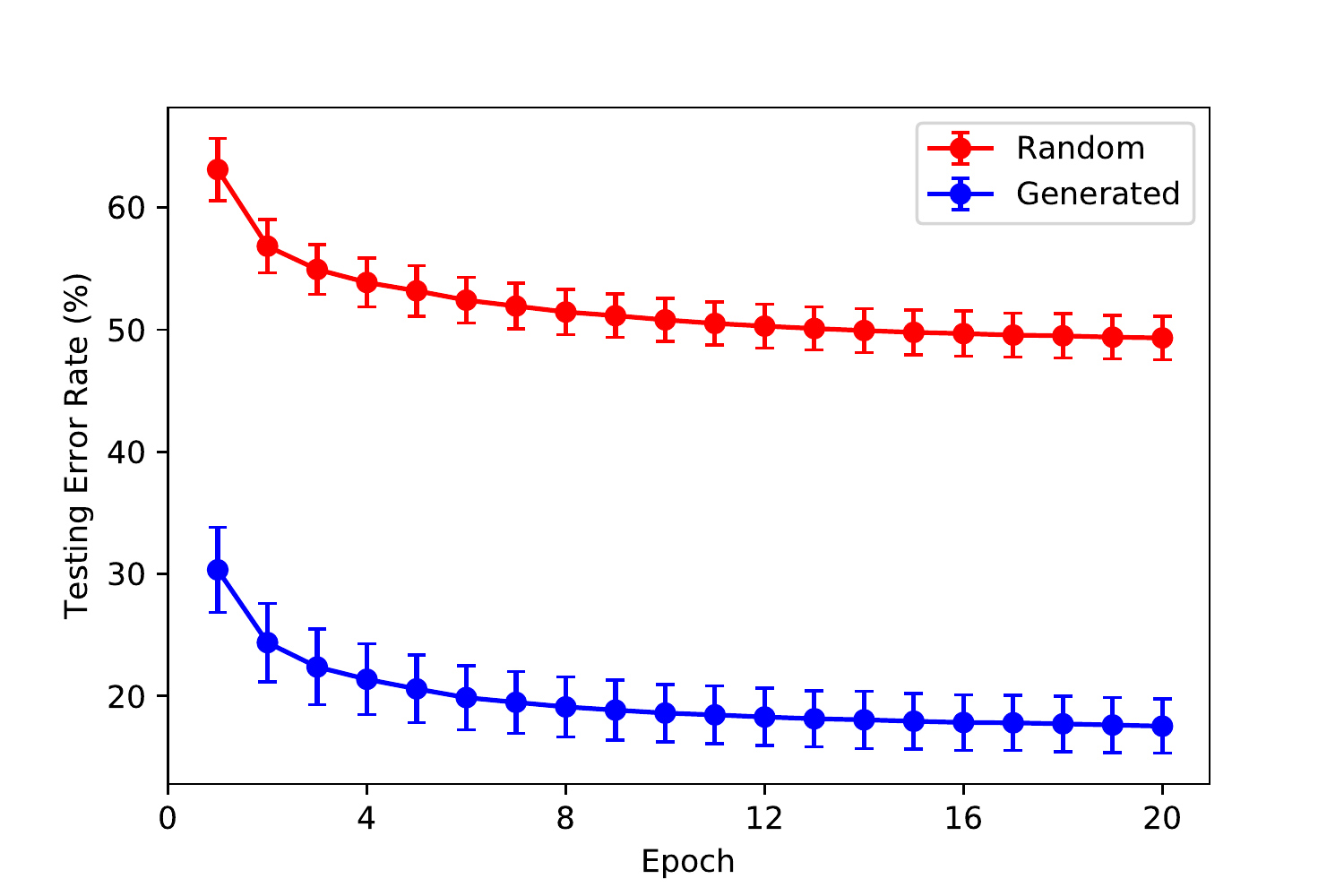}
		
	\end{center}
	\begin{center}
		\includegraphics[width=0.4\linewidth]{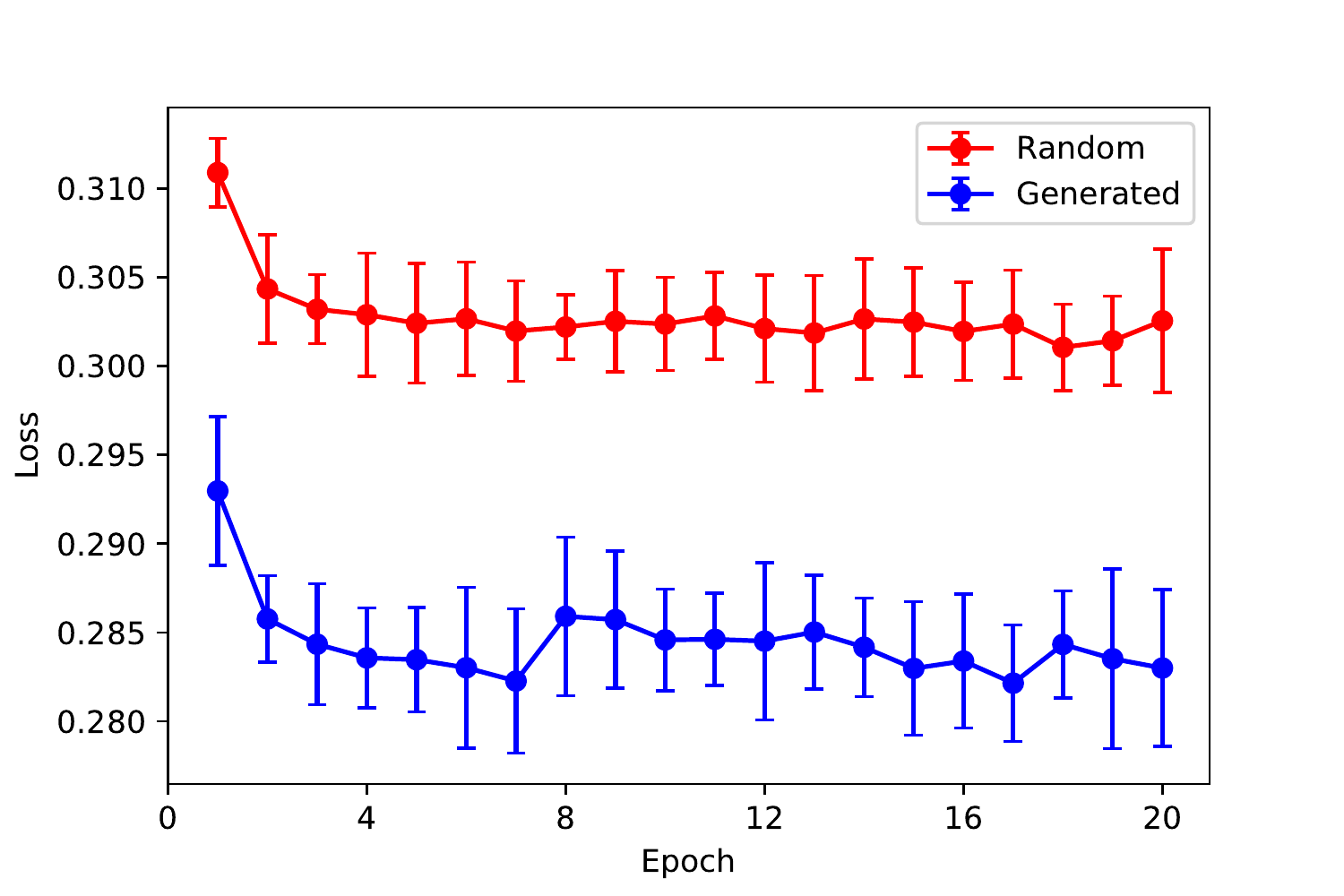}
		\includegraphics[width=0.4\linewidth]{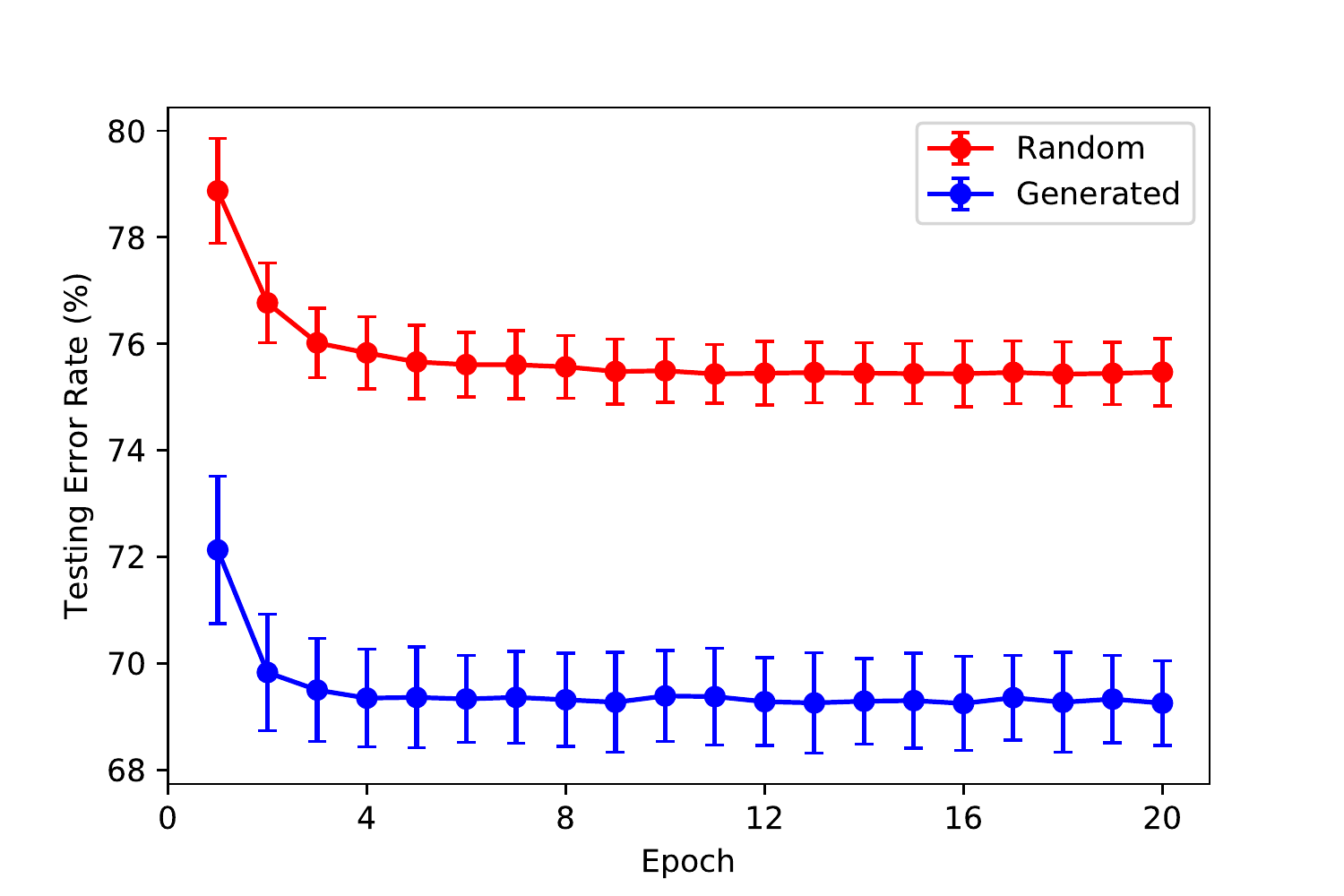}
		
	\end{center}
	
	\caption{Comparisons between Random Feature (Random) and Sampled Feature (Generated) on MNIST (The $1_{st}$ row), CIFAR-10 (The $2_{nd}$ row). The $1_{st}$ and the $2_{nd}$ columns plot the training loss and the testing error rate respectively. The optimization process only considers the last layer $u$. The two-level neural net contains 10 hidden nodes ($m'=10$). The error bars represent standard deviation.}
	\label{fig:empirical}
\end{figure}

Figure \ref{fig:empirical} reports the results of aforementioned experiments. To limit the capacity of the two-layer NN, we only consider $m'=10$, which is an extremely unproductive example that makes random feature almost unable to depict the label information. 
The first column plots the training loss, indicating the capacity of random and sampled features for the classification problem. The second column plots the test error rate, which represents the generalization ability. Red lines indicate random feature -- Random Kitchen Sinks. Blue lines indicate features sampled from the learned weight distribution.

These experiments showed that the sampled weights provide better features that can lead to superior performance on classification problems. The results are consistent with our theory that the features can be learned by \emph{over population} so that the weight distribution can be strengthened by a generative model, rather than the kernel view that NN only employs random feature. 

\subsection{Weights Visualization}\label{weight_visual}
We compare three types of weights: (a) Random weights: Random Kitchen Sinks with Normal distribution; (b) Optimized weights: trained NN weights with back-propagation; (c) Generated weights: weights sampled from the generative model. Figure \ref{fig:visualization} shows that both optimized and generated weights have certain patterns and structures, which are able to extract corresponding attributes of the instances, i.e., learned features. Although the generated weights are relative blur, it is due to the drawbacks of our generative model. Note that the corners of the learned weight matrix appear random, which is similar to the corresponding parts of the random initialization, while edges and target specific patterns can be observed in the center.

\begin{figure}[H]
	\begin{center}
		\begin{tabular}{ccc} 
			\includegraphics[width=0.3\linewidth]{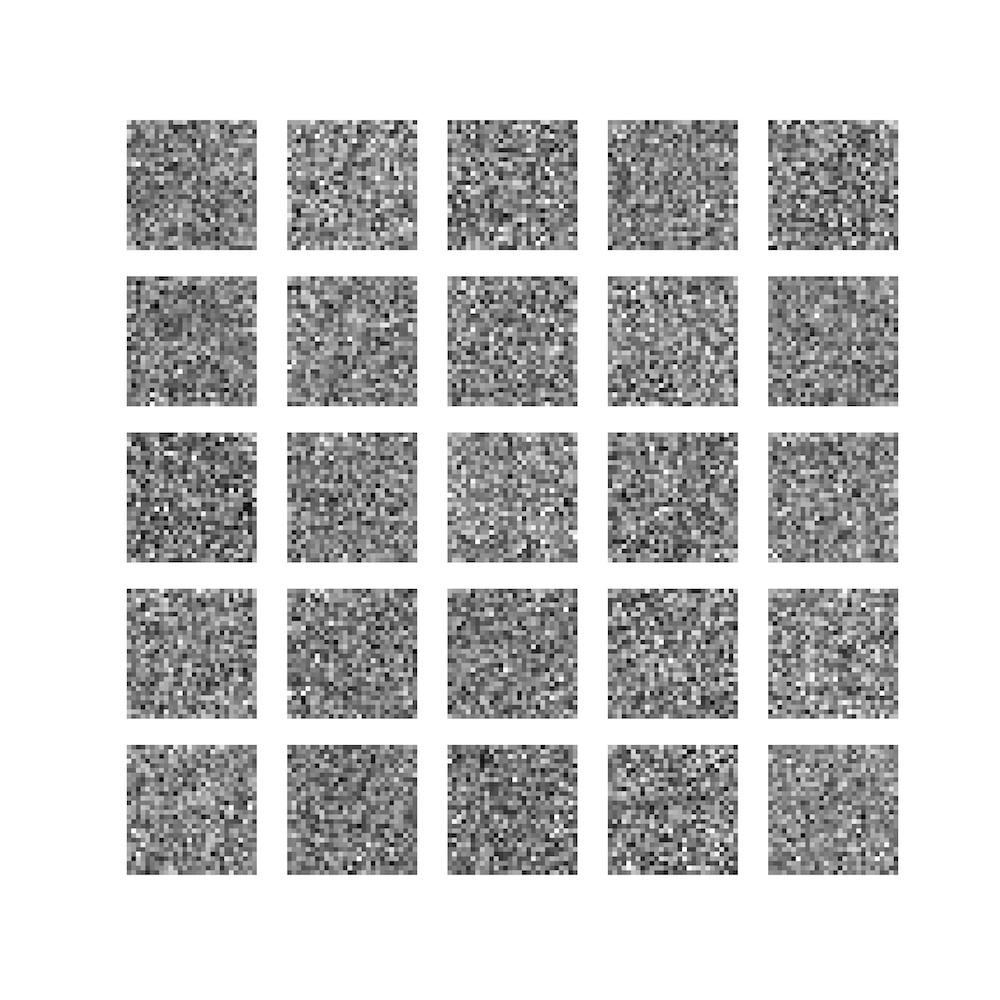}&
			\includegraphics[width=0.3\linewidth]{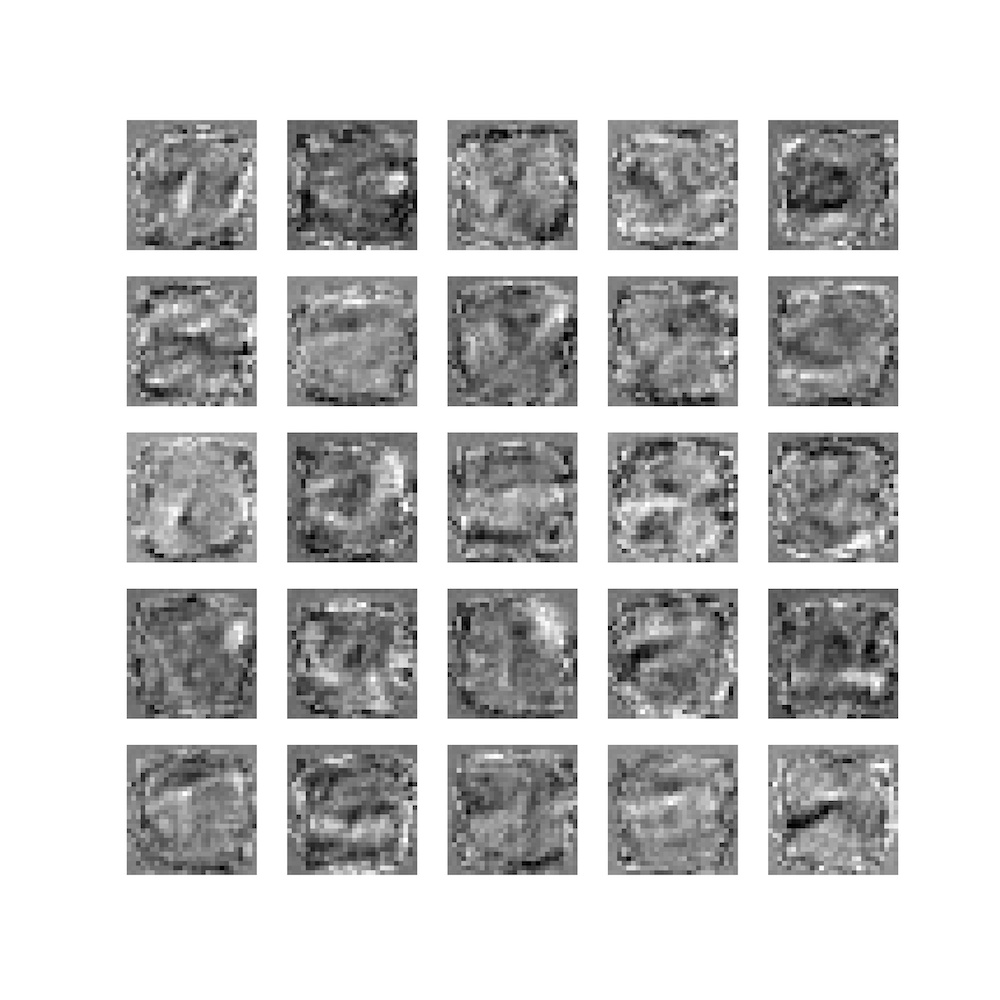}&
			\includegraphics[width=0.3\linewidth]{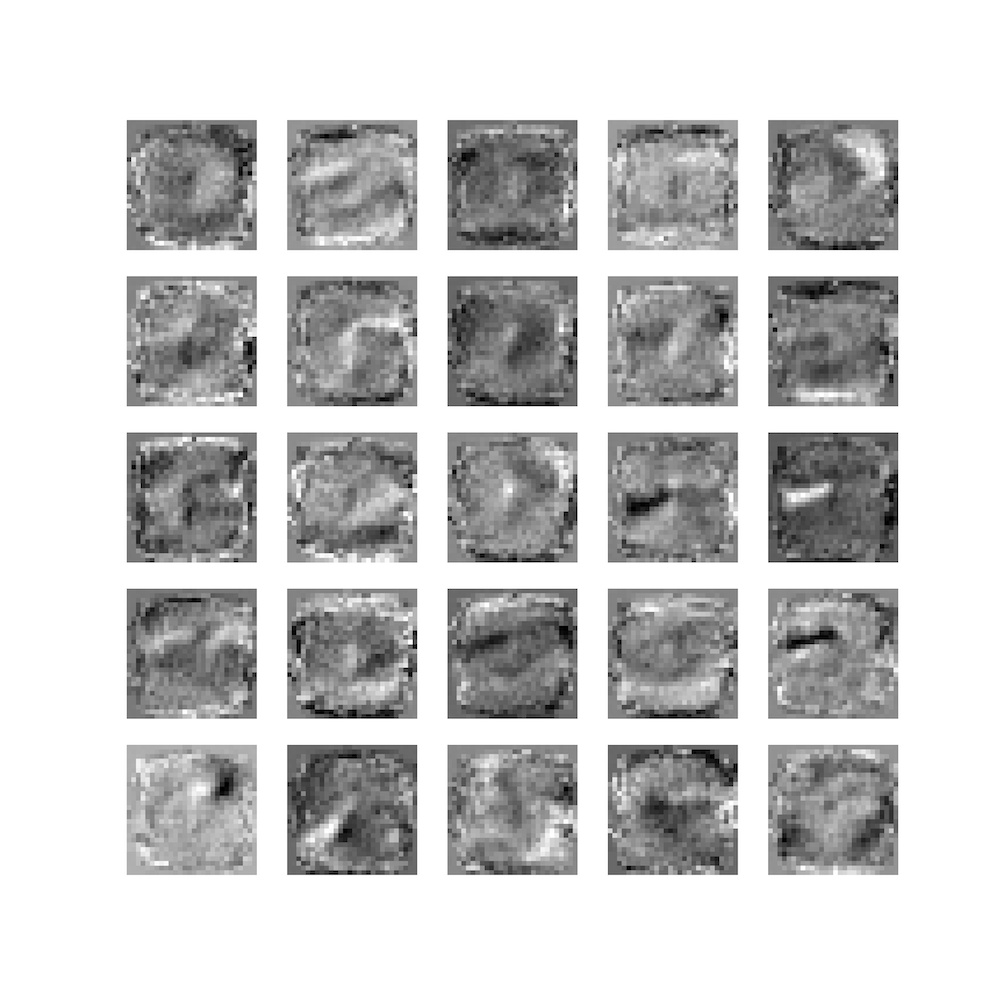}\\
			(a) Random weights & (b) Optimized weights & (c) Generated weights
		\end{tabular}
	\end{center}

	\caption{Visualization of weights on MNIST dataset (reshaped as $28\times 28$ matrix)}
	\label{fig:visualization}
\end{figure}

Beyond what has been reported here, additional experiments that can be used
to validate our theory are provided in Appendix \ref{efficacy_representation}.  
These additional empirical results also validate the claim that the task optimized features from two-level NN are significantly better than the task independent random features. This means that the target function space can be more reliably estimated with fewer hidden nodes.

\section{Conclusions}\label{sec:conclusion}
In this paper, we proposed a new framework for analyzing
over-parameterized NNs referred to as neural feature
repopulation. It was shown that when the number of neurons approach
infinity, the NN feature learning process becomes convex. With
appropriate regularization, the learned features is nearly optimal
in terms of efficient representation of the target function. 
This is the first theory that can satisfactorily explain the effectiveness of
feature learning process of two-level NNs. 

Our framework is closely related to the mean field view, but there are fundamental differences that allow us to easily generalize the method to deep neural networks. In comparison, direct generalization of the mean field view will lead to technical difficulties that cannot be overcome without introducing new ideas such as what we do. 
Specifically, in our view,
neural networks are equivalent to linear classifiers with learned features optimized for the learning task. The feature learning process is characterized by regularization, and the learned features can be represented by distributions over hidden nodes. Discrete neural networks can be obtained by sampling hidden nodes from the learned feature distributions. The efficiency of such sampling measures the quality of learned feature representations. With appropriate reparameterization, the overall system can be posed as a convex problem. The generalization of these ideas to deep neural networks will be left to a subsequent work.

\paragraph{Acknowledgement}
The authors would like to thank Yihong Gu for helpful discussions and studying Theorem \ref{conver1} together. The  authors  also  thank Jason Lee,  Haishan Ye for  helpful discussions.




\bibliographystyle{apalike2}
\bibliography{overbib2}

\pagebreak\appendix
\pb

\section{Proofs in Section \ref{sec:analysis}}\label{app: con}
\subsection{Proof of  Proposition \ref{proposition: unique}: Global Optimal  Solution}
\begin{proof}[Proof of Proposition \ref{proposition: unique}]
Since $Q(p)<+\infty$, we have $\int_{\RR^{d+1}} p\mathrm{ln} (p) d\theta du <
+\infty$. This means that $p$ is absolutely continuous with respect to
the Lebesgue measure, thus it is a probability density function.  We consider
\begin{eqnarray}\label{pro}
&\minimize_p& \quad Q' (p)\\
&\mathrm{s.t.}& p(\theta, u)\geq 0, \quad \forall \theta\in \RR^d, u\in \RR,\notag\\
&& \int_{\RR^{d+1}} p(\theta, u)d\theta du = 1,\notag
\end{eqnarray}
where 
\begin{eqnarray}\label{all}
Q'(p) &=& \underbrace{\E_{x,y}\phi\left(\int_{\RR^{d+1}}  uh'(\theta, x) p(\theta, u) d\theta du, y   \right)}_{I_6^{(1)}} + \underbrace{\int_{\RR^{d+1}}  \left(\frac{\lambda_1 |u |^2+ \lambda_2 \|\theta \|^2}{2} p(\theta, u)\right)d\theta d u}_{I_6^{(2)}}\notag\\
&& +\lambda_3\underbrace{\int_{\RR^{d+1}}  p(\theta, u) \mathrm{ln}p(\theta, u) d\theta d u}_{I_7}.
\end{eqnarray}
We can check that  $I_6^{(1)}$ and $I_6^{(2)}$ are convex on $p$ and
$I_7$ is strictly convex on $p$. It follows that there is a unique
solution a.e., for $p_*(\theta, u)$ of Problem \eqref{pro}.  Moreover, $p_*$ satisfies the Karush–Kuhn–Tucker (KKT) conditions \citep{MartinBurger}, i.e. for any $\theta \in \RR^{d}$ and $u\in \RR$, 
\begin{eqnarray}
\EE_{x,y}[\phi'_*(x,y)u h'(\theta, x)] + \frac{\lambda_1| u|^2}{2}+ \frac{\lambda_2 \|\theta \|^2}{2} + \lambda_3\mathrm{ln}(p_*(\theta,u)) +  \lambda_3 + B_1(\theta, u) + B_2 &=& 0, \label{4 1}\\
B_1(\theta,u) p_*(\theta, u) &=& 0,\notag\\
p_*(\theta, u) &\geq& 0, \notag\\
B_1(\theta, u) &\leq& 0,\notag\\
 \int_{\RR^{d+1}} p_*(\theta, u)d\theta du &=& 1,
\end{eqnarray}
where  $\phi'_*(x,y) = \nabla_{y'}\phi( \int_{\RR^{d+1}} uh'(\theta, x) p_*(\theta, u) d\theta du, y  )$.
We first prove $p_*(\theta, u)>0$.  

Because $ \int_{\RR^{d+1}} p_*(\theta, u)d\theta du = 1 $, there exists a $(\theta_1, u_1)$ with $p_*(\theta_1, u_1)>0$. So $B_1(\theta_1, u_1) = 0.$  From the Assumption  \ref{ass:h2}, $|\nabla_{y'} \phi|$  is bounded by $L_1$.  Then from  \eqref{4 1}, we have    $B_2$ is bounded. 

Suppose there is a point $(\theta_2, u_2)$ with $p_*(\theta_2, u_2) = 0$. Then $\mathrm{ln}(p_*(\theta_2,u_2))= -\infty$. Because $B_1(\theta_2, u_2)\leq 0$, we  have $B_2 = +\infty$, which is contradiction to the fact that $B_2$ is bounded. This in turn shows that $p_*(\theta, u)>0$, $B_1(\theta,u)=0$ holds for all $\theta$ and $u$.  Also, $\rho_*(\theta) = \int_{\RR} p_*(\theta, u)du >0 $ holds for all $\theta$.  With $B_1(\theta, u)=0$ and $B_2$ is bounded,   from \eqref{4 1}, we have 
\begin{eqnarray}\label{p res}
p_*(\theta,u) = \frac{\exp\left( - \frac{\lambda_1 }{2\lambda_3}|u |^2 - \frac{\lambda_2}{2\lambda_3}\| \theta\|^2 - \frac{1}{\lambda_3}\EE_{x,y}[\phi'_*(x,y)h'(\theta,x)]u   \right)}{C_5},
\end{eqnarray}
where $C_5 = \int_{\RR^{d+1}}  \exp\left( - \frac{\lambda_1 }{2\lambda_3}|u |^2 - \frac{\lambda_2}{2\lambda_3}\| \theta\|^2 - \frac{1}{\lambda_3}\EE_{x,y}[\phi'_*(x,y)h'(\theta,x)]u   \right) d\theta du $ is a constant for normalization.  $C_5$ is finite by noting 
\begin{eqnarray}
&&| \EE_{x,y}[ \phi'_*(x,y) h'(\theta,x)]u  |\leq L_1  \EE_{x,y} |  h'(\theta,x) |u \leq L_1  \sqrt{\EE_{x,y} |  h'(\theta,x)|^2} u \leq L_1   \sqrt{B_v} u.
\end{eqnarray}
\end{proof}
\subsection{Proof of  $p_t>0$}
\begin{lemma}\label{p great 0}
Assume $p_t(\theta, u)$ evolves as  \eqref{evo p}. And $\lambda_3 >0$, then $p_t(\theta, u)> 0$ for all $t>0$. 
\end{lemma}
\begin{proof}
In the proof, we use $\theta'$ to denote $[\theta, u]$.  With a little abuse of notation, we use $p(\theta')$ to denote $p(\theta, u)$.  Let $g_3(t,\theta') = [g_1(t, \theta, u), g_2(t, \theta, u)]$, then  \eqref{evo p} can be rewritten as:
\begin{eqnarray}
\frac{d p_t(\theta')}{d t} =  \lambda_3  \nabla_{\theta'}^2 p_t(\theta') -  \nabla_{\theta'}\cdot [p_t(\theta')g_3(t,\theta') ].    
\end{eqnarray}
Let $\hp_t(\theta') = \mathrm{ln}  p_t(\theta') $. When $t=0$,  from  Assumption \ref{ass:h4},    $\hp_0(\theta') > -\infty $ holds for all $\theta' \in \RR^{d+1}$.  We want to show that $\hp_t$ cannot go to negative infinity.  We have
\begin{equation}
    \frac{d \exp(\hp_t(\theta'))}{dt} = \lambda_3 \nabla_{\theta'}^2 [\exp(\hp_t(\theta'))] -  \nabla_{\theta'}\cdot [\exp(\hp_t(\theta'))g_3(t,\theta') ].
\end{equation}
It follows that
\begin{eqnarray}
\exp(\hp_t(\theta'))  \frac{d \hp_t(\theta')}{dt} &=& \lambda_3\left(  \exp(\hp_t(\theta'))   \nabla_{\theta'}^2 [\hp_t(\theta')] + \exp(\hp_t(\theta'))\|\nabla_{\theta'}\hp_t(\theta')\|^2       \right)\\
&& - \exp(\hp_t(\theta')) [\nabla_{\theta'}\hp_t(\theta')]\cdot  g_3(t,\theta') - \exp(\hp_t(\theta'))\nabla_{\theta'}\cdot g_3(t, \theta').\notag
\end{eqnarray}
We obtain
\begin{eqnarray}\label{hp}
\frac{d \hp_t(\theta')}{dt} &=&  \lambda_3  \nabla_{\theta'}^2 \hp_t(\theta') +  \lambda_3  \|\nabla_{\theta'}\hp_t(\theta')\|^2-  [\nabla_{\theta'}\hp_t(\theta')]\cdot g_3(t,\theta') - \nabla_{\theta'}\cdot g_3(t, \theta').
\end{eqnarray}

From Assumption \ref{ass:h2} that $\|\nabla_{y'} \phi\|\leq L_1$ and Assumption \ref{ass:h3} that $\|\nabla_{\theta} h'(\theta,x) \|\leq C_3$, we have
\begin{eqnarray}\label{bound g_2}
\| g_2(t, \theta, u) \| \leq L_1C_3|u|+\lambda_2\| \theta\|.
\end{eqnarray}

Also  from Assumption \ref{ass:h3} that  $\| h'(\theta,x) \|\leq C_1 \| \theta\|+C_2$, we have
\begin{eqnarray}\label{bound g_1}
 \|g_1(t, \theta, u)\| \leq  L_1(C_1\| \theta\|+C_2) +  \lambda_1| u |. 
\end{eqnarray}
So 
$$\| g_3\| \leq  \| g_1\|+\| g_2\| \leq (L_1C_1+L_1C_3+\lambda_1+\lambda_2)\|\theta'\|+ L_1C_2,$$
which indicates that
\begin{eqnarray}\label{g3}
\| g_3\|^2 \leq  2 (L_1C_1+L_1C_3+\lambda_1+\lambda_2)^2\|\theta'\|^2 + 2L_2^2C_2^2.
\end{eqnarray}
Similarly, from $|\nabla^2_{\theta} h(\theta, x)|\leq C_3$, we can bound
$$ |\nabla_{\theta}\cdot g_2(t, \theta, u)|  \leq  L_1 C_3\|u\| + \lambda_2   $$
and
$$ |\nabla_{u}\cdot g_1(t, \theta,u)|  \leq  \lambda_1. $$
We have 
\begin{eqnarray}\label{nabla g3}
|\nabla_{\theta'} \cdot g_3  | \leq  L_1C_3\|\theta'\|+\lambda_1+\lambda_2\leq L_1C_3\| \theta'\|^2+\lambda_1+\lambda_2+L_1C_3.
\end{eqnarray}
In the following, we set 
$$A =  (L_1C_1+L_1C_3+\lambda_1+\lambda_2)^2/(2\lambda_3^2) +L_1C_3$$ and $$B =    L_1^2C_2^2/(2\lambda_3^2) +(L_1C_3+\lambda_1+\lambda_2). $$  
Let
\begin{eqnarray}
g_4(t, \theta') = \lambda_3  \|\nabla_{\theta'}\hp_t(\theta')\|^2-  [\nabla_{\theta'}\hp_t(\theta')]\cdot g_3(t,\theta') - \nabla_{\theta'}\cdot g_3(t, \theta')\notag.
\end{eqnarray}
We have
\begin{eqnarray}\label{g_4b}
&&g_4(t, \theta')\\
& =&\lambda_3 \|\nabla_{\theta'}\hp_t(\theta') - g_3(t,\theta')/(2\lambda_3)  \|^2  -  \|g_3(t,\theta')/(2\lambda_3) \|^2  - \nabla_{\theta'}\cdot g_3(t, \theta')\notag\\
&\overset{a}{\geq}& \lambda_3 \|\nabla_{\theta'}\hp_t(\theta') - g_3(t,\theta')/(2\lambda_3)  \|^2 - A \|\theta' \|^2 - B\notag\\
&\geq& - A \|\theta' \|^2 - B,\notag
\end{eqnarray}
where $\overset{a}{\geq}$ uses \eqref{g3} and \eqref{nabla g3}. 
On the other hand,  by the green’s function method, the solution of \eqref{hp} satisfies  
\begin{eqnarray}
\hp_t(\theta') &=&   \frac{1}{(4\pi \lambda_3 t)^{\frac{d+1}{2}} } \int_{\RR^{d+1}} \hp_0(\bar{\theta}') \exp\left(  \frac{- \| \theta' -\bar{\theta}'\|^2}{4\lambda_3 t} \right)d \bar{\theta}' \\
&&+ \int_{0}^{t}\int_{\RR^{d+1}}  \frac{1}{(4\pi \lambda_3 (t-s))^{\frac{d+1}{2}} } \exp\left(  \frac{- \| \theta' -\bar{\theta}'\|^2}{4\lambda_3 (t-s)}\right)g_4(s, \bar{\theta}') d\bar{\theta}' d s\notag\\
&\overset{\eqref{g_4b}}{\geq}& \frac{1}{(4\pi \lambda_3 t)^{\frac{d+1}{2}} } \int_{\RR^{d+1}} \hp_0(\bar{\theta}') \exp\left(  \frac{- \| \theta' -\bar{\theta}'\|^2}{4\lambda_3 t} \right)d \bar{\theta}' \notag\\
&&- \int_{0}^{t}\int_{\RR^{d+1}}  \frac{1}{(4\pi \lambda_3 (t-s))^{\frac{d+1}{2}} } \exp\left(  \frac{- \| \theta' -\bar{\theta}'\|^2}{4\lambda_3 (t-s)}\right)(A\|\bar{\theta}' \|^2+B) d\bar{\theta}' d s\notag\\
&=&\frac{1}{(4\pi \lambda_3 t)^{\frac{d+1}{2}} } \int_{\RR^{d+1}} \hp_0(\bar{\theta}') \exp\left(  \frac{- \| \theta' -\bar{\theta}'\|^2}{4\lambda_3 t} \right)d \bar{\theta}'- (A \| \theta'\|^2 +B)t - A \lambda_3 t^2.\notag
\end{eqnarray}
So for all $\theta'$ and $t \geq 0$, we have $\hp_t(\theta')>-\infty$, which implies that $p_t(\theta')>0$.
\end{proof}

\subsection{Proof of  Monotonously Non-increasing of $Q'(p_t)$}
\begin{lemma}\label{descent}
Under Assumption \ref{ass:h2}, \ref{ass:h3}, and \ref{ass:h4}, suppose  $p_t$ evolves as \eqref{evo p}, we have 
\begin{eqnarray}
\frac{\partial Q'(p_t)}{\partial t}=
 - \int_{\RR^{d+1}} p_t(\theta, u )\left[\left\| g_2(t, \theta, u)-\lambda_3 \frac{\nabla_{\theta}p(\theta, u) }{p_t(\theta, u )}\right\|^2 + \left| g_1(t, \theta, u)-\lambda_3  \frac{\nabla_{u}p(\theta, u) }{p_t(\theta, u )}\right|^2\right] d\theta du.\notag
\end{eqnarray}
\end{lemma}

\begin{proof}[Proof of Lemma \ref{descent}]
Observe that \eqref{evo p} actually performs Wasserstein Gradient Flow on $Q'(p)$ \citep{MeiE7665}. By the standard analysis,  we can show that the $Q'(p)$ is monotonous non-increasing.  Let
\begin{eqnarray}
&&g_p(\theta, u)\\
&=& \frac{\partial Q'(p)}{\partial p}\notag\\
&=& \EE_{x,y}\left[\nabla_{y'} \phi\left(\int_{\RR^{d+1}} u h'(\theta,x)p(\theta, u)d\theta du, y\right) u h'(\theta,x)\right]+ \frac{\lambda_1\|\theta\|^2}{2}+\frac{\lambda_2\|u\|^2}{2} + \lambda_3\mathrm{ln}(p(\theta,u)) +\lambda_3.\notag
\end{eqnarray}
We can check from \eqref{evo p} that 
\begin{eqnarray}
\frac{d p_t(\theta, u)}{dt} = \nabla\cdot [p_t(\theta,u) \nabla g_{p_t}(\theta, u) ]. 
\end{eqnarray}
So by the chain rule, we have
\begin{eqnarray}
&&\frac{\partial Q'(p_t)}{\partial t}\\
&=& \int_{\RR^{d+1}}  \frac{\partial Q'(p_t)}{\partial p_t}\frac{d p_t}{d t}d\theta d u \notag\\
&=&\int_{\RR^{d+1}} g_{p_t}(\theta, t) \nabla\cdot [p_t(\theta,u) \nabla g_{p_t}(\theta, u) ]d \theta du \notag\\
&=& - \int_{\RR^{d+1}} p_t(\theta,u)\left\|\nabla g_{p_t}(\theta, u) \right\|^2 d\theta du \notag\\
&=& - \int_{\RR^{d+1}} p_t(\theta, u )\left\| g_2(t, \theta, u)-\lambda_3  \frac{\nabla_{\theta}p(\theta, u) }{p_t(\theta, u )}\right\|^2 d\theta du -  \int_{\RR^{d+1}} p_t(\theta, u )\left| g_1(t, \theta, u)-\lambda_3 \frac{\nabla_{u}p(\theta, u) }{p_t(\theta, u )}\right|^2 d\theta du,\notag
\end{eqnarray}
which is the desired result of Lemma \ref{descent}.
\end{proof}

\subsection{Proof of  Theorem \ref{conver1}: Convergence of NGD}\label{proof: conver}
\begin{proof}[Proof of Theorem \ref{conver1}]
In the proof, we still use $\theta'$ to denote $[\theta, u]$.  
First, before going into the detailed proofs, we provide the proof sketch below.
\begin{enumerate}[Step 1.]
\item   We  prove that $\E_{p_t}\| \theta'\|^2\leq B_M$ for all $t\geq 0$, where $B_M$ is a finite constant. \label{th 1s}
\item \label{th 2s}  From Step \ref{th 1s},    the second moment of
  $p_t(\theta')$ is uniformly bounded by $B_M$.   So  $p_t(\theta')$
  is uniformly tight.  Thus there exists a $p_\infty$  and a
  subsequence $p_{k}$ with $k\to \infty$,  $p_{k}$ that converges
  weakly to $p_\infty$. Let $$\psi_p(\theta,u) = \EE_{x,y}\left[\nabla_{y'} \phi\left(\int_{\RR^{d+1}} u h'(\theta,x)p(\theta')d\theta', y\right) u h'(\theta,x)\right]+ \frac{\lambda_1\|\theta\|^2}{2}+\frac{\lambda_2 |u|^2}{2}.$$ We prove
$$\lim_{k\to \infty}\int_{\RR^{d+1}}\|\nabla \psi_{p_k}(\theta,u) -  \nabla\psi_{p_\infty}(\theta,u)\|^2p_k(\theta')d\theta' = 0.$$ 
\item \label{th 3s} We further prove 
\begin{eqnarray}\label{expall}
\lim_{k\to \infty} \int_{\RR^{d+1}} \left|[p_k(\theta')\exp(\psi_{p_\infty}(\theta, u)/\lambda_3)]^{\frac{1}{2}} -c_k\right|^2 G(\theta') d\theta' = 0,  
\end{eqnarray}
where 
$$  G(\theta') = T_G\exp\left[- \frac{ \lambda_1}{2\lambda_3}| u|^2- \frac{\lambda_2}{2\lambda_3}\| \theta\|^2\right], $$
  $T_G = \left(\sqrt{\frac{\lambda_1}{2\pi \lambda_3}}\right)\left(\sqrt{\frac{\lambda_2}{2\pi \lambda_3}}\right)^d$ is a constant for normalization, and $c_k$ is bounded for all $k\geq 0$.
\item  \label{th 4s}  Because $c_k$ is  bounded,  we can take a sub-sequence $t_k$ with $\lim_{t_k\to \infty} c_k = c_{\infty}$. Then
\begin{eqnarray}\label{sub}
\lim_{t_k\to \infty} \int_{\RR^{d+1}} \left|\left[p_k(\theta')(\exp{\psi_{p_\infty}(\theta,u)/\lambda_3)}\right]^{1/2} -c_\infty\right|^2 G(\theta') d\theta' = 0. 
\end{eqnarray}
Furthermore, because $ [p_k(\theta')\exp{\psi_{p_\infty}(\theta,u)/\lambda_3}]^{1/2} $  converges to $c_\infty$ in measure, there exists a sub-sequence $\tau_k$ such  that 
\begin{eqnarray}
\lim_{\tau_k \to \infty} \left[p_{\tau_k}(\theta')\exp(\psi_{p_\infty}(\theta,u)/\lambda_3)\right]^{1/2} -c_\infty = 0,\ a.e.
\end{eqnarray}
It follows that
\begin{eqnarray}
\lim_{\tau_k \to \infty} p_{\tau_k}(\theta')=c_\infty^2 \exp(-\psi_{p_\infty}(\theta,u)/\lambda_3),\  a.e.
\end{eqnarray}
Let $\tilde{p}_{\infty} = c_\infty^2\exp(-\phi_{p_\infty}(\theta,u)/\lambda_3)$. We prove $p_{\infty}= \tilde{p}_{\infty}$, a.e.
\item \label{th 5s}   Finally, we prove $p_{\infty}= \tilde{p}_{\infty} = p_*$, a.e and $\lim_{t\to\infty} Q(p_t) \to Q (p_*) $.
\end{enumerate}

\begin{itemize}
\item   Proof of Step \ref{th 1s}:

From Lemma 10.1 of \cite{MeiE7665} (also see the Background \ref{bound enro}),  for any $z> 0 $. the entropy can be bounded as   
\begin{eqnarray}
-\int_{\RR^{d+1}} p(\theta')\mathrm{ln}(p(\theta'))d\theta' \leq  1+ \E_{p}\| \theta'\|^2/z + (d+1) \mathrm{ln}(2\pi z).
\end{eqnarray}
 Let $z = \frac{4\lambda_3}{\min(\lambda_1, \lambda_2)}$. We have
\begin{eqnarray}\label{bound second}
\lambda_3\int p_t(\theta')\mathrm{ln}(p_t(\theta'))d\theta' \geq -\lambda_3 - \frac{\min(\lambda_1, \lambda_2)}{4}\E_{p_t}\| \theta'\|^2 - \lambda_3(d+1)\mathrm{ln}(2\pi z).
\end{eqnarray}
We obtain
\begin{eqnarray}
&&\int_{\RR^{d+1}} \frac{\min(\lambda_1, \lambda_2)}{4} \|\theta'\|^2p_t(\theta')d\theta'\\
&\overset{\eqref{bound second}}{\leq}&\int \frac{\min(\lambda_1, \lambda_2)}{2} \|\theta'\|^2p_t(\theta')d\theta' +\lambda_3\int_{\RR^{d+1}} p_t(\theta')\mathrm{ln}(p_t(\theta'))d\theta'  +\lambda_3 + \lambda_3(d+1)\mathrm{ln}(2\pi z) \notag\\
&\overset{\phi\geq B_l}{\leq}& Q'_t  +\lambda_3 + \lambda_3(d+1)\mathrm{ln}(2\pi z) -B_l\notag\\
&\overset{a}{\leq} &Q'_0  +\lambda_3 + \lambda_3(d+1)\mathrm{ln}(2\pi z)-B_l,\notag
\end{eqnarray}
where  $\overset{a}{\leq}$ uses Lemma \ref{descent} that $Q_t$ is non-increasing.  Set $B_M = 4(Q'_0  +\lambda_3 + \lambda_3(d+1)\mathrm{ln}(2\pi z)+B_l) /   \min(\lambda_1, \lambda_2) $, we have   the second moment of $p_t(\theta')$ is uniformly bounded by $B_M$.

\item    Proof of Step \ref{th 2s}:

First, by the definition of $\psi_p$, we have
\begin{eqnarray}
&&\int_{\RR^{d+1}}\left\|\nabla \psi_{p_k}(\theta,u) -  \nabla\psi_{p_\infty}(\theta,u)\right\|^2p_k(\theta')d\theta'\\
&=&\underbrace{\int_{\RR^{d+1}}\left\| \E_{x,y}\left[\left(\nabla_{y'} \phi\left(\varphi_x(p_k), y\right)- \nabla_{y'} \phi\left(\varphi_x(p_\infty), y\right)\right)\nabla_{\theta}h'(\theta,x) \right]\right\|^2 | u |^2 p_k(\theta')d\theta'}_{I_8}\notag\\
&&+\underbrace{\int_{\RR^{d+1}}\left| \E_{x,y}\left[\left(\nabla_{y'} \phi\left(\varphi_x(p_k), y\right)- \nabla_{y'} \phi\left(\varphi_x(p_\infty), y\right)\right)h'(\theta,x) \right]\right|^2p_k(\theta')d\theta'}_{I_9}, \notag
\end{eqnarray}
where we let $\varphi_x(p) = \int_{\RR^{d+1}} u h'(\theta,x) p(\theta')   d\theta'$.

For $I_8$, using $\|\nabla_{\theta} h'(\theta, x)\|\leq C_3$ and $\phi_{y'}(y',y)$ has $L_2$-Lipschitz continuous gradient on $y'$, we have
\begin{eqnarray}
&&\int_{\RR^{d+1}}\left\| \E_{x,y}\left[\left(\nabla_{y'} \phi\left(\varphi_x(p_k), y\right)- \nabla_{y'} \phi \left(\varphi_x(p_\infty) , y\right)\right)\nabla_{\theta}h'(\theta,x) \right]\right\|^2 | u 
|^2 p_k(\theta')d\theta'\notag\\
&\overset{a}{\leq}& \int_{\RR^{d+1}} \left(\E_{x,y}\left[\left| \nabla_{y'} \phi\left(\varphi_x(p_k), y\right)- \nabla \phi_{y'}\left(\varphi_x(p_\infty), y\right)\right|\left\|\nabla_{\theta}h'(\theta,x) \right\|\right]\right)^2 | u 
|^2 p_k(\theta')d\theta'\notag\\
&\leq& L^2_2C_3^2 \int_{\RR^{d+1}} \left(\E_{x,y}\left| \int_{\RR^{d+1}} u h'(\theta,x)[p_{k}(\theta') - p_\infty(\theta')]d\theta'\right|\right)^2 | u 
|^2 p_k(\theta')d\theta' \notag\\
&\overset{b}{\leq}& L^2_2C_3^2  B_M  \left( \E_{x,y}\left|\int_{\RR^{d+1}} u h'(\theta,x)[p_k(\theta') - p_\infty(\theta')]d\theta'\right|\right)^2,
\end{eqnarray}
where  in $\overset{a}\leq$, we use
$$ \|\EE [\xi_1\xi_2]\|^2 =  (\|\EE [\xi_1\xi_2]\|)^2\leq  (\EE \|[\xi_1\xi_2]\|)^2 \leq  (\EE [|\xi_1|\cdot\|\xi_2\|])^2,$$
with $\xi_1 = \nabla_{y'} \phi\left(\varphi_x(p_k), y\right)- \nabla_{y'} \phi \left(\varphi_x(p_\infty) , y\right) $ and $\xi_2= \nabla_{\theta} h'(\theta,x)$ and in $\overset{b}\leq$ we use the second moment of $p_k$ is bounded by $B_M$ proved in Step \ref{th 1s}.  Similarly, we can  bound $I_9$ as
\begin{eqnarray}
&&\int_{\RR^{d+1}}\left| \E_{x,y}\left[\left(\nabla_{y'} \phi\left(\varphi_x(p_k), y\right)- \nabla_{y'} \phi\left(\varphi_x(p_\infty), y\right)\right)h'(\theta,x) \right]\right|^2p_k(\theta')d\theta'\\
&\leq& L_2^2 \int_{\RR^{d+1}} \left(\E_{x,y}\left| \int_{\RR^{d+1}} u h'(\theta,x)[p_{k}(\theta') - p_\infty(\theta')]d\theta'\right|(C_1\|\theta\|+C_2)\right)^2  p_k(\theta')d\theta'\notag\\
&\overset{a}{\leq}&L_2^2   \left(\int_{\RR^{d+1}}  (2C^2_1 \|\theta' \|^2 +2C_2^2)p_k(\theta') d\theta'\right)  \left(\E_{x,y}\left| \int_{\RR^{d+1}} u h'(\theta,x)[p_k(\theta') - p_\infty(\theta')]d\theta'\right|\right)^2\notag\\
&\leq&2L_2^2(C_1^2B_M+C_2^2) \left(\E_{x,y}\left| \int_{\RR^{d+1}} u h'(\theta,x)[p_k(\theta') - p_\infty(\theta')]d\theta'\right|\right)^2,\notag
\end{eqnarray}
where $\overset{a}\leq$ uses 
$$ (C_1\|\theta\|+C_2)^2 \leq   2C^2_1\|\theta\|^2+ 2C_2^2\leq  2C^2_1\|\theta'\|^2+ 2C_2^2.  $$
Next, we show $$\lim_{k\to \infty}\E_{x,y}\left|\int_{\RR^{d+1}} u
  h'(\theta,x)[p_k(\theta') - p_\infty(\theta')]d\theta'\right| = 0.$$
According to the definition of limit,  we only need to  prove that for
any $\ep>0$, there exists a $K_0$ so that when $k\geq K_0$, we have $$ \E_{x,y}\left|\int u h'(\theta,x)[p_k(\theta') - p_\infty(\theta')]d\theta'\right|\leq \ep. $$  Let   $B = \max\left(1,\frac{16(4+3B_v C_1+B_vC_2)^2B_M^2}{9\ep^2}\right)$, we consider the test function class $f_x: \RR^{d+1}\to\RR$, which satisfies:
\begin{eqnarray}
f_x(\theta, u)=
\begin{cases}
uh'(\theta, x),& \| \theta\|^2+ |u|^2 \leq B^2,\\
\frac{uB}{R(R-B+1)}h'(\frac{B\theta}{R}x),& \| \theta\|^2+ |u|^2 = R^2 > B^2 .
\end{cases}
\end{eqnarray}
Therefore   $f_x$ is uniformly bounded and  is uniformly Lipschitz continuous for all $x$. 
Because $p_k$ weakly converges to $p_{\infty}$, we have
$d_{\text{BL}}(p_k, p_{\infty})\to 0 $ (refer to the definition in the
Background \ref{bl dis}) which indicates that  there exists $K_1$ which only depends on $B$ and constants $C_1, C_2, C_3$ such that  when $k\geq K_1$, for all $f_x$, we have
\begin{eqnarray}
\left|\int_{\RR^{d+1}} f_x(\theta, u) [p_k(\theta')-p_\infty(\theta')]  d\theta'\right|\leq \frac{\ep}{2}.
\end{eqnarray}
Thus let $\mathbb{B}$ be the ball $\{\theta' : \| \theta'\|\leq B \}$, we have
\begin{eqnarray}\label{temp2}
&&\E_{x,y}\left|\int_{\RR^{d+1}} u h'(\theta,x)[p
_k(\theta') - p_\infty(\theta')]d\theta'\right| \\
&\leq& \E_{x,y}\left|\int_{\RR^{d+1}} f_x(\theta, u)[p_k(\theta') - p_\infty(\theta')]d\theta'\right|\notag\\
&&+  \E_{x,y}\left|\int_{\mathbb{B}^c} [u h'(\theta,x)-f_x(\theta,u)][p_k(\theta') - p_\infty(\theta')]d\theta'\right|\notag\\
&\leq& \frac{\ep}{2}+ \E_{x,y}\left[\int_{\mathbb{B}^c}\left(|u h'(\theta,x) + |f_x(\theta,u)|\right) (p_k(\theta')+ p_\infty(\theta'))d\theta'\right]\notag\\
&\overset{a}{\leq}& \frac{\ep}{2}+ \E_{x,y}\left[\int_{\mathbb{B}^c}\left(|u h'(\theta,x)| + |f_x(\tilde{\theta},\tilde{u})|\right) (p_k(\theta')+ p_\infty(\theta'))d\theta'\right]\notag\\
&\overset{b}{\leq}&\frac{\ep}{2} + \E_{x,y}\int_{\mathbb{B}^c}\left(\frac{2}{3}|u|^{1.5}+ \frac{1}{3}|h'(\theta,x)|^3 + \frac{2}{3}|\tilde{u}|^{1.5}+ \frac{1}{3}|h'(\tilde{\theta},x)|^3\right) [p_k(\theta')+ p_\infty(\theta')]d\theta'\notag\\
&\leq&\frac{\ep}{2} + \int_{\mathbb{B}^c}\left(\frac{2}{3}|u|^{1.5}+ \E_{x,y}\left[\frac{1}{3}|h'(\theta,x)|^3\right] + \frac{2}{3}|\tilde{u}|^{1.5}+ \E_{x,y}\left[\frac{1}{3}|h'(\tilde{\theta},x)|^3\right]\right) [p_k(\theta')+ p_\infty(\theta')]d\theta'\notag\\
&\overset{c}{\leq}&\frac{\ep}{2} + \int_{\mathbb{B}^c}\left(\frac{4}{3}|u|^{1.5}+ \frac{2}{3}B_v(C_1\| \theta\|+C_2) \right) (p_k(\theta')+ p_\infty(\theta'))d\theta'\notag,
\end{eqnarray}
where in $\overset{a}\leq$, we set $\tilde{\theta} = B\theta/R$ and $\tilde{u} = u\theta/R$, and uses $|f_x(\theta, u)| \leq |f_x(\tilde{\theta}, \tilde{u})| $ on $\mathbb{B}^c$, in $\overset{b}\leq$, we use Young's inequality, i.e.
\begin{equation}
| u h'(\theta, x) |\leq \frac{2}{3}|u |^{1.5} + \frac{1}{3}|h'(x, \theta) |^3,
\end{equation}
and in $\overset{c}\leq$, we use the fact that $|h'(\theta,x)|\leq C_1\| \theta\|+C_2$ from Assumption \ref{ass:h3} and $\E_{x}\| h'(\theta,x)\|\leq B_v$ from Assumption \ref{ass:h1}, then
\begin{eqnarray}
\E_{x,y}\left[\frac{1}{3}|h'(\theta,x)|^3\right] \leq  \frac{(C_1\| \theta\|+C_2)}{3} \E_{x,y}|h'(\theta,x)|^2 \leq   B_v(C_1\| \theta\|+C_2)/3,
\end{eqnarray}
and
\begin{eqnarray}
\E_{x,y}\left[\frac{1}{3}|h'(\tilde{\theta},x)|^3\right] \leq   B_v(C_1\| \tilde{\theta}\|+C_2)/3 \leq  B_v(C_1\| \theta\|+C_2)/3.
\end{eqnarray}

Furthermore, for all $\theta'\in \mathbb{B} $, we have
\begin{eqnarray}\label{temp1}
&&\frac{4}{3}|u|^{1.5}+ \frac{2}{3}B_v(C_1\| \theta\|+C_2)\\
&\leq&   \frac{4}{3}|u|^{1.5}+ \frac{2}{3}B_v(C_1\| \theta\|^{1.5}+C_2+C_1)\notag\\
&\leq&  \frac{4+2 B_v C_1}{3}\left(|u|^{1.5}+\| \theta\|^{1.5}\right)+\frac{2B_v(C_2+C_1)}{3}\notag\\
&\overset{a}\leq& \frac{2^{\frac{1}{4}}(4+2 B_v C_1)}{3} \|\theta' \|^{1.5} +  \frac{2B_v(C_2+C_1)}{3}\notag\\
&\overset{b}\leq& \frac{8+4 B_v C_1}{3} \|\theta' \|^{1.5} +  \frac{2B_v(C_2+C_1)}{3}\| \theta'\|^{1.5}\notag\\
&\overset{c}\leq&   \frac{8+6 B_v C_1+2 B_v C_2}{3} \frac{\|\theta' \|^2}{\sqrt{B}}.\notag
\end{eqnarray}
In $\overset{a}\leq$, we consider function $q(x) = |x|^{3/4}$ which is concave. So 
$$ q\left(\frac{x+y}{2}\right)\geq \frac{q(x)+q(y)}{2},    $$
then $\overset{a}\leq$ is obtained by setting $x = \| \theta\|^2$ and $y  = | u|^2$. $\overset{b}\leq$ uses $2^{1/4}\leq 2$ and $ 1\leq \|\theta' \|^{1.5}$ since $\|\theta' \|\geq B\geq1$. $\overset{c}\leq$ uses $\|\theta' \|^{0.5}\geq B^{0.5}$.  Plugging \eqref{temp1} into \eqref{temp2}, we have
\begin{eqnarray}
&&\E_{x,y}\left|\int_{\RR^{d+1}} u h'(\theta,x)[p
_k(\theta') - p_\infty(\theta')]d\theta'\right| \\
&\leq&\frac{\ep}{2}+ \frac{8+6 B_v C_1+2 B_v C_2}{3\sqrt{B}} \int_{\mathbb{B}^c} \|\theta'\|^2 (p_k(\theta')+ p_\infty(\theta'))d\theta'\notag\\
&\leq&\frac{\ep}{2} + \frac{(16+12 B_v C_1+4 B_v C_2)B_M}{3\sqrt{B}}\notag\\
&\leq&\ep, \notag
\end{eqnarray}
where in the second equality we use the fact that the second moment of
$p_k$ is bounded by $B_M$ from Step \ref{th 1s}, and so does
$p_{\infty}$. We obtain  
\[
\lim_{k\to \infty}\int_{\RR^{d+1}}\|\nabla \psi_{p_k}(\theta,u) -
\nabla\psi_{p_\infty}(\theta,u)\|^2p_k(\theta')d\theta' = 0 .
\]

\item  Proof of Step \ref{th 3s}:

From Lemma \ref{descent}, $Q'_t$ is non-increasing, so 
\begin{eqnarray}
\lim_{k\to\infty}\int_{\RR^{d+1}}  \|\nabla(\psi_{p_k}(\theta, u)+ \lambda_3 \mathrm{ln}p_k(\theta,u) ) \|^2p_k(\theta')d\theta' = 0.
\end{eqnarray}
From Step \ref{th 2s}, we have
\begin{eqnarray}
\lim_{k\to\infty}\int_{\RR^{d+1}}  \|\nabla(\psi_{p_\infty}(\theta, u)+\lambda_3 \mathrm{ln} p_k(\theta,u)  )\|^2 p_k(\theta')d\theta' = 0.
\end{eqnarray}

Following from by \cite{MeiE7665}, we have
\begin{eqnarray}
&&\int_{\RR^{d+1}}  \|\nabla(\psi_{p_\infty}(\theta, u)+\lambda_3 \mathrm{ln} p_k(\theta,u)  )\|^2 p_k(\theta')d\theta'\\
&=&\lambda_3^2\int_{\RR^{d+1}}\|\nabla(p_k(\theta, u) \exp( \psi_{p_\infty}(\theta,u)/\lambda_3))\|^2 p^{-1}_k(\theta')\exp(-2\psi_{p_\infty}(\theta,u)/\lambda_3)d\theta'\notag\\
&=&4\lambda_3^2\int_{\RR^{d+1}}  \|\nabla [p_k(\theta, u) \exp( \psi_{p_\infty}(\theta,u)/\lambda_3)]^{1/2}\|^2 \exp(-\psi_{p_\infty}(\theta,u)/\lambda_3)d\theta'.\notag
\end{eqnarray}
On the other hand, according to the definition of $\psi_p$, we have 
\begin{eqnarray}\label{exp}
\psi_{p_{\infty}} (\theta, u) &\overset{a}{\leq}& \frac{\lambda_1}{2}\| \theta\|^2+ \frac{\lambda_2 }{2}|u|^2 +	L_1 \E_{x,y}\int_{\RR^{d+1}} |u h'(\theta,x) |p_{\infty}(\theta')d\theta'\\
&\overset{b}{\leq}&\frac{\lambda_1}{2}\| \theta\|^2+ \frac{\lambda_2}{2}|u|^2 +	L_1\int_{\RR^{d+1}} \left(\frac{|u |^2}{2}  + C_1\| \theta\|^2+C_2\right) p_{\infty}(\theta')d\theta'\notag\\
&\overset{c}{\leq}&\frac{\lambda_1}{2}\| \theta\|^2+ \frac{\lambda_2}{2}|u|^2 +	L_1\left[C_1+0.5\right]B_M+L_1 C_2,\notag
\end{eqnarray}
where  $\overset{a}{\leq}$ uses Assumption \ref{ass:h2} that $|\nabla_{y'} \phi|\leq L_1$, $\overset{b}\leq$ uses $|u h'(\theta, x)| \leq \frac{|u|^2}{2} + \frac{h'(\theta,x)^2}{2}\leq  \frac{|u|^2}{2}+ \frac{(C_1\|\theta\|+C_2)^2}{2}\leq \frac{|u|^2}{2}+ C_1^2\|\theta\|^2+C_2^2$,
and $\overset{c}{\leq}$ uses that the second moment of $p_{\infty}$ is bounded by $B_M$.  \eqref{exp} shows that there is a constant $C_g$ such that for all  $\theta$ and $u$, we have 
\begin{eqnarray}\label{cg}
\exp(-\psi_{p_{\infty}}(\theta, u)/\lambda_3) \geq C_G G(\theta'),
\end{eqnarray}
where $C_G$ is a constant. Then
\begin{eqnarray}\label{exp1}
\lim_{k\to \infty}4\lambda_3^2 C_G\int_{\RR^{d+1}}  \|\nabla [p_k(\theta, u) \exp( \psi_{p_\infty}(\theta,u)/\lambda_3)]^{1/2}\|^2 G(\theta')d\theta' = 0.
\end{eqnarray}

On the other hand,  because $G(\theta')$ is a Gaussian distribution,  it satisfies the Poincare inequality. That is
\begin{eqnarray}\label{exp5}
 &&\int_{\RR^{d+1}}  \left\|\nabla [p_k(\theta, u) \exp( \psi_{p_\infty}(\theta,u)/\lambda_3)]^{1/2}\right\|^2 G(\theta')d\theta'\notag\\
 &\geq& K   \int_{\RR^{d+1}}  \left|[p_k(\theta, u) \exp( \psi_{p_\infty}(\theta,u)/\lambda_3)]^{1/2} -c_k\right|^2 G(\theta')d\theta',
\end{eqnarray}
where $K$ is a constant and 
$$c_k =  \int_{\RR^{d+1}} [p_k(\theta, u) \exp( \psi_{p_\infty}(\theta,u)/\lambda_3)]^{1/2} G(\theta')d\theta'.$$
Plugging \eqref{exp5} into \eqref{exp1}, we obtain \eqref{expall}. We then prove $c_k$ is bounded for all $k\geq 0$. First, it is obvious that $c_k\geq 0$. To obtain that $c_k$ is upper bounded,  we consider
\begin{eqnarray}\label{c bound}
&&\int_{\RR^{d+1}} (p_k(\theta, u) \exp( \psi_{p_\infty}(\theta,u)/\lambda_3))^{1/2} G(\theta,u)d\theta'\\
&\overset{\eqref{cg}}\leq& \int_{\RR^{d+1}} p_k(\theta, u)^{1/2}  G(\theta')^{1/2}d\theta'\notag\\
&\overset{a}{\leq}&  \sqrt{\int_{\RR^{d+1}} p_k(\theta, u) d\theta' \int_{\RR^{d+1}} G(\theta') d\theta'}\notag\\
&=&1,\notag
\end{eqnarray}
where $\overset{a}\leq$ uses Cauchy inequality with $\< a,b\>\leq \|a\|\|b\|$ with $\< a, b\> = \int_{\RR^{d+1}} a(\theta')b(\theta') d\theta'$ and $\|a\|=\sqrt{\<a,a\>}$.
 With \eqref{c bound} in hand, $c_k$ is upper  bounded.

\item Proof of Step \ref{th 4s}:

We prove the argument by contradiction. Define the set $E = \{ \theta | p_{\infty}(\theta')> \tilde{p}_{\infty }(\theta')\}$. Without loss of generality, we can assume $m(E)= \ep_1>0$. On the other hand, because $\lim_{\tau_k \to \infty} p_{\tau_k}(\theta') = \tilde{p}_{\infty}(\theta')$, a.e.  There exists a set $S_1\subseteq E$ with $m(S_1)>0 $ such that $\lim_{\tau_k\to \infty}\sup_{\theta'\in S}|\tilde{p}_{\infty}(\theta') -\tilde{p}_{\tau_k}(\theta')  | = 0$. This implies that
\begin{eqnarray}\label{con11}
\lim_{\tau_k\to \infty} \int_{\RR^{d+1}} I_{\dot{S}_1}(\theta') p_{\tau_k}(\theta')d\theta' = \int_{\RR^{d+1}} I_{\dot{S}_1}(\theta') \tilde{p}_{\infty}(\theta')d\theta,
\end{eqnarray} 
where $I_{\dot{S}_1}$ is the indicator function of $\dot{S}_1$.

Then we consider the test function
\begin{eqnarray}
I(\theta, u)=
\begin{cases}
0,&   \bar{S}_1^c,\\
0,&   \bar{S}_1\setminus \dot{S}_1, \\
1,&  \dot{S}_1 .
\end{cases}
\end{eqnarray}
We have $m(\bar{S}_1\setminus \dot{S}_1) = 0$ and on the set $
\bar{S}_1^c$ and $\dot{S}_1$, $I(\theta, u)$ is continuous. From the
fact that $p_{\tau_k}$ weakly converges to $p_{\infty}$ and
$p_{\infty}$ is absolutely continuous with respect to the Lebesgue measure since
$Q'(p_{\infty})\leq Q'(p_0)<+\infty$, we have
\begin{eqnarray}\label{con22}
\lim_{\tau_k \to \infty} \int_{\RR^{d+1}} I_{\dot{S}_1}(\theta') p_{\tau_k}(\theta')d\theta' = \int_{\RR^{d+1}} I_{\dot{S}_1}(\theta') p_{\infty}(\theta')d\theta'.
\end{eqnarray} 
Combining  \eqref{con22} and \eqref{con11}, we have
\begin{eqnarray}
 0  = \int_{\RR^{d+1}} I_{\dot{S}_1}(\theta') (p_{\infty}(\theta') - \tilde{p}_{\infty}(\theta'))d\theta' = m(S_1) >0.
\end{eqnarray}
This is a contradictory.  Therefore we obtain $p_{\infty}= \tilde{p}_{\infty}$, a.e.

\item Proof of Step \ref{th 5s}:

Since $\int_{\RR^{d+1}} p_{\infty} (\theta')d\theta' = 1$, we have $\int_{\RR^{d+1}} \tilde{p}_{\infty}(\theta') d\theta' = 1$. Using the definition that  $\tilde{p}_{\infty} = c_\infty^2\exp(-\phi_{p_\infty}(\theta,u)/\lambda_3)$, we have
\begin{eqnarray}
c_\infty^2 = \left[\int_{\RR^{d+1}} \exp(-\phi_{p_\infty}(\theta,u)/\lambda_3)  d\theta'\right]^{-1},
\end{eqnarray}
which indicates that 
\begin{eqnarray}
p_{\infty}(\theta,u) =
  \frac{\exp(-\phi_{p_\infty}(\theta,u)/\lambda_3) }{ \int_{\RR^{d+1}}
  \exp(-\phi_{p_\infty}(\theta,u)/\lambda_3)  d\theta' } .
\end{eqnarray}

This implies that $p_{\infty} =\tilde{p}_{\infty} = p_{*}$, a.e.  On the other hand, because $Q(p_t)$ is monotone non-increasing, we have $\lim_{t\to\infty}Q(p_t) = Q(p_{*})$.  Because  $p_{*}$ is the unique minimum solution, 
all the  weakly converging   sub-sequences converge to  $p_{*}$. Therefore $p_t$ converges weakly to $p_*$. 
This finishes the proof. 
\end{itemize}
\end{proof}

\subsection{Proof of  Theorem \ref{optimal fea}: Target Feature Repopulation}
\begin{proof}[Proof of Theorem \ref{optimal fea} ]
In the proof, let $p(u| \theta) = p(\theta, u)/\rho(\theta) $. The optimization problem \eqref{problem2} can be rewritten as follows when $p > 0$: 
\begin{eqnarray}\label{problem 3}
&\minimize_{\rho, \omega, p(\omega|\theta)}& \quad Q'' (\rho, \omega, p(u| \theta))\\
&\mathrm{s.t.}& \omega(\theta)= \int_{\RR} u p(u| \theta)d u , \quad \forall \theta\in \RR^d\notag\\
&& \int_\RR p(u| \theta) du  = 1, \quad \forall \theta\in \RR^d\notag\\
&& \int_{\RR^d} \rho(\theta)d\theta  = 1,\notag\\
&&  \rho(\theta) \geq 0,\notag\\
&& p(u| \theta) \geq 0, \notag
\end{eqnarray}
where 
\begin{eqnarray}
&&Q''(\rho, \omega, p(u| \theta))\\
 &=& \E_{x,y}\phi\left(\int_{\RR^d} \omega(\theta) h'(\theta, x) \rho(\theta) d\theta, y   \right) +\frac{\lambda_2}{2}\int_{\RR^d} \| \theta\|^2 \rho(\theta) d\theta + \frac{\lambda_1}{2}\int_{\RR^d}  | \omega(\theta)|^2\rho(\theta) d\theta\notag\\
&&+\frac{\lambda_1}{2}\int_{\RR^d} \rho(\theta)p(u| \theta) |u - \omega(\theta) |^2 d\theta du   + \frac{\lambda_3}{2} \int_{\RR^d}  \rho(\theta) \mathrm{ln} \rho(\theta) d\theta +  \frac{\lambda_3}{2}\int_{\RR^d}  p(u| \theta) \rho(\theta)  \mathrm{ln} (p(u| \theta))d u d\theta.\notag
\end{eqnarray}
Because $p_*$ is the minimal solution of $Q'$,   $(\rho_*, \omega_*, p_*(u|\theta))$ is at least a local stationary point. In other words, $(\rho_*, \omega_*, p_*(u|\theta))$ satisfies the KKT condition for \eqref{problem 3}. Let  $\rho_*(\theta)B_3(\theta)$ be the multiplier for the constraint $\omega(\theta)= \int u p(u| \theta)d u$,   $\rho_*(\theta)B_4(\theta)$   be the multiplier for $\int p(u| \theta) du  = 1$, and $B_5$ be the multiplier for  constrain for $\int \rho(\theta)d\theta  = 1$. Because $\rho_*(\theta)>0$ and $p_*(u| \theta)>0$, so the multipliers for the two constrains are $0$ by the complementary slackness.  We have, for $\rho_*$,
\begin{eqnarray}\label{u -1}
0 &=& \E_{x,y}[\phi'_* h'(\theta, x)]\omega_*(\theta) + \frac{\lambda_1  |\omega_*(\theta)|^2}{2} + \frac{\lambda_2\|\theta\|^2}{2} + \frac{\lambda_3}{2}[\mathrm{ln} \rho_*(\theta) +1]\notag\\
&& + \frac{\lambda_1}{2}\int_{\RR} p_*(u|\theta)|u-\omega_*(\theta) |^2 d u + \frac{\lambda_3}{2}\int_{\RR} p_*(u| \theta)\mathrm{ln} ( p_*(u| \theta))d u + B_5,
\end{eqnarray}
where $\phi'_*(x,y) = \nabla_{y'}\phi( \int_{\RR^{d+1}} uh'(\theta, x) p_*(\theta, u) d\theta du, y  )$,  for $\omega_*$,
\begin{eqnarray}\label{u 0}
0 &=& \E_{x,y}[\phi'_* h'(\theta, x)]\rho_*(\theta) + \lambda_1  \rho_*(\theta)\omega_*(\theta) +\rho_*(\theta)B_3(\theta)=0,
\end{eqnarray}
and for $p_*(u|\theta)$,
\begin{eqnarray}\label{u_1}
0 =\frac{\lambda_1 \rho_*(\theta)|u -\omega_*(\theta)|^2}{2} + \frac{\lambda_3\rho_*(\theta)}{2}[\mathrm{ln}(p_*(u | \theta)+1)] +  \rho_*(\theta)B_4(\theta)  - \rho_*(\theta) B_3(\theta)u.
\end{eqnarray}
Diving $\rho_*(\theta)$ on both sides of \eqref{u_1} and using $\int p_*(u|\theta) d u = 1$, we have
\begin{eqnarray}\label{u_2}
p_*(u|\theta) = \frac{\exp[-\frac{B_3(\theta) u}{\lambda_3} - \frac{\lambda_1}{2\lambda_3}|u-\omega_*(\theta)|^2 ]}{ B_4(\theta)},
\end{eqnarray}
where $B_4(\theta) = \int_\RR \exp[-\frac{B_3(\theta) u}{\lambda_3} - \frac{\lambda_1}{2\lambda_3}|\omega_*(\theta)-u|^2 ]d u$ is  a constant for normalization. Given $\theta$, $p(u|\theta)$ is a Gaussian distribution. Because $\omega_*(
\theta) = \int_{\RR} u p_*(u|\theta)d u$. We have $B_3(\theta) = 0$. Then plugging $B_3(\theta) = 0$  into  \eqref{u 0}, we have
\begin{eqnarray}\label{so w}
\omega_*(\theta) = - \frac{\E_{x,y}[\phi'_* h'(\theta, x)]}{\lambda_1}.
\end{eqnarray}
 From \eqref{u_2}, $p_*(u|\theta) = N(\omega_*(\theta), \frac{\lambda_3}{\lambda_1} )$. Using calculus, we obtain
\begin{eqnarray}\label{u 5}
\frac{\lambda_1}{2}\int_\RR p_*(u|\theta)|u- \omega_*(\theta) |^2 d u + \frac{\lambda_3}{2}\int_\RR p_*(u| \theta)\mathrm{ln} ( p_*(u| \theta))d u = \frac{\lambda_3}{2} - \frac{\lambda_3}{4}\mathrm{ln}\left(2\pi e \frac{\lambda_3}{\lambda_1}\right).
\end{eqnarray}
Then plugging \eqref{u 5} into \eqref{u -1}, we have
\begin{eqnarray}\label{588}
0 &=& \E_{x,y}[\phi'_* h'(\theta, x)]\omega_*(\theta) + \frac{\lambda_1 |\omega_*(\theta)|^2}{2} + \frac{\lambda_2\|\theta\|^2}{2} + \frac{\lambda_3}{2}[\mathrm{ln} \rho_*(\theta) +1]\notag\\
&& +\frac{\lambda_3}{2} - \frac{\lambda_3}{4}\mathrm{ln}\left(2\pi e \frac{\lambda_3}{\lambda_1}\right)  + B_5.
\end{eqnarray}
Let $B_6 = \frac{\lambda_3}{2} - \frac{\lambda_3}{4}\mathrm{ln}\left(2\pi e \frac{\lambda_3}{\lambda_1}\right)  + B_5$ and plug  \eqref{so w} into \eqref{588}, we obtain
\begin{eqnarray}\label{rho0}
0 &=& - \frac{\lambda_1 }{2}| \omega_*(\theta)|^2 +  \frac{\lambda_2}{2}\|\theta \|^2 + \frac{\lambda_3}{2}[\mathrm{ln} \rho_*(\theta) +1] + B_6.
\end{eqnarray}
On the other hand,  from \eqref{p res} and $p(u|\theta)=N(\omega_*(\theta), \lambda_3/\lambda_1)$, we have
\begin{eqnarray}\label{rhou}
\rho_*(\theta) = \frac{\exp\left(-\frac{\lambda_2}{2\lambda_3}\| \theta\|^2 - \frac{\lambda_1}{2\lambda_3}| \omega_*(\theta) |^2\right)}{B_7},
\end{eqnarray}
where $B_7$ is a constant for normalization.
From \eqref{so w}, we have
\begin{eqnarray}
| \omega_*(\theta)|\leq \frac{L_2\sqrt{\EE_x| h'(\theta,x)|^2}}{\lambda_1}\leq \frac{L_2\sqrt{B_v}}{\lambda_1}.
\end{eqnarray}
So $B_7$ is finite.  Plugging \eqref{rhou} into \eqref{rho0}, we have
\begin{eqnarray}
0 = -\frac{\lambda_1 }{2}| \omega_*(\theta)|^2 + \frac{\lambda_2}{2}\|\theta \|^2- \frac{\lambda_2}{4}\| \theta\|^2  - \frac{\lambda_1 }{4}| \omega_*(\theta)|^2+ \frac{3\lambda_1}{4} B_8,
\end{eqnarray}
where $B_8 = \frac{4}{3\lambda_1} (B_6+ \frac{\lambda_3}{2}- \frac{\lambda_3}{2}\mathrm{ln}B_7)$ is a constant.
For $\mu_*(\theta)  = \rho_*(\theta)\omega_*(\theta)$, we have
\begin{eqnarray}\label{endthe}
0 = -\left| \frac{\mu_*(\theta)}{\rho_*(\theta)}\right|^2 + \frac{\lambda_2}{3\lambda_1}\|\theta \|^2 + B_8.
\end{eqnarray}
\end{proof}

\subsection{Proof of  Corollary \ref{optimal fea2}: Near Optimal Feature Representation}
\begin{proof}[Proof of Corollary \ref{optimal fea2}]
Solving \eqref{endthe},  we obtain 
\begin{eqnarray}
\rho_*(\theta) = \frac{ |\mu_*(\theta)|}{\sqrt{\frac{\lambda_2}{3\lambda_1}\| \theta\|^2 +B_8} }.
\end{eqnarray}
In the ball of $\{\theta : \| \theta\|^2\leq M\}$, we have
\begin{eqnarray}
\frac{|\mu_*(\theta) |}{B_8+ \frac{\lambda_2}{3\lambda_1}M} \leq \rho_*(\theta) \leq  \frac{|\mu_*(\theta) |}{B_8}. 
\end{eqnarray}
So when $\lambda_1>0$ and  $\lambda_2\to 0$, we can simply set $M = \frac{1}{\sqrt{\lambda_2}}\to \infty$, and $\frac{\lambda_2}{\lambda_1}M\to 0$, because $B_8$ is used for normalization which will not go to $0$ (but go to $C_*$),  we have 
\begin{eqnarray}
\rho_*(\theta) \to \frac{|\mu_*(\theta) |}{C_*}.
\end{eqnarray}
 $C_*=\int_{\RR^d} |\mu_*(\theta) |d\theta$.  We obtain our final result.
\end{proof}

\subsection{Proof of the Lower Bound for $V(\tilde{\mu}_*, \rho_N)$}\label{vv pro}
\begin{proof}
By symmetry, we have
\begin{eqnarray}\label{en2}
&&V(\tilde{\mu}_*, \rho_N)\\
&=& 2\int_{1}^{1+a} \left|\frac{\tilde{\mu}_*(\theta)}{\rho_N(\theta)} \right|^2\rho_N(\theta)d\theta\notag\\
&=&2\int_{1}^{1+a} \frac{1}{4a^2} \sqrt{2\pi}\sigma \exp\left( \frac{\theta^2}{2\sigma^2} \right)d\theta\notag\\
&\geq&  2a \cdot  \frac{1}{4a^2}  \sqrt{2\pi}\sigma \exp\left( \frac{4}{2\sigma^2} \right),\notag
\end{eqnarray}
where  last inequality uses  $\exp\left(
  \frac{\theta^2}{2\sigma^2}\right)$ is   monotonously increasing and
$0<a<1$. By computing the derivative for $\sigma \exp\left(
  \frac{2}{\sigma^2}\right)$, we know that the minimal value is
achieved when $\sigma= 2$. Therefore  \begin{eqnarray}\label{en1}
\sigma \exp\left( \frac{2}{\sigma^2}\right)\geq 2\exp(1/2).
\end{eqnarray}
Plugging \eqref{en1} into \eqref{en2}, we obtain the result.

\end{proof}

\section{Proofs in Section \ref{sec:discNN}}\label{proof Discrete NN}
Some arguments in the proofs of  Lemma \ref{approximate} and Lemma
\ref{lemma:evo} are informal.  Similar results with  more rigorous
treatments can be found in
\cite{chizat2018global,MeiE7665,mei2019mean}.    We directly assume
that $\rho_t, \nabla_{\theta}\rho_t,$  and $\nabla_{\theta}^2 \rho_t$
(or   $p_t, \nabla_{\theta}p_t,$  and $\nabla_{\theta}^2 p_t$) are
continuous, so that we can avoid some complex details of weak solutions. 
\subsection{Proof of Lemma \ref{approximate}: GD dynamic}
\begin{proof}[Proof of Lemma \ref{approximate}]
Let $p_t(\theta, u)$ be the join distribution of $(\theta, u)$ satisfying $p_t(\theta, u) = \rho_t \delta(u = \omega_t(\theta))$. The $[\theta^t_j, u^t_j]$ with $j\in m$ follow from $p_t(\theta,u)$. Also from \eqref{ff}, we can write $f$ as: 
\begin{eqnarray}
f(\omega_t, \rho_t, x) = \int_{\RR^{d+1}} u h'(\theta, x)p_t(\theta, u) d\theta d u.
\end{eqnarray} 
Given $x\in\mathbb{R}^d, f(\omega_t, \rho_t, x)<\infty$, then by the Law of Large Number, we have with $m\to\infty$, 
\begin{eqnarray}
\hat{f}(u_t, \theta_t, x) - f(\omega_t, \rho_t, x) \overset{\text{a.s.}}{\to} 0,
\end{eqnarray}
which implies \eqref{approximate1}.  Denote 
$$g_2'(t, \theta, u) = - \EE_{x,y}[\nabla_{y'}\phi(\hat{f}(u_t, \theta_t,x),y)u \nabla_\theta h'(\theta, x)] -\lambda_2 \nabla_{\theta} [r_2( \theta)]. $$
For all $j \in [m]$,  from the update rule of GD, we have
\begin{eqnarray}
\theta^j_{t+1} =  \theta^j_{t} + \Delta t g_2'(t, \theta_t^j, u_t^j), 
\end{eqnarray}
Let $\Delta t\to 0$, using $u^j_t = \omega_t(\theta_t^j)$, we have
\begin{eqnarray}\label{dtheta}
\frac{d \theta^j_t}{dt} =  g_2'(t, \theta^j_t, \omega_t(\theta^j_t)).
\end{eqnarray}
So $\rho_t$ can be obtained by its Fokker-Planck equation (Please refer to the background for Fokker-Planck equation in Appendix \ref{fokker}), which is 
\begin{eqnarray}
\frac{d \rho_t(\theta) }{dt} = -
  \nabla_\theta\cdot[\rho_t(\theta)g_2'(t, \theta, \omega_t(\theta))] .
\end{eqnarray}
With $m\to \infty$,  and because $\nabla_{y'} \phi$ is $L_2$ continuous and $h'(\theta,x)$ and $\rho_t$  are also second-order smooth,  we obtain
\begin{eqnarray}\label{g2'}
\nabla_\theta\cdot[\rho_t(\theta)g_2'(t, \theta, \omega_t(\theta))] -\nabla_\theta\cdot[\rho_t(\theta)g_2(t, \theta, \omega_t(\theta))] \overset{\text{a.s.}}{\to} 0.
\end{eqnarray}
Thus we obtain \eqref{tt1}. To prove \eqref{tt2}, let
\begin{eqnarray}\label{g1'}
g_1'(t, \theta,u) =  - \EE_{x,y}[\nabla_f\phi(\hat{f}(u_t,
  \theta_t,x),y)h'(\theta, x)] -\lambda_1  \nabla_u r_1 (u) .
\end{eqnarray} 
From the update rule of GD, we have
\begin{eqnarray}\label{omega 1}
\omega_{t+\Delta t}(\theta_{t+\Delta t}) = \omega_{t}(\theta_{t}) + g_1'(t, \theta_t, \omega_t(\theta)) \Delta t.
\end{eqnarray}
On the other hand, from \eqref{dtheta}, we have
\begin{eqnarray}\label{omega 2}
&&\omega_{t+\Delta t}(\theta_{t+\Delta t})\\
&=&\omega_{t+\Delta t}(\theta_{t} +g_2'(t, \theta_t, \omega_t(\theta))\Delta t +o(\Delta t) )  \notag\\
&=&\omega_{t}(\theta_{t} +g_2'(t, \theta_t, \omega_t(\theta))\Delta t +o(\Delta t)) +  \frac{d\omega_t(\theta_{t} +g_2'(t, \theta_t,\omega_t(\theta))\Delta t +o(\Delta t))}{d t}\Delta t\notag\\
&=& \omega_{t}(\theta_{t}) +[\nabla_{\theta} \omega_t(\theta)] \cdot g_2'(t,\theta_t, \omega_t(\theta)) \Delta t +o(\Delta t) + \frac{d\omega_t(\theta_{t} +g_2'(t, \theta, \omega_t(\theta))\Delta t +o(\Delta t))}{d t}\Delta t.\notag
\end{eqnarray}
Plugging \eqref{omega 1} into \eqref{omega 2}, we have
\begin{eqnarray}
\!\!\frac{d\omega_t(\theta_{t} +g_2'(t, \theta_t, \omega_t(\theta))\Delta t +o(\Delta t))}{d t}= - [\nabla_{\theta} (\omega_t(\theta))] \cdot g_2'(t,\theta_t, \omega_t(\theta)) + g_1'(t, \theta_t,\omega_t(\theta)) +o(1).
\end{eqnarray}
Let $\Delta t\to 0$, and let $m\to \infty$,  we obtain \eqref{tt2}.
\end{proof}
\subsection{Proof of Lemma \ref{lemma:evo}: NGD dynamic}
\begin{proof}[Proof of Lemma \ref{lemma:evo}]
For \eqref{tt3}, consider the $(d+1)$ dimensional particle $(\theta_t, u_t)$ with updates written  as:
\begin{eqnarray}
\theta_{t+1} &=& \theta_{t} +\Delta t g_2'(t,\theta, u)+ \sqrt{2\lambda_3\Delta t }N(0, I_d)  ,\notag\\
u_{t+1}&=& u_{t} +\Delta t g_1'(t,\theta,u)+ \sqrt{2\lambda_3\Delta t } N(0, 1),\notag
\end{eqnarray}
where $g_2'(t,\theta,u)$ and  $g_1'(t,\theta,u)$  are defined in \eqref{g2'} and \eqref{g1'}, respectively.
Let $\Delta t\to 0$, the Fokker-Planck equation of its distribution  is 
\begin{eqnarray}
   \frac{d p_t(\theta,u) }{dt} = - \nabla_\theta\cdot[p_t(\theta,
  u)g_2'(t, \theta,u)] -  \nabla_\u[p_t(\theta, u)g_1'(t, \theta,u)]+
  \lambda_3 \nabla^2[p_t(\theta,u)] .
\end{eqnarray}
Also let $m\to \infty$, we can obtain \eqref{evo p}.
\end{proof}

\subsection{Proof of Lemma \ref{evo rho}: Evolution of $\rho_t$ and $\omega_t$}
\begin{proof}[Proof of Lemma \ref{evo rho}]
Let  $g_{21}(t, \theta) =- \EE_{x,y}[\nabla_f\phi(f(\omega_t, \rho_t,x),y) \nabla_\theta h'(\theta, x)]. $
Then we have
\begin{eqnarray}\label{split g2}
g_2(t,\theta, u) =  g_{21}(t, \theta) u -\lambda_2 \nabla_{\theta} [r_2( \theta)].
\end{eqnarray}

To prove \eqref{tt3}, we integrate on both sides of \eqref{evo p}
over $u$, and obtain
\begin{eqnarray}\label{i4}
\!\!\!\!\!\!\frac{d \rho_{t}(\theta)}{d t} = - \! \underbrace{\int_{\RR}\nabla_\theta(p_t(\theta, u)g_2(t, \theta,u))d u}_{I_1} -\!\underbrace{ \int_{\RR} \nabla_\u[p_t(\theta, u)g_1(t, \theta)]d u}_{I_2}+\!\underbrace{\int_{\RR} \lambda_3 \nabla^2[p_t(\theta,u)]d u}_{I_3}.
\end{eqnarray}
We separately simplify $I_1$, $I_2$, and $I_3$.  For $I_1$, we have
\begin{eqnarray}\label{i1}
I_1 &=&\int_{\RR}\nabla_\theta\cdot[p_t(\theta, u)g_2(t, \theta,u)]d u\\
&\overset{\eqref{split g2}}{=}&\int_{\RR}\nabla_\theta\cdot[p_t(\theta, u)(g_{21}(t, \theta)u   -\lambda_2 \nabla_{\theta} r_2( \theta)) ]d u\notag\\
&=& \nabla_{\theta}\cdot\left( \int_{\RR}p_t(\theta, u)g_{21}(t, \theta)u   - p_t(\theta, u)\lambda_2 \nabla_{\theta} r_2( \theta) d u\right)\notag\\
&\overset{a}{=}&\nabla_{\theta}\cdot\left[ \rho_t(\theta) g_{21}(t, \theta) \omega_t(\theta) -  \rho_t(\theta )\lambda_2 \nabla_{\theta} r_2( \theta)\right]\notag\\
&=& \nabla_{\theta}\cdot[\rho_t(\theta)g_2(t, \theta, \omega_t(\theta))],\notag
\end{eqnarray}
where in equality $\overset{a}{=}$, we use  the definition that $\rho_t$ is the marginal distribution of $p_t(\theta, u)$ and $\omega_t(\theta) = \EE_{p_t}[u | \theta]$.

For $I_2$, we have 
\begin{eqnarray}\label{i2}
I_2 &=& \int_{\RR} \nabla_\u(p_t(\theta, u)g_1(t, \theta,u))d u \\
&=&  [p_t(\theta, u)g_1(t, \theta,u)] |_{u=-\infty}^{\infty}\notag\\
&\overset{a}{=}&0,\notag
\end{eqnarray}
where in $\overset{a}{=}$, we use \eqref{bound g_1}.

For $I_3$, we have
\begin{eqnarray}\label{i3}
I_3&=&\int_{\RR} \lambda_3 \nabla^2(p_t(\theta,u))d u\\
&=& \lambda_3\int_{\RR}  \nabla^2_{\theta}(p_t(\theta,u))d u +  \lambda_3\int_{\RR}  \nabla^2_{u}(p_t(\theta,u))d u\notag\\
&=&\lambda_3 \nabla_{\theta}^2 \rho_t(\theta).\notag
\end{eqnarray}
Plugging \eqref{i1}, \eqref{i2}, and \eqref{i3} into \eqref{i4}, we obtain \eqref{tt3}.

Let $\mu_t(\theta) = \rho_t(\theta) \omega_t(\theta) = \int_\RR u p(\theta, u) du$. Then multiply  by $u$  on both sides of \eqref{evo p}, and then integrate the result on both sides over $u$, we have
\begin{eqnarray}\label{evo u4}
\frac{d \mu_t(\theta)}{d t} &=& -  \int_{\RR}u\nabla_\theta\cdot[p_t(\theta, u)g_2(t, \theta,u)]d u -\underbrace{ \int_{\RR} u\nabla_\u[p_t(\theta, u)g_1(t, \theta,u)]d u}_{I_4}\\&&+\underbrace{\int_{\RR} \lambda_3 u\nabla^2[p_t(\theta,u)]d u}_{I_5}.\notag
\end{eqnarray}
For $I_4$, from integration by parts, we have
\begin{eqnarray}\label{evo u1}
I_4 &=&  \int_{\RR} u\nabla_\u[p_t(\theta, u)g_1(t, \theta,u)]d u \\
&\overset{a}=& - \int_{\RR}p_t(\theta, u)g_1(\theta, t,u) du\notag\\
&\overset{b}=&- \int_{\RR}p_t(\theta, u)g_{11}(\theta, t) - p_t(\theta, u) \lambda_1 u du\notag\\
&=&-\rho_t(\theta)g_1(t, \theta, \omega_t(\theta)),\notag
\end{eqnarray}
where $\overset{a}=$ uses \eqref{bound g_1} and the second moment of $p_t$ is bounded and $\overset{b}=$we let $g_{11}(t,\theta) = g_1(t,\theta, u) + \lambda_1 u$.
Similarly, for $I_5$, we have 
\begin{eqnarray}\label{evo u2}
I_5 &=& \int_{\RR} \lambda_3 u\nabla^2[p_t(\theta,u)]d u\\
&=&\lambda_3 \nabla^2_{\theta} \left[\int_{\RR}  up_t(\theta,u)d u\right] +  \lambda_3  \int_{\RR} u\nabla^2_u[p_t(\theta,u)]d u\notag\\
&=&\lambda_3 \nabla_{\theta}^2 [\rho_t(\theta)\omega_t(\theta)].\notag
\end{eqnarray}
On the other hand, because  
\begin{eqnarray}\label{evo u}
\frac{d \mu_t(\theta)}{d t} &=&  \rho_t(\theta) \frac{d \omega_t(\theta)}{dt} +  \omega_t(\theta) \frac{d \rho_t(\theta)}{dt}\\
& \overset{\eqref{tt3}}{=}& \rho_t(\theta) \frac{d \omega_t(\theta)}{d t} + \omega_t(\theta) \left\{ - \nabla_\theta\cdot[\rho_t(\theta)g_2(t, \theta, \omega_t(\theta))]+ \lambda_3 \nabla^2_{\theta} \rho_t(\theta) \right\}.\notag
\end{eqnarray}
Then plugging \eqref{evo u}, \eqref{evo u1}, \eqref{evo u2} into  \eqref{evo u4}, we have 
\begin{eqnarray}
\frac{d\omega_t(\theta)}{dt} &\overset{a}{=}&g_1(t, \theta, \omega_t(\theta)) - \nabla_\theta[\omega_t(\theta)] \cdot g_2(t, \theta, \omega_t(\theta)) - \frac{\lambda_3\omega_t(\theta)}{\rho_t(\theta)}\nabla^2_\theta[\rho_t(\theta)]+ \frac{\lambda_3}{\rho_t(\theta)}\nabla^2_\theta[\omega_t(\theta)\rho_t(\theta)] \notag\\
&&-\frac{1}{\rho_t(\theta)}\nabla_{\theta}\cdot\left[\int_\RR p(\theta, u)u g_2(t, \theta, u)du -\rho_t(\theta)\omega_t(\theta) g_2(t, \theta, \omega_t(\theta)) \right]\notag\\
&\overset{b}{=}&g_1(t, \theta, \omega_t(\theta)) - \nabla_\theta(\omega_t(\theta)) \cdot g_2(t, \theta, \omega_t(\theta))  + \lambda_3 \nabla_\theta^2[\omega_t(\theta)] + \frac{2\lambda_3}{\rho_t(\theta)}[\nabla\rho_t(\theta)]\cdot [\nabla_{\theta}\omega_t(\theta)]\notag\\
&&-\frac{1}{\rho_t(\theta)}\nabla_{\theta}\cdot\left(\int_\RR p(\theta, u)u [g_2(t, \theta, u)-g_2(t, \theta, \omega_t(\theta))]d u \right), 
\end{eqnarray}
where in $\overset{a}=$, we use
\begin{eqnarray}
&&\omega_t(\theta)  \nabla_{\theta}\cdot[\rho_t(\theta)g_2(t,\theta, \omega_t(\theta))  ]\notag\\
&=&\nabla_{\theta}\cdot[\omega_t(\theta)\rho_t(\theta)g_2(t,\theta, \omega_t(\theta))] - \rho_t(\theta) [\nabla_{\theta}\omega_t(\theta)]\cdot g_2(t, \theta, \omega_t(\theta))\notag
\end{eqnarray}
and in $\overset{b}=$, we use 
\begin{eqnarray}
\nabla^2_\theta [ \omega_t(\theta) \rho_t(\theta) ]=\nabla^2_\theta [ \omega_t(\theta)] \rho_t(\theta) +  \nabla^2_\theta [ \rho_t(\theta)] \omega_t(\theta)+ 2 [\nabla_{\theta} \rho_t(\theta)]\cdot [\nabla_{\theta} \omega_t(\theta)].\notag
\end{eqnarray}
We achieve \eqref{tt4}.
\end{proof}

\newpage

\section{Background}
\subsection{Fokker-Planck Equation}\label{fokker}
Suppose the movement of a particle in $m$-dimensional space can be characterized by 
the stochastic differential equation given below:
\begin{eqnarray}
dx_t  = g(x_t, t)d_t + \sqrt{2\beta^{-1}}\Sigma d_{B_t},
\end{eqnarray}
where $x_t$ is $m$-dimensional vector denoting the random position of the particle at time $t$, $g(x_t,t)$ is $m$-dimensional drift vector, $\Sigma$ is a $m\times n$ matrix,   and $B_t$ denotes a $n$ dimensional Brownian motion.  Let $x_t\sim p(x,t)$, the  Fokker-Planck equation describes the evolution of $p(x,t)$ as:
\begin{eqnarray}
\frac{\partial p(x,t)}{\partial t}=  \frac{\Sigma\Sigma^\top}{\beta} \nabla^2 p(x,t) - \nabla \cdot [p(x,t)g(x_t, t)].
\end{eqnarray}
 The proof can be found in  textbook about Fokker-Planck equations, e.g. \cite{risken1996fokker}.

\subsection{BL Distance}\label{bl dis}
The BL distance is defined between probability measures by 
\begin{eqnarray}
d_{\mathrm{BL}} (\mu, \nu) = \sup\left\{ \left|\int f(\x)\mu(d x) - \int f(\x) \nu(d x)\right|~:~ \| f\|_{\infty}\leq 1,  \| f\|_{\mathrm{Lip}} \leq 1   \right\}.
\end{eqnarray}
It is known that if $p_k$ converges weakly to $p_\infty$, we have $d_{\mathrm{BL}}(p_k, p_\infty) \to 0$ \citep{van1996weak}.

\subsection{Entropy}\label{bound enro}
\begin{lemma}[Lemma 10.1 in \cite{MeiE7665}]
Let $p(\theta)$ be the probability density function for  $d$-dimensional random variable $\theta$, we have
\begin{eqnarray}\label{entro0}
-\int_{\RR^d} p(\theta)\mathrm{ln}(p(\theta))d\theta \leq  1+ \E_{p}\| \theta\|^2/z + d \mathrm{ln}(2\pi z),
\end{eqnarray}
for any $z>0$.
\end{lemma}
We give a proof here, which is directly taken from \cite{MeiE7665}.
\begin{proof}
Let $\Omega=\{\theta: (\sqrt{2\pi z})^{-d}\exp\left(-\| \theta\|^2/(2z) \right)\leq p(\theta)^{1/2}\leq 1  \} $. Then
\begin{eqnarray}\label{entro1}
&&-\int_{\RR^d} p(\theta)\mathrm{ln}(p(\theta))d\theta\\
&\leq&\int_{\RR^d} p(\theta)|\min(\mathrm{ln}(p(\theta)), 0)  |d\theta\notag\\
&=&\int_{\Omega} p(\theta)|\min(\mathrm{ln}(p(\theta)), 0)  |d\theta +\int_{\Omega^c} p(\theta)|\min(\mathrm{ln}(p(\theta)), 0)  |d\theta.\notag
\end{eqnarray}
 For the first term, we have
\begin{eqnarray}\label{entro2}
&&\int_{\Omega} p(\theta)|\min(\mathrm{ln}(p(\theta)), 0)  |d\theta\\
&\leq&\int_{\RR^d}p(\theta)[\|\theta \|^2/z +d \mathrm{ln}(2\pi z)  ]d\theta \notag\\
&=&\E_{p}\| \theta\|^2/z+d \mathrm{ln}(2\pi z).\notag
\end{eqnarray}
Because $|p(\theta) \mathrm{ln}(p(\theta)) | \leq \sqrt{p(\theta)}$ for all $0\leq p(\theta)\leq 1$, for the second term, we have
\begin{eqnarray}\label{entro3}
&&\int_{\Omega^c} p(\theta)|\min(\mathrm{ln}(p(\theta)), 0)  |d\theta\\
&\leq& \int_{\RR^d}  (\sqrt{2\pi z})^{-d}\exp\left(-\| \theta\|^2/(2z)\right)d\theta \notag\\
&=&1.\notag
\end{eqnarray}
Plugging \eqref{entro2} and \eqref{entro3} into \eqref{entro1}, we obtain \eqref{entro0}.

\end{proof}

\newpage

\section{Additional Experiments}
\subsection{Simulated Data}\label{simu}
We employ NGD (Algorithm \ref{algo:NGD}) to learn the first-layer
weights of a binary classification problem. The classification problem
is adapted from \emph{make\_classification} in Scikit-learn
\citep{pedregosa2011scikit}. We use 100-dimensional features,
including only 4 informative dimensions, 10 redundant dimensions, 10
repeated dimensions, and 76 random noise. The dataset includes 500
training instances and 500 testing instances. As a result, an NN with
4 hidden nodes is able to represent label-related information, but we
often use a larger network for a better performance. Since the
simulated data is well-structured, the original Random Kitchen Sinks
can learn accurate predictions. 
In our experiments, we train 10 hidden nodes as our source distribution and employ Conditional Variational Auto-Encoder to generate additional samples from the source distribution.

\begin{figure}[H]
	\begin{center}
			\begin{tabular}{ccc} 
			\includegraphics[width=0.31\linewidth]{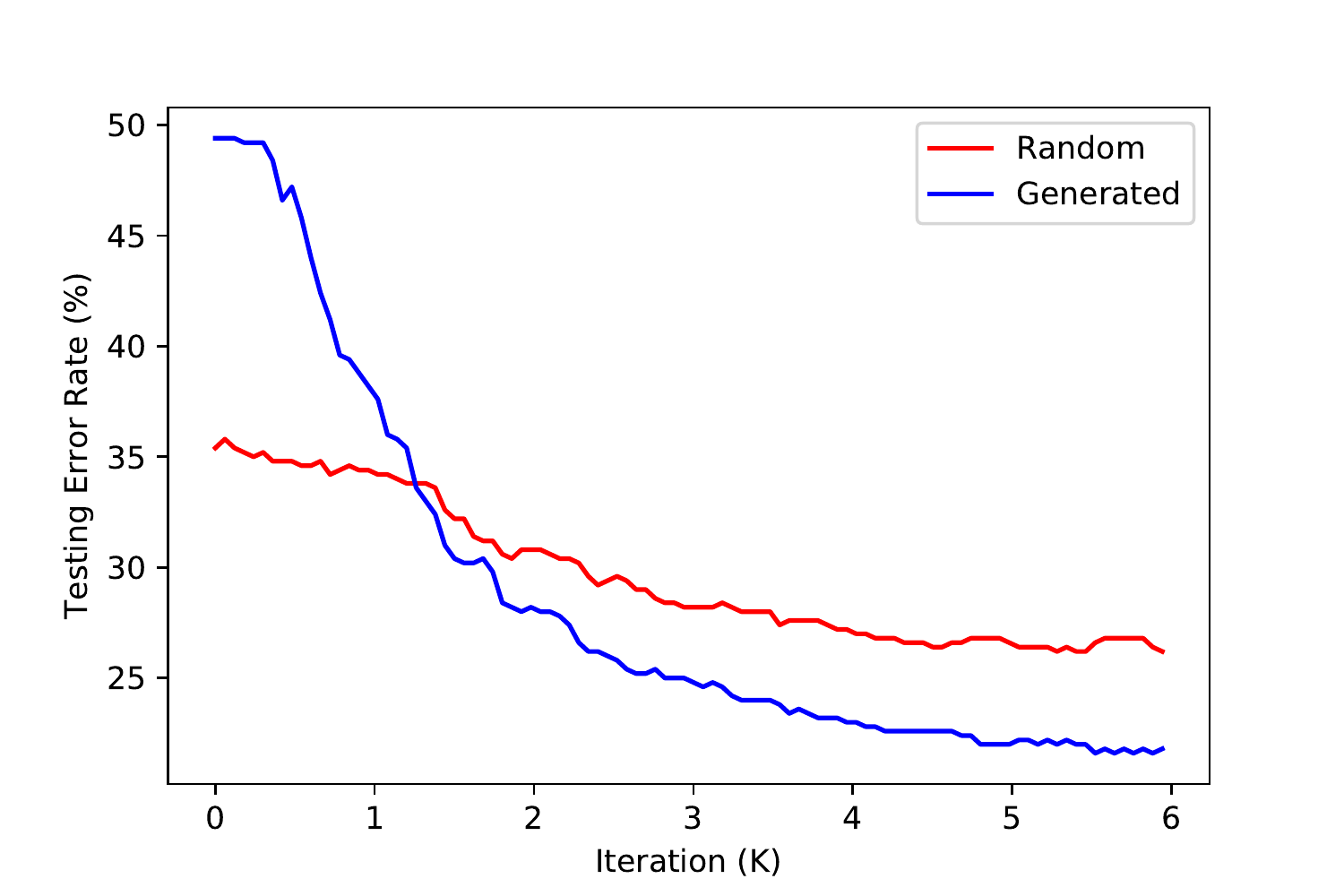}
			&\includegraphics[width=0.31\linewidth]{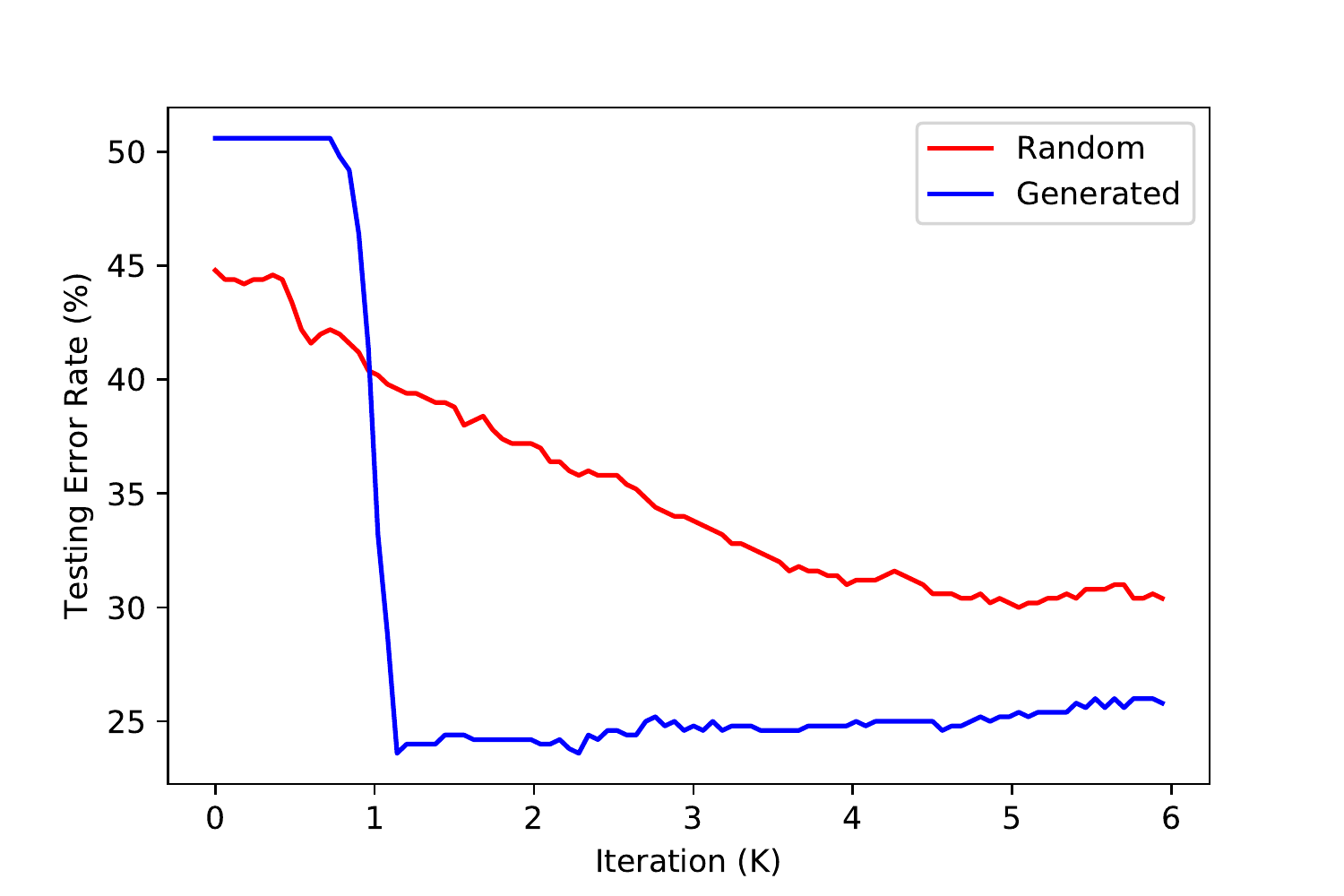} 
			&\includegraphics[width=0.31\linewidth]{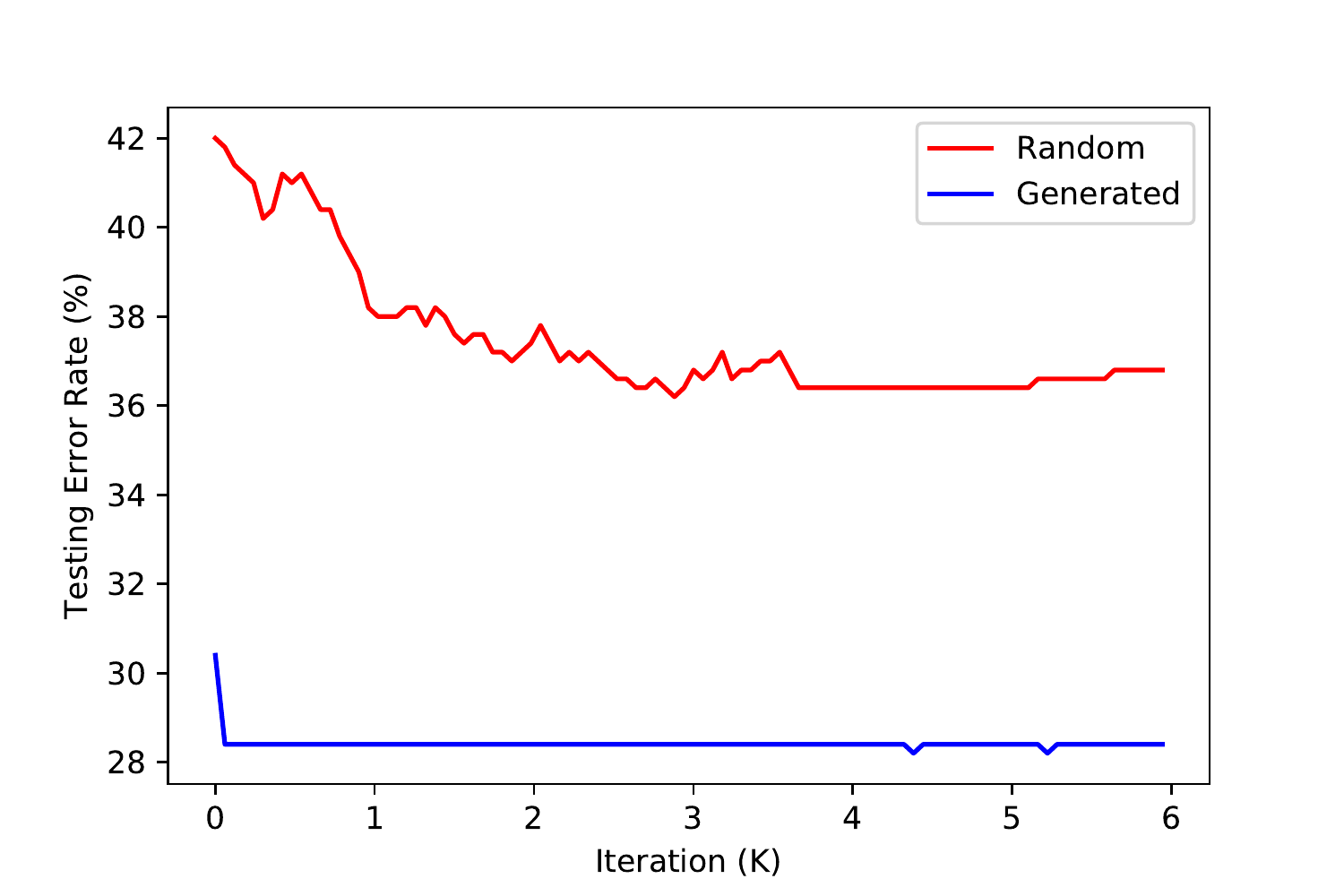}\\
			(a) 10 hidden nodes
			&(b) 5 hidden nodes
			&(c) 2 hidden nodes
		\end{tabular}
	\end{center}
	\caption{Test error rate on simulated data. }
	\label{fig:simulate}
\end{figure}

Figure \ref{fig:simulate} shows that by sampling from the optimized weight distribution, one can obtain features better than task independent random features, and one can achieve  better performance with small-sized networks. 
As a result, the optimization process with NGD behaves similarly as that of SGD. 

\subsection{Efficiency of Feature Representation}\label{efficacy_representation}
Additionally, we study the efficiency of the feature representations (Definition \ref{def optimal}).  We consider the variance of NN function in terms of $x$,   so that the definition is adapted to be data dependent. In particular, we normalize the variance of the first layer output by batch-normalization (without affine layer) and compute the the sum of $(u^i)^2$, which is equal to the variance. Note that the comparison of efficiency should be done at comparable classification accuracy levels. Otherwise the regularizer may dominate the optimization process. We only perform experiments on \emph{simple datasets}, which can achieve good classification performance with random features. Despite this, the performance of optimized features are still better on both simulated data and MNIST (Optimized: 81.6\%, 98.3\%; Random: 77.8\%, 93.0\%). 
Results are displayed below.

\begin{figure}[H]
	\begin{center}
		\begin{tabular}{cc} 
			\includegraphics[width=0.45\linewidth]{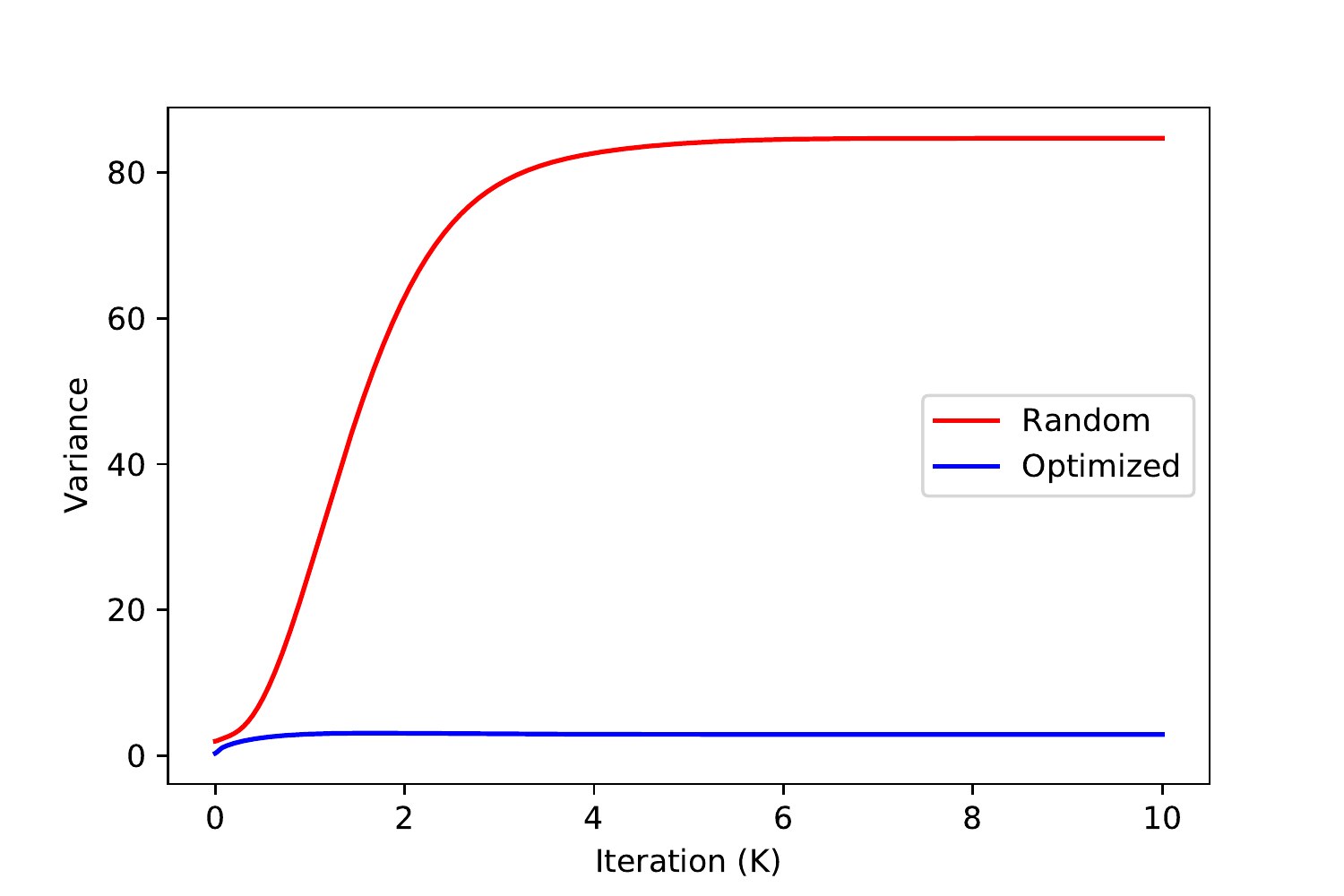}&
			\includegraphics[width=0.45\linewidth]{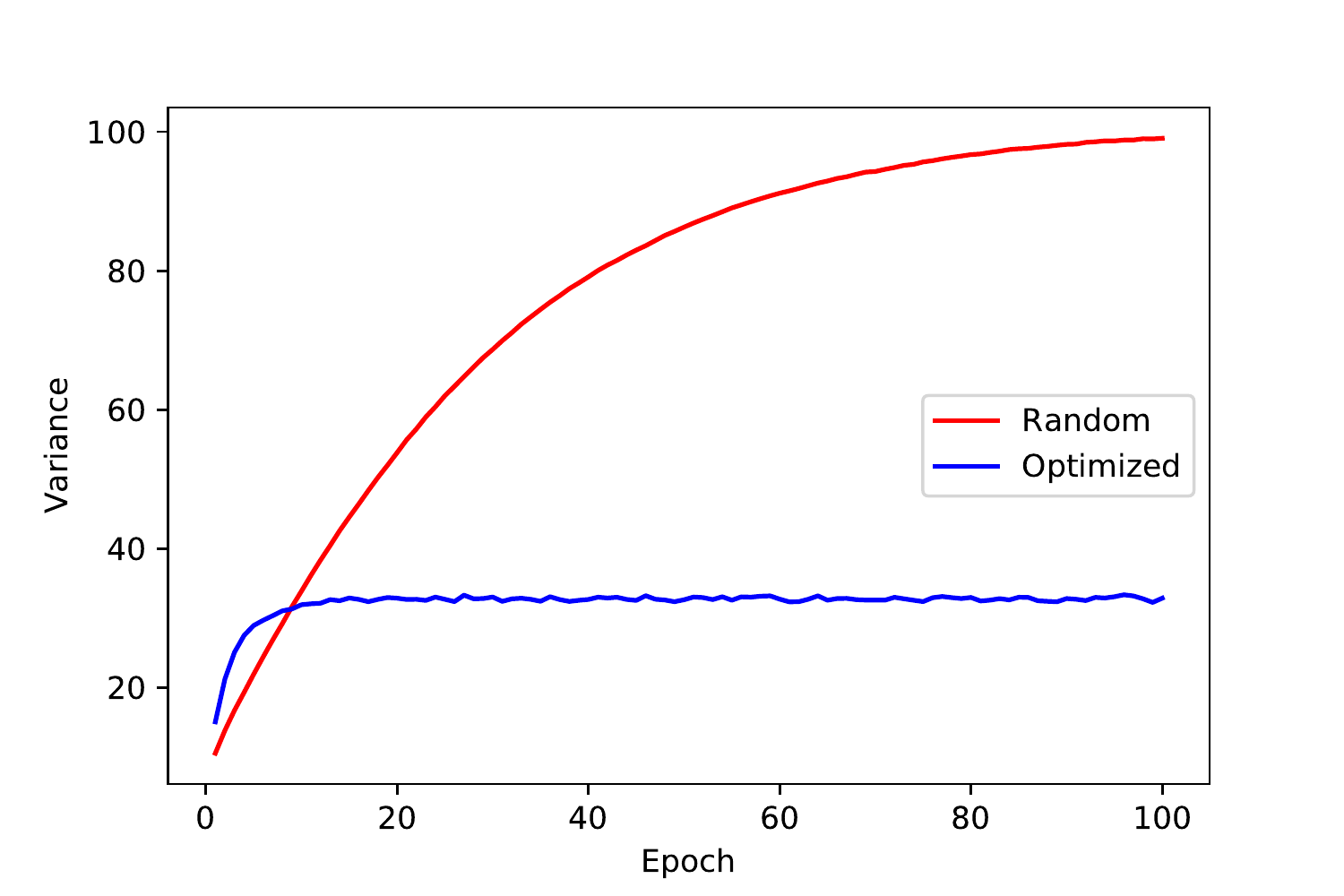}\\
			(a) Simulated data& (b) MNIST
		\end{tabular}
	\end{center}
	\caption{Variance of hidden node functions}
	\label{fig:simulate_weights_var}
\end{figure}

We perform the experiments on both simulated data and MNIST and use a relative large net ($m'=1000$) to keep the random net capable of achieving high classification accuracy. In addition, weight decay is set as $10^{-3}/10^{-4}$ for simulated data and for MNIST respectively, so that the variance values converge. In order to make a fair comparison, we optimize a pre-trained network first, and then use the pre-trained weights as the first-layer initialization (and keep the top layer random), rather than optimize two layers simultaneously.

Figure \ref{fig:simulate_weights_var} and \ref{fig:simulate_weights_hist} shows that the optimized first-layer weights reduce the variance significantly and the variation of $\|u\|$ is small, whereas the random features get more unimportant weights ($\|u\| \approx 0$) for simulated data, and the resulting variances are large for both MNIST and simulated data. This indicates that the hidden layer requires more samples to approximate the target function.

\begin{figure}[H]
	\begin{center}
		\begin{tabular}{cc} 
			\includegraphics[width=0.45\linewidth]{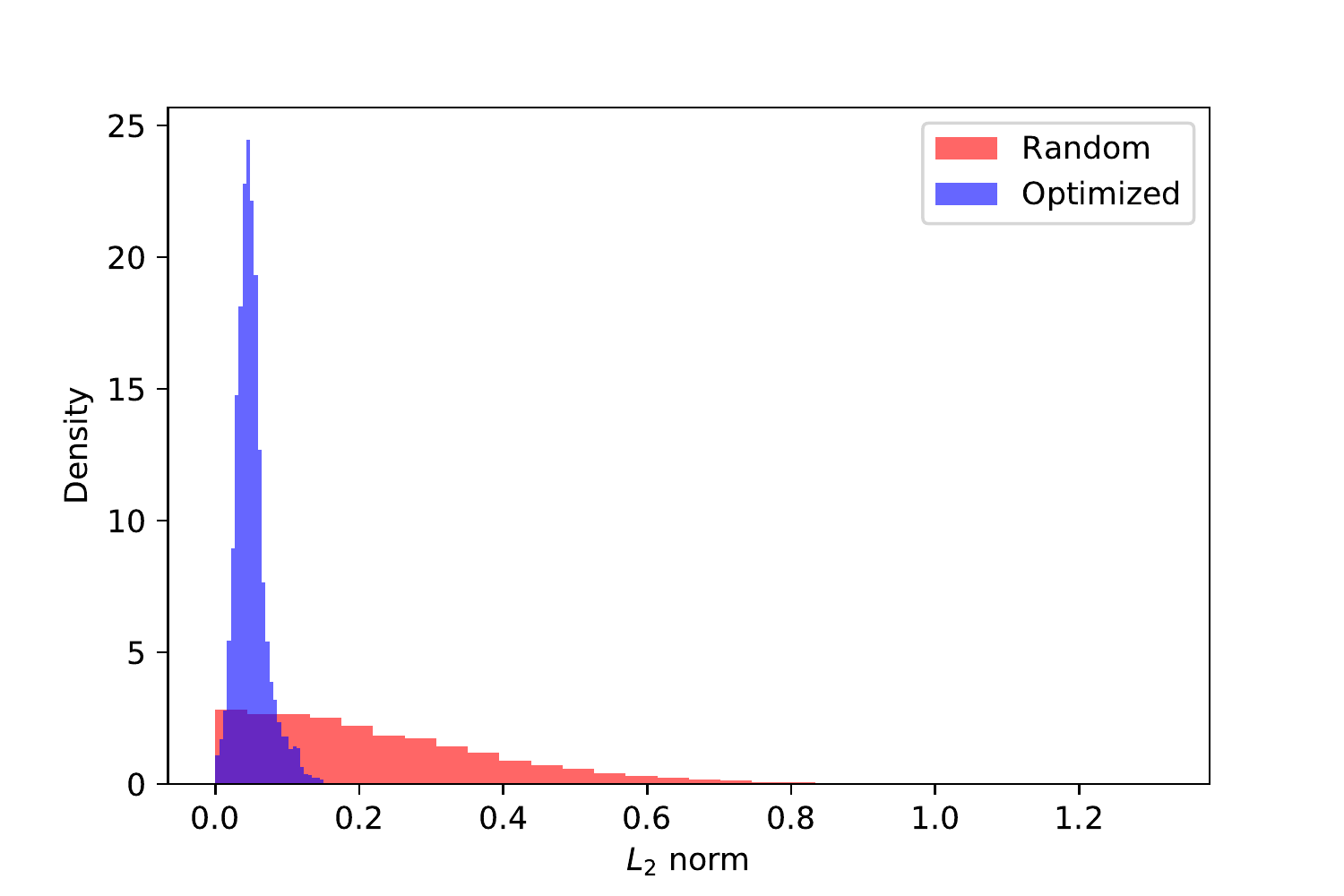}&
			\includegraphics[width=0.45\linewidth]{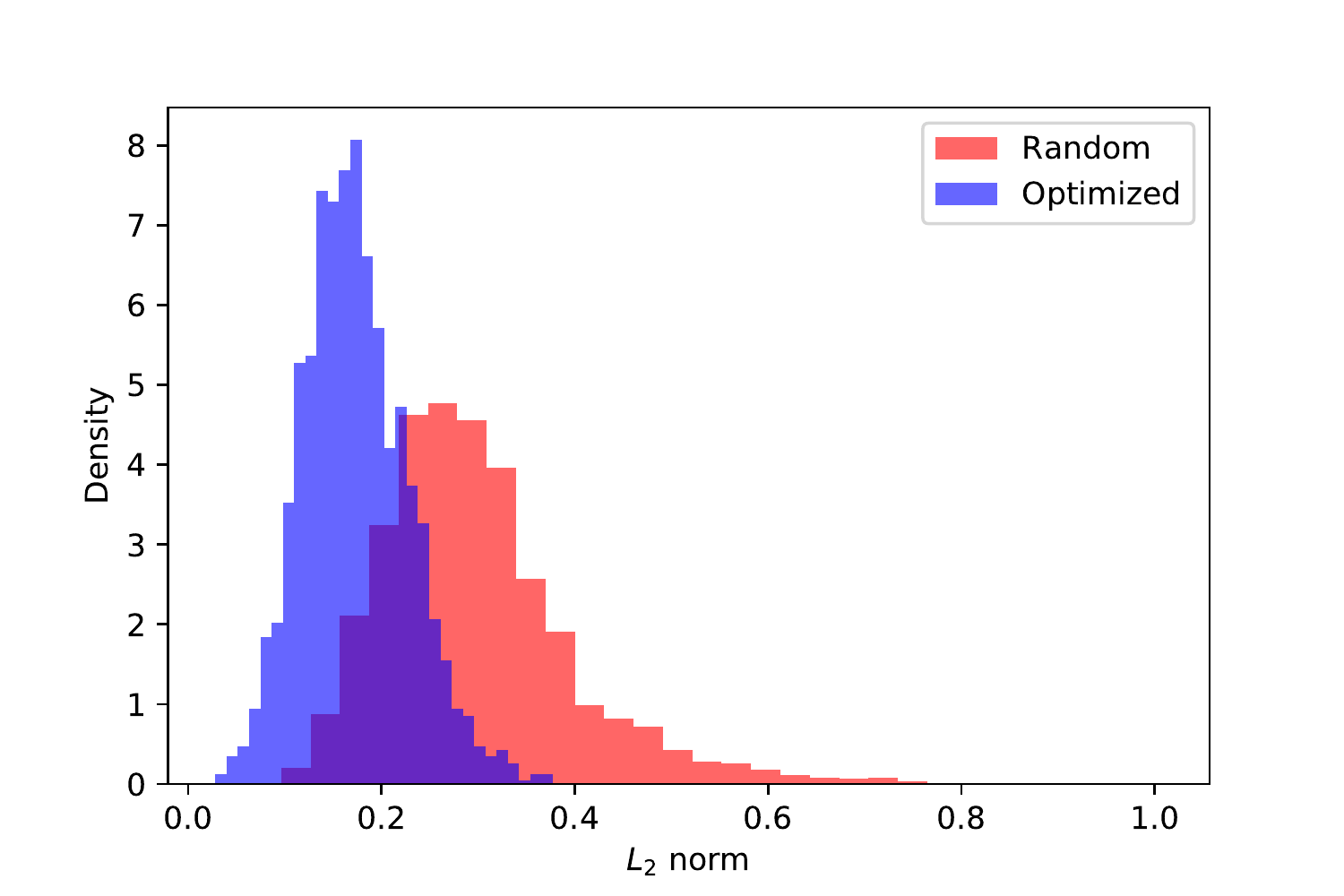}\\
			(a) Simulated data& (b) MNIST
		\end{tabular}
	\end{center}
	\caption{Histogram of $||u/m'||$}
	\label{fig:simulate_weights_hist}
\end{figure}

\subsection{Importance Sampling versus Uniform Sampling}\label{impor}

In Figure \ref{fig:simulate_weights_hist}, we verify that the importance of each node is almost the same, which means most of nodes in an NN are effective. However, in some special cases, some nodes might be off due to the hyper-parameters and optimization strategy. For example, when we train a large two-level network ($m=10000$) with a large weight decay $\lambda_2$ for $\theta$ (Simulated data/MNIST: $10^{-3}$; CIFAR-10: $10^{-2}$), only a few weights of the optimized NN are effective, which can also be derived from Corollary \ref{optimal fea2} ($\lambda_2$ should be small). Figure \ref{fig:IS_count} plots the histogram of such a phenomenon that most of the $\|u/m'\|$ are smaller than $10^{-2}$, which means that very few $\theta$ samples are effective.

\begin{figure}[H]
	\begin{center}
			\begin{tabular}{ccc} 
			\includegraphics[width=0.31\linewidth]{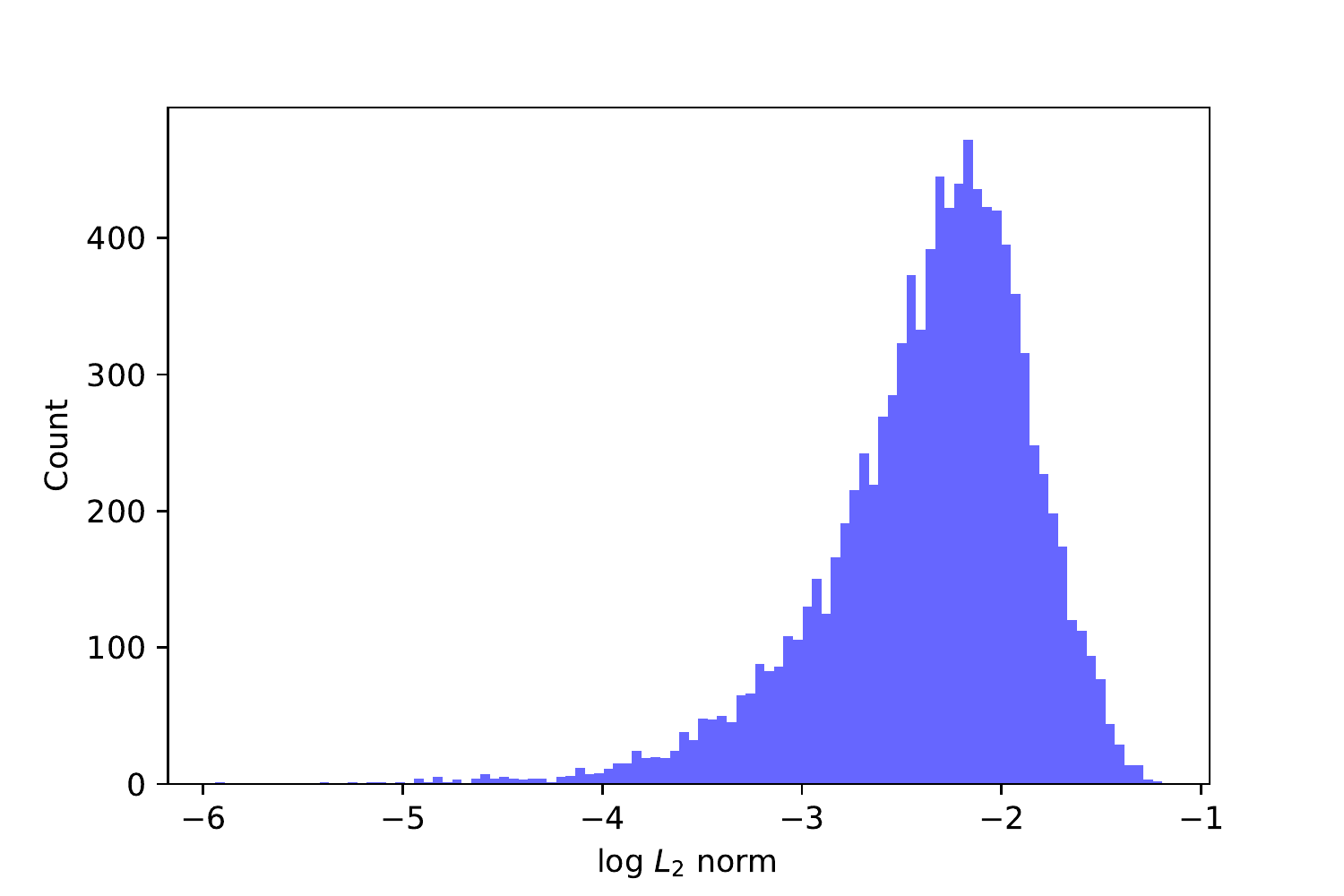}
			&\includegraphics[width=0.31\linewidth]{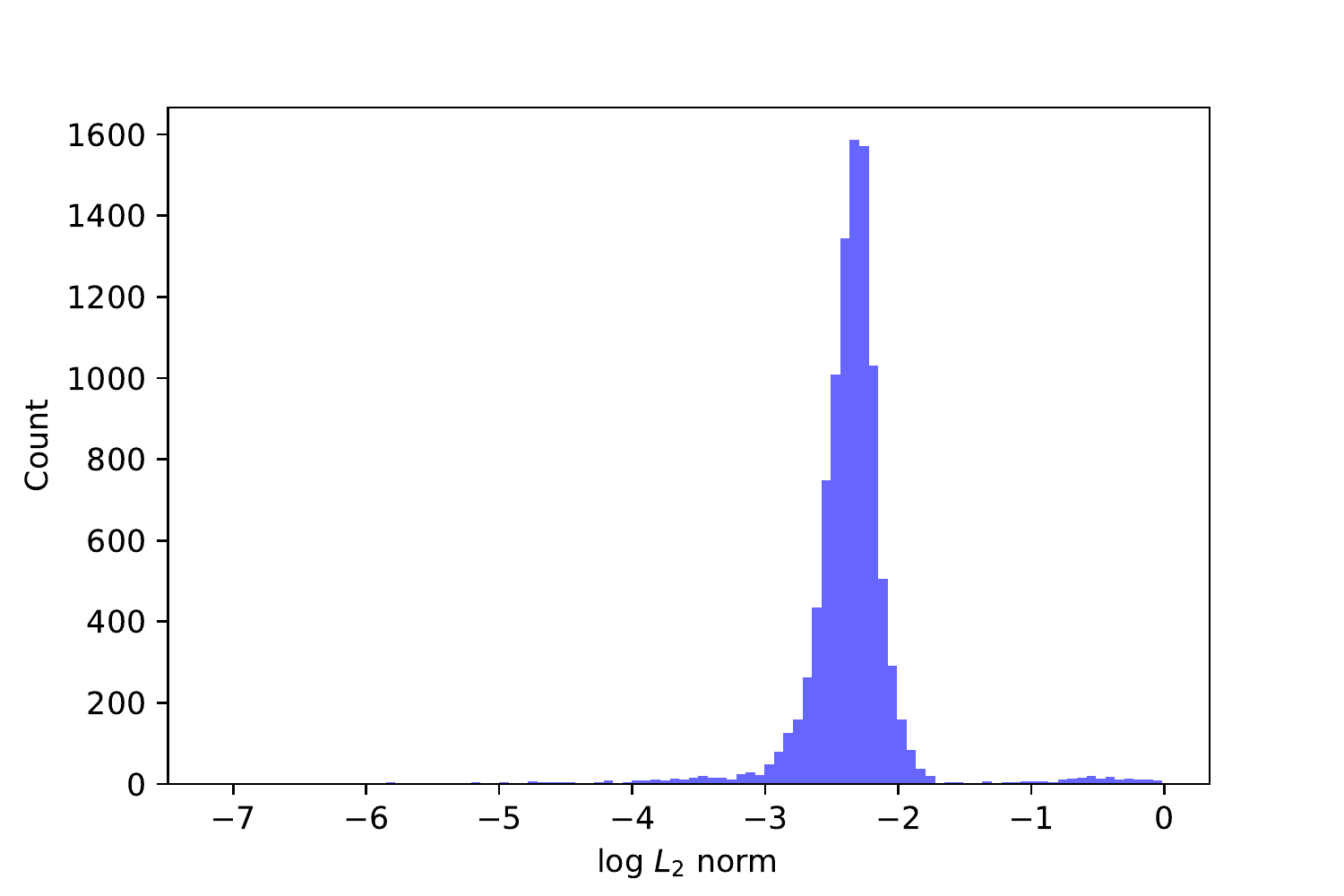} 
			&\includegraphics[width=0.31\linewidth]{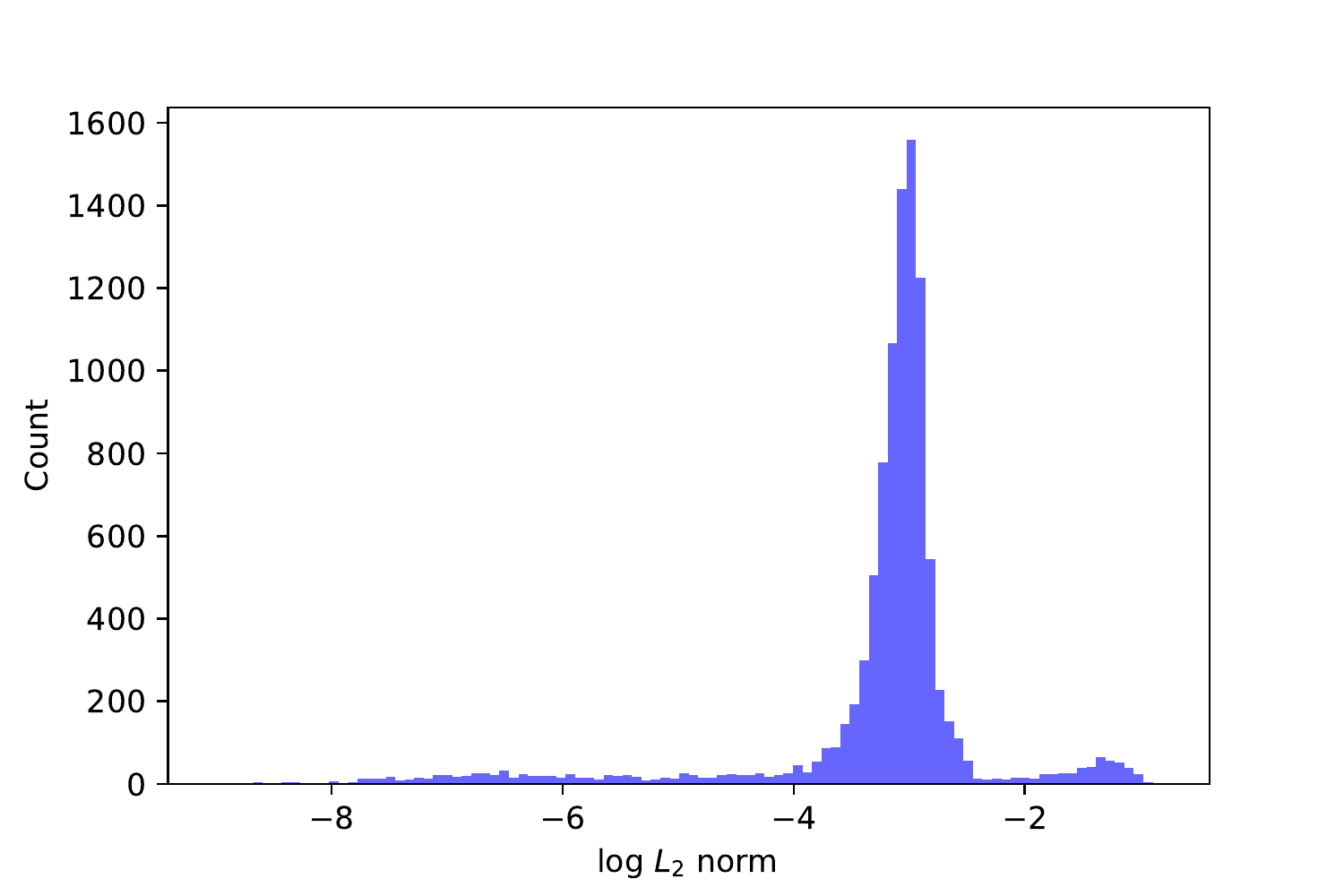}\\
			(a) Simulated Data
			& (b) MNIST
			& (c) CIFAR-10
		\end{tabular}
	\end{center}

	\caption{Histogram of $\log_{10}\|u/m'\|$ with large $\lambda_2$}
	\label{fig:IS_count}
\end{figure}

If we treat the corresponding $\|u^i\|_2$ as the importance of $\theta^i$, we can resample the weight vectors in a small NN, so that unimportant feature dimension can be dropped. We compare both importance sampling (IS) and uniform sampling (US) below.

\begin{figure}[H]
	\begin{center}
			\begin{tabular}{ccc} 
			\includegraphics[width=0.31\linewidth]{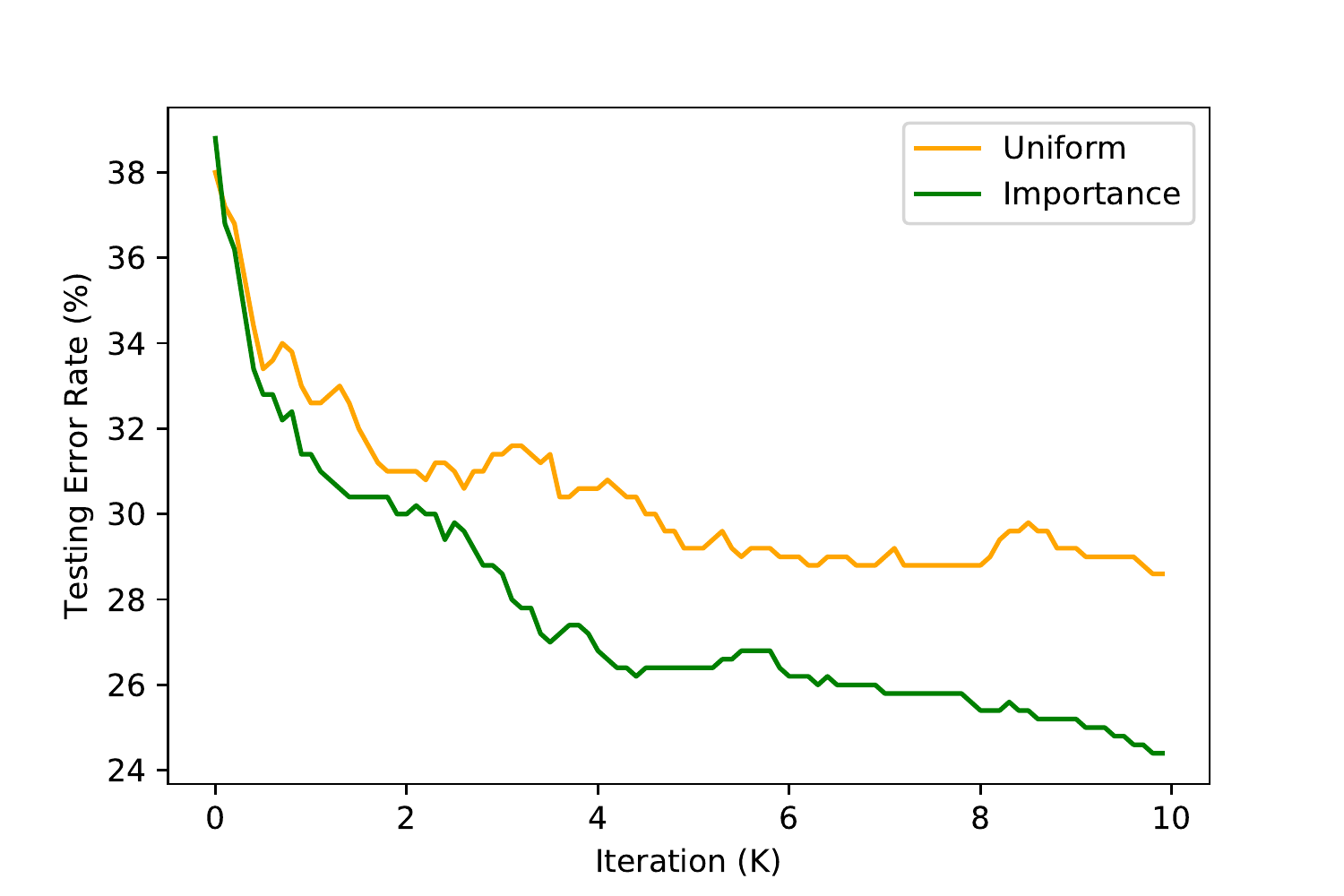}
			&\includegraphics[width=0.31\linewidth]{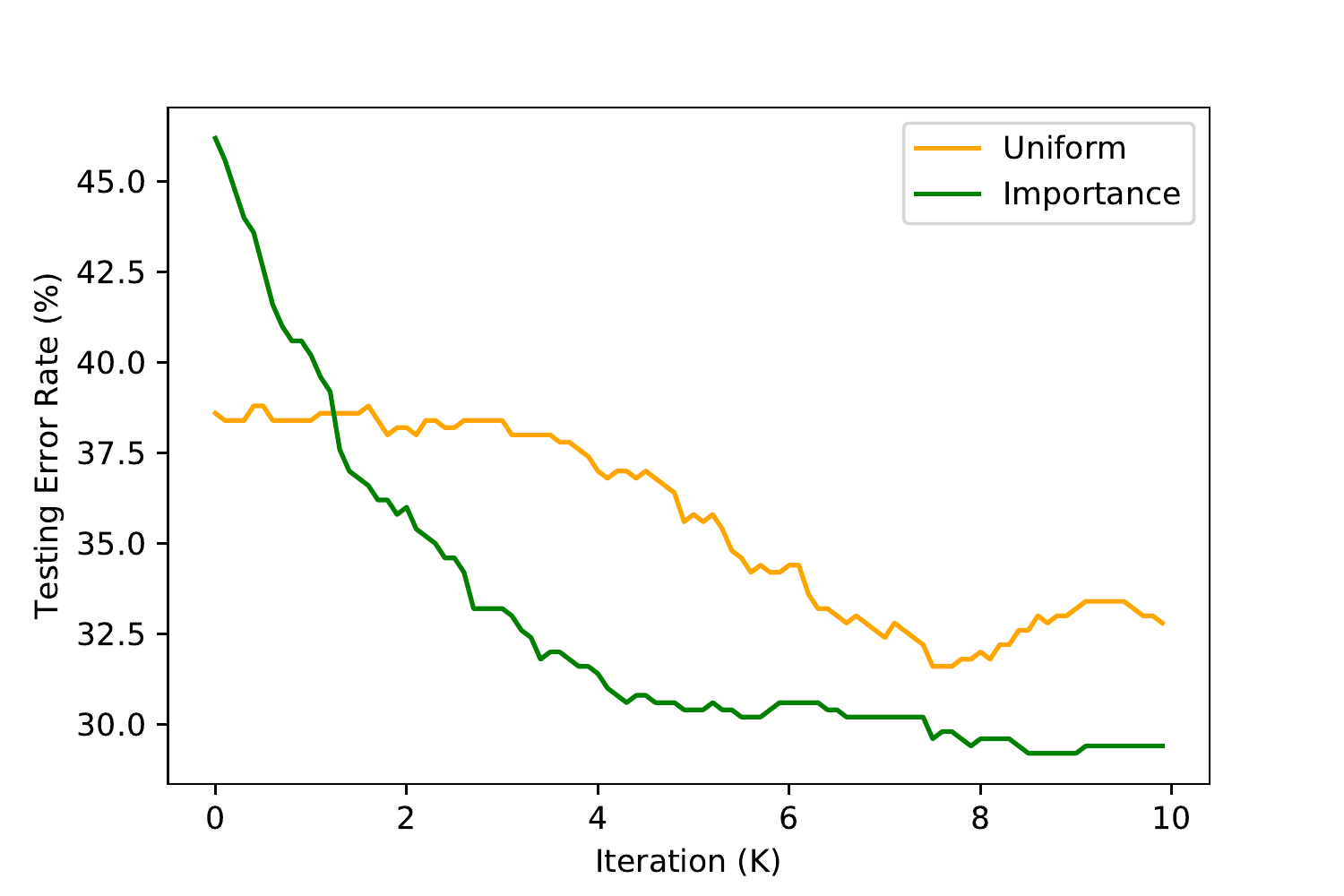} 
			&\includegraphics[width=0.31\linewidth]{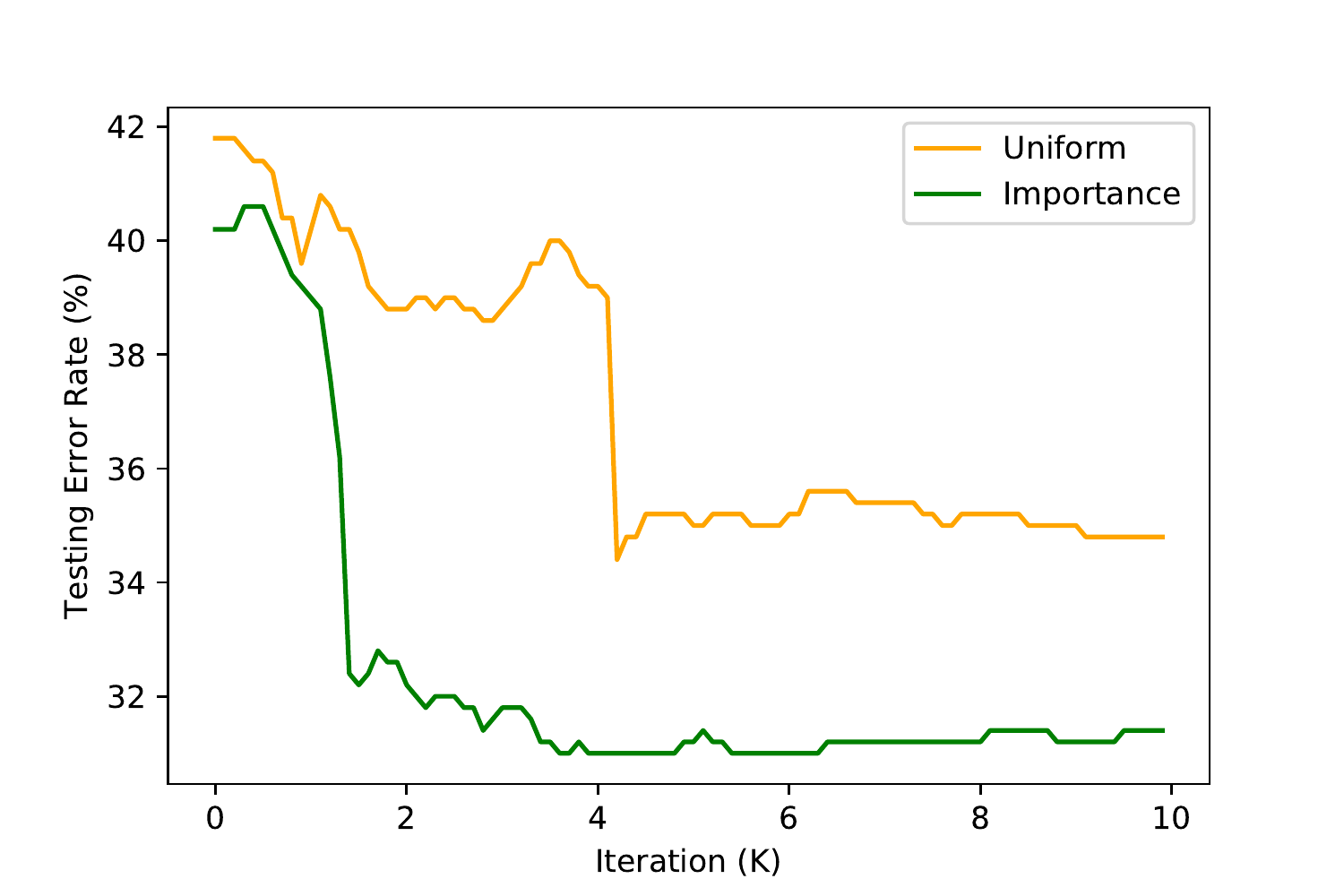}\\
			(a) 10 hidden nodes
			&(b) 5 hidden nodes
			&(c) 2 hidden nodes
		\end{tabular}
	\end{center}

	\caption{Comparisons between Importance and Uniform Sampling (Importance/Uniform) over regularized weight distribution
	on Simulated data. Only $u$ is considered during optimization.}
	\label{fig:IS_simulated}
\end{figure}

\begin{figure}[H]
	\begin{center}
			\includegraphics[width=0.4\linewidth]{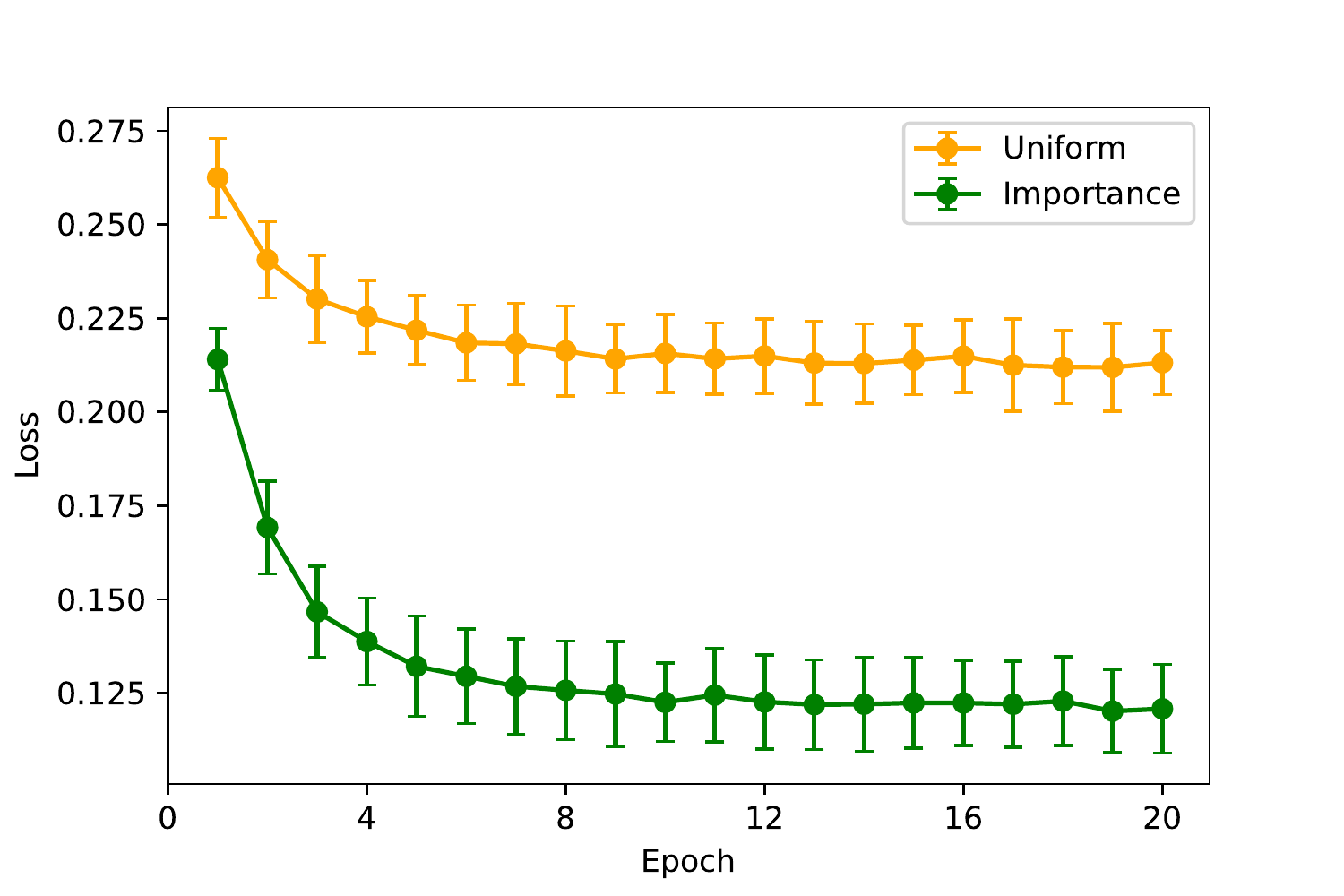}
		\includegraphics[width=0.4\linewidth]{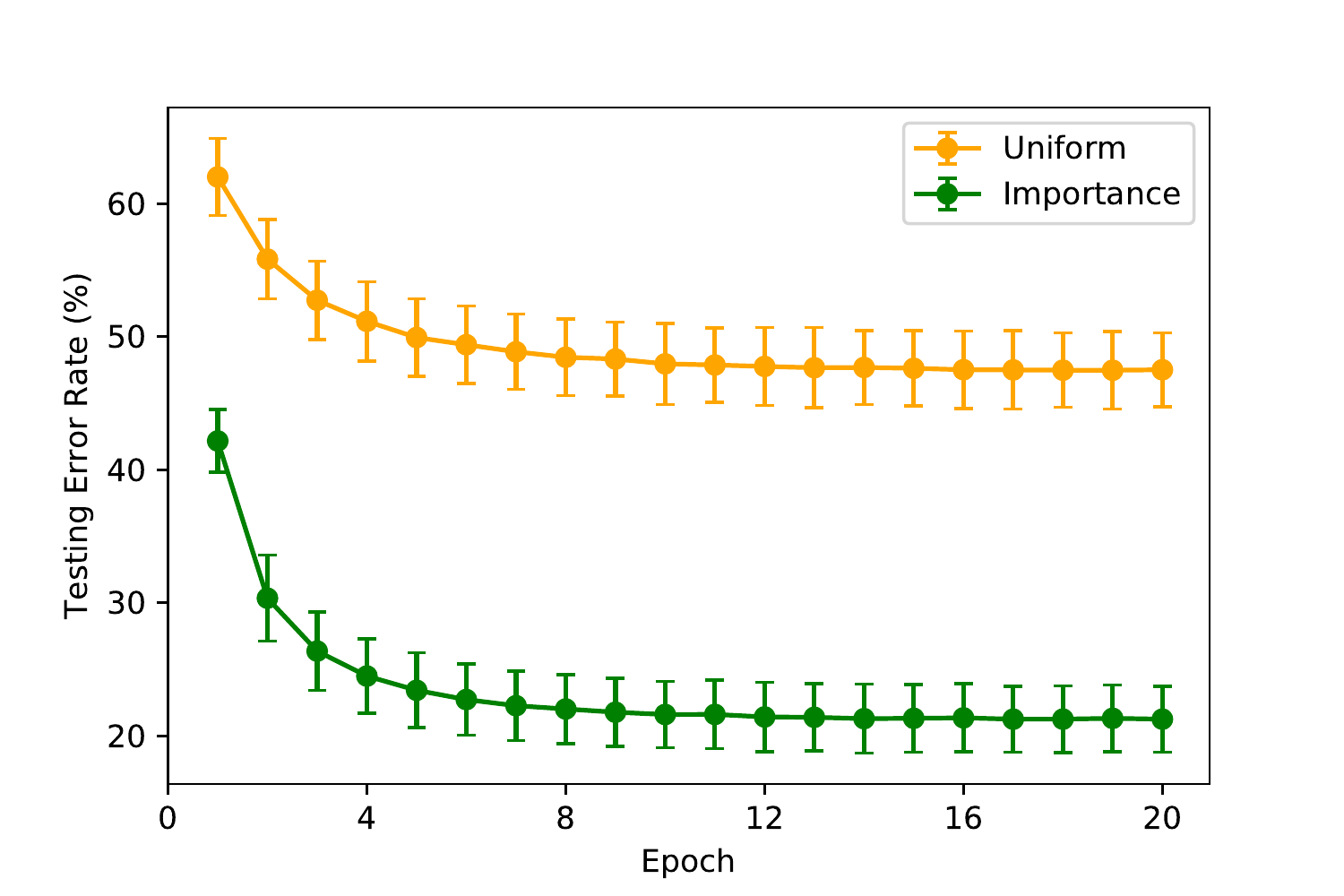}
	\end{center}
	\begin{center}
		\includegraphics[width=0.4\linewidth]{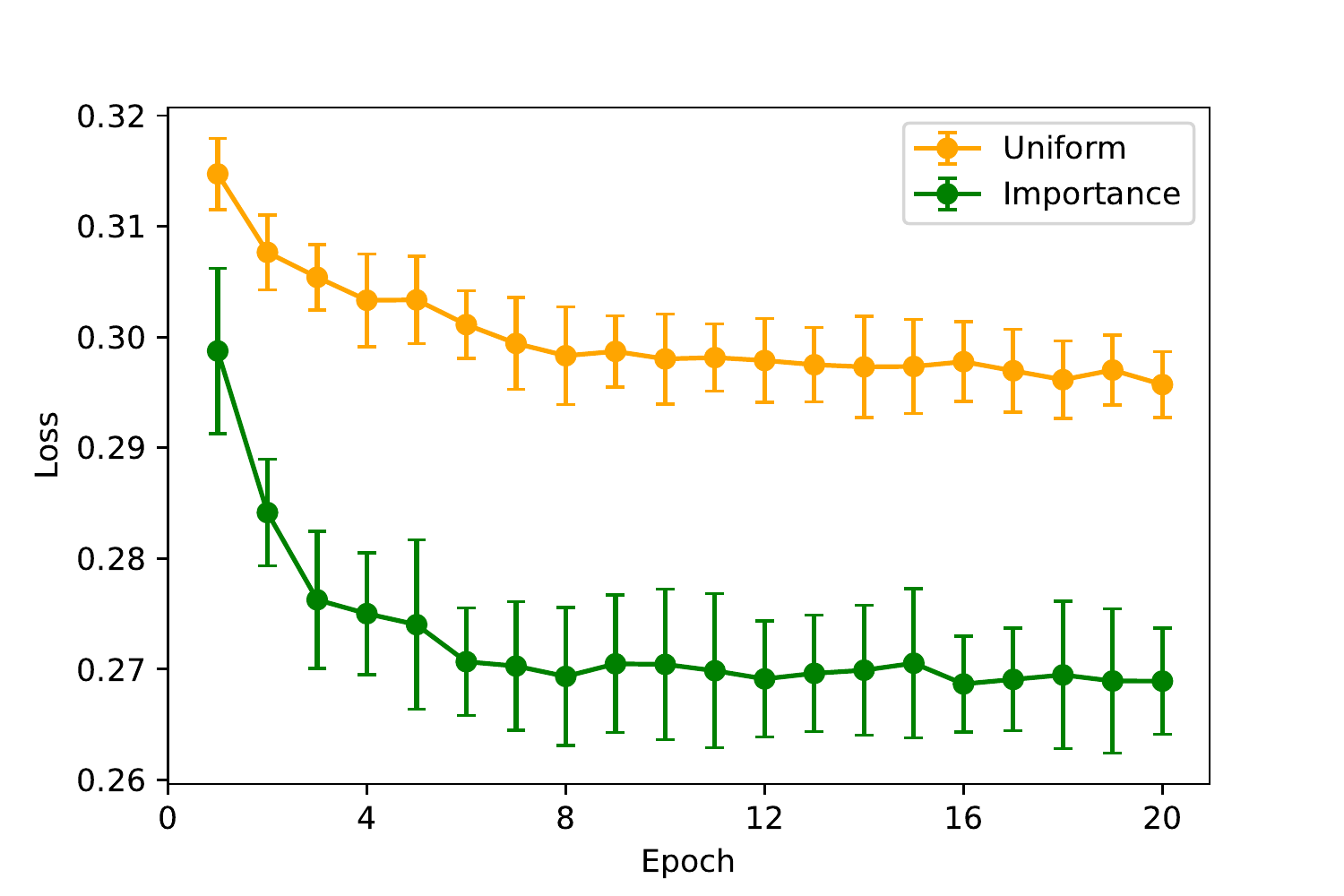}
		\includegraphics[width=0.4\linewidth]{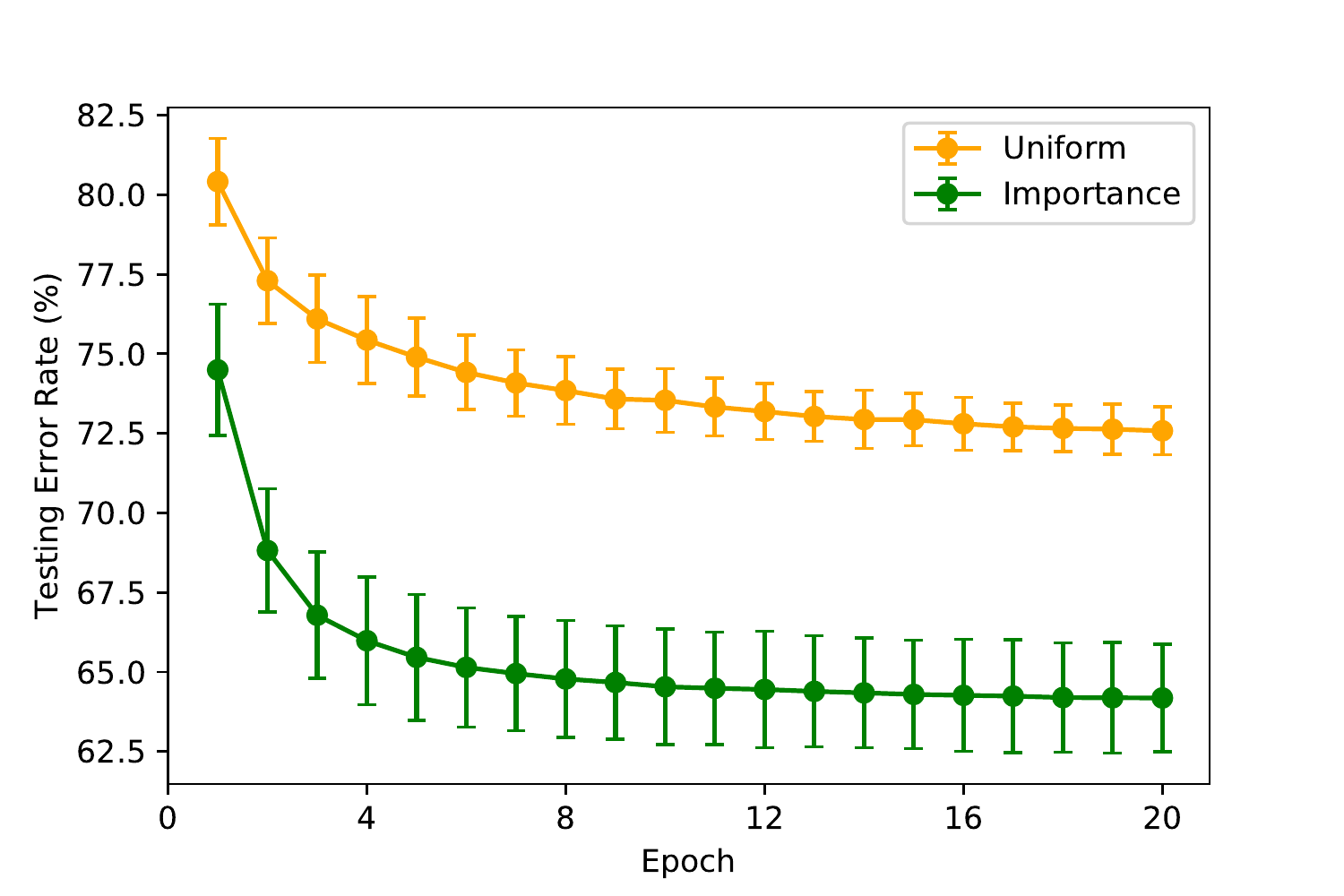}
	\end{center}
	
	\caption{Comparisons between Importance Sampling (Importance) and Uniform Sampling (Uniform) over regularized weight distribution
	on MNIST (The $1_{st}$ row), CIFAR-10 (The $2_{nd}$ row). The $1_{st}$ and the $2_{nd}$ columns plot the training loss and test error rate respectively. The optimization process considers only $u$. The 2-layer network contains 10 hidden nodes. The error bars represent standard deviation.}
	\label{fig:IS}
\end{figure}

Figure \ref{fig:IS_simulated} and \ref{fig:IS} show that features generated by IS outperforms that by US significantly. The gap between IS and US behaves similarly as Figure \ref{fig:empirical}. The reason for this phenomenon is that the regularizer in loss function forces some weights to be near zero, which are very close to random distribution and thus ineffective. Also, this can be derived from \eqref{ccc0}. When $\lambda_2M/\lambda_1$ is large, the bandwidth of $\|\omega_*(\theta)\|=\|\mu_*(\theta)\|/\rho_*(\theta)$ is large when $\theta$ is large.
On the other hand, IS associates an importance score for each weight and ignore the ineffective weights. As a result, the importance sampled result is similar to the well-trained generative model. 

\subsection{Feature initialization: Random versus Generated}\label{feat_init}
Practically, when $\theta$ is sampled from the learned distribution, we can also optimize $u$ and $\theta$ simultaneously. From this view, we can regard the sampled $\theta$ is a better initialization strategy than normal distribution prior. We can use a relative large network, such as $m'=100$ to validate the performance. Figure \ref{fig:empirical_W} shows the results.

\begin{figure}[H]
	\begin{center}
			\includegraphics[width=0.4\linewidth]{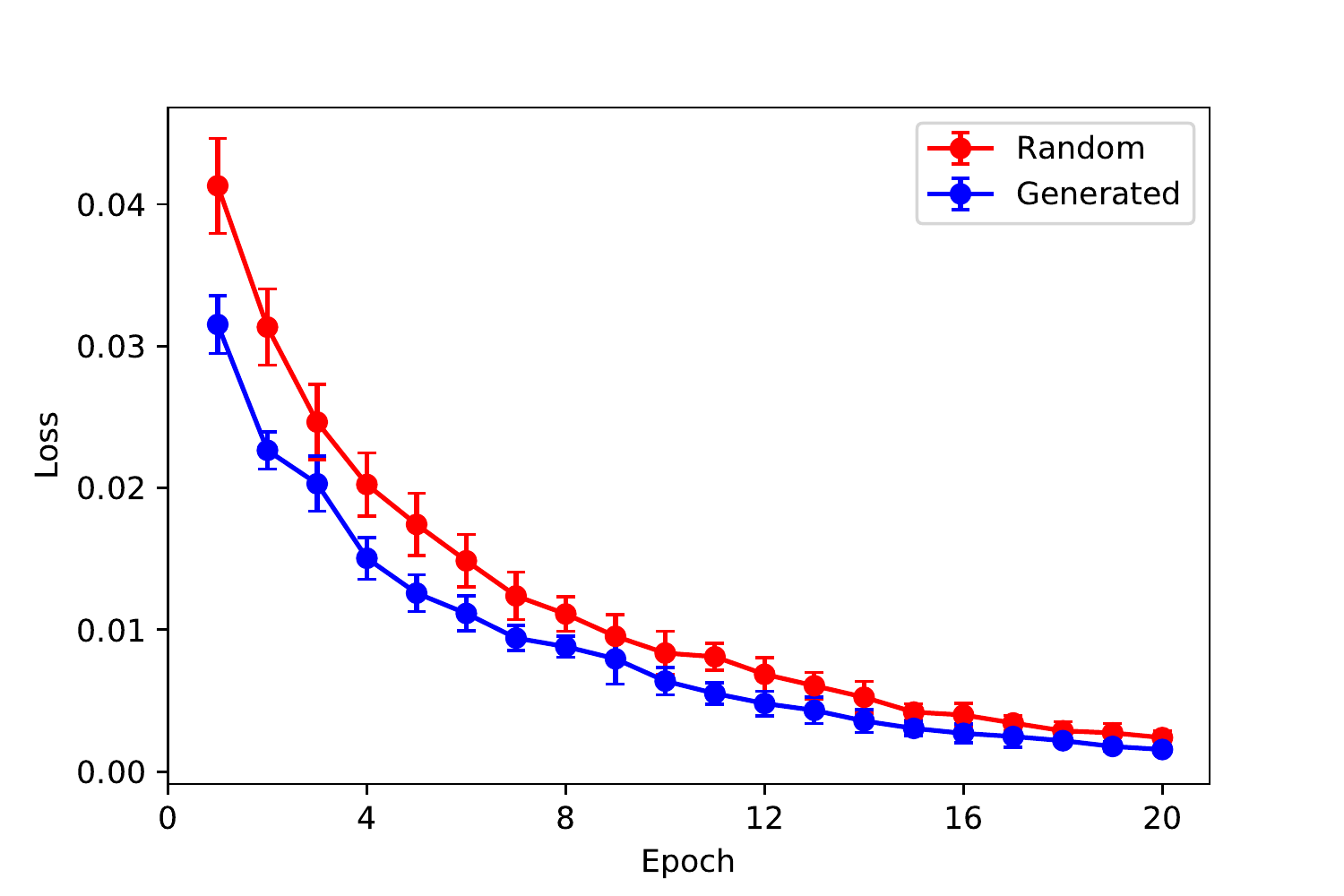}
		\includegraphics[width=0.4\linewidth]{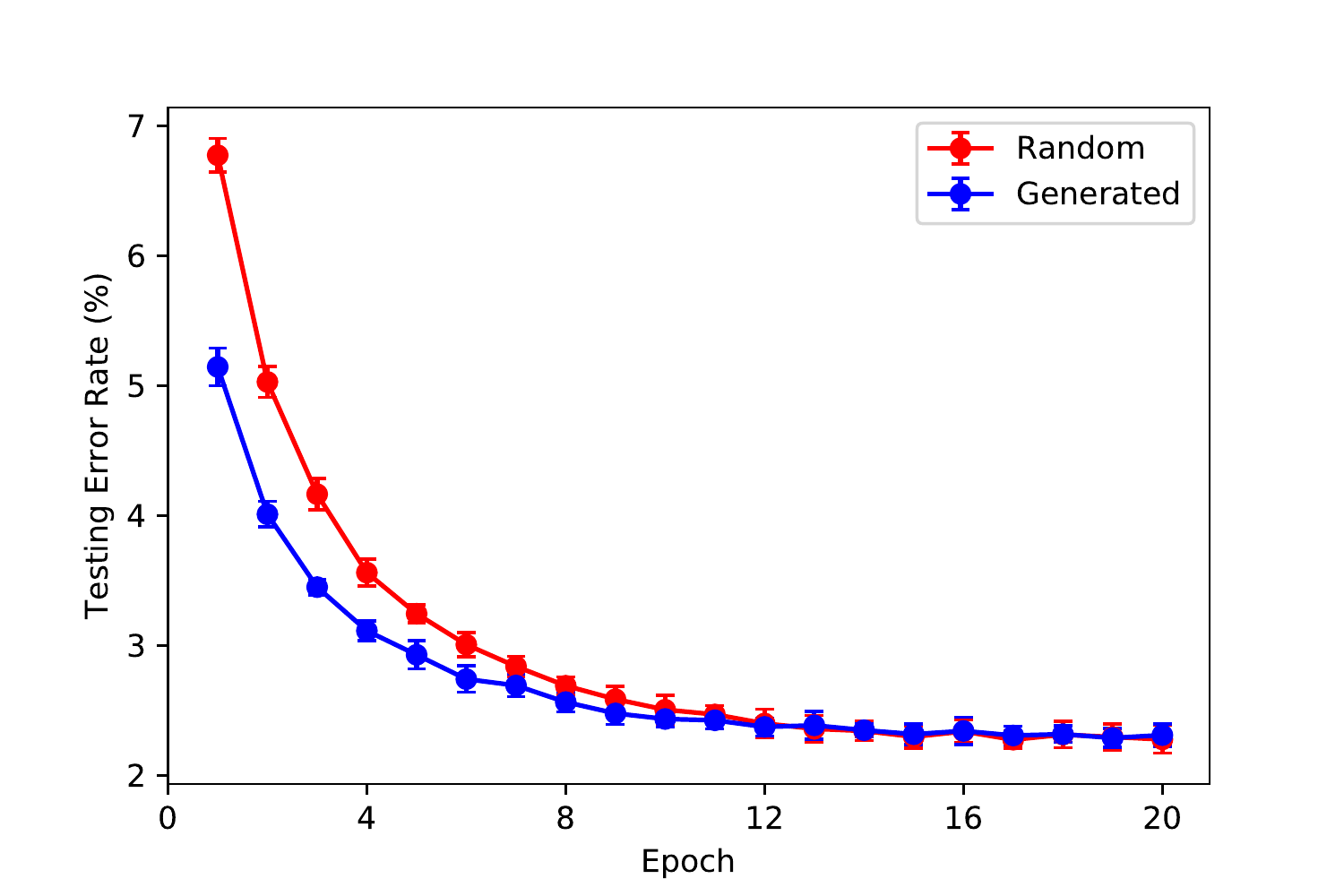}
	\end{center}
	\begin{center}
		\includegraphics[width=0.4\linewidth]{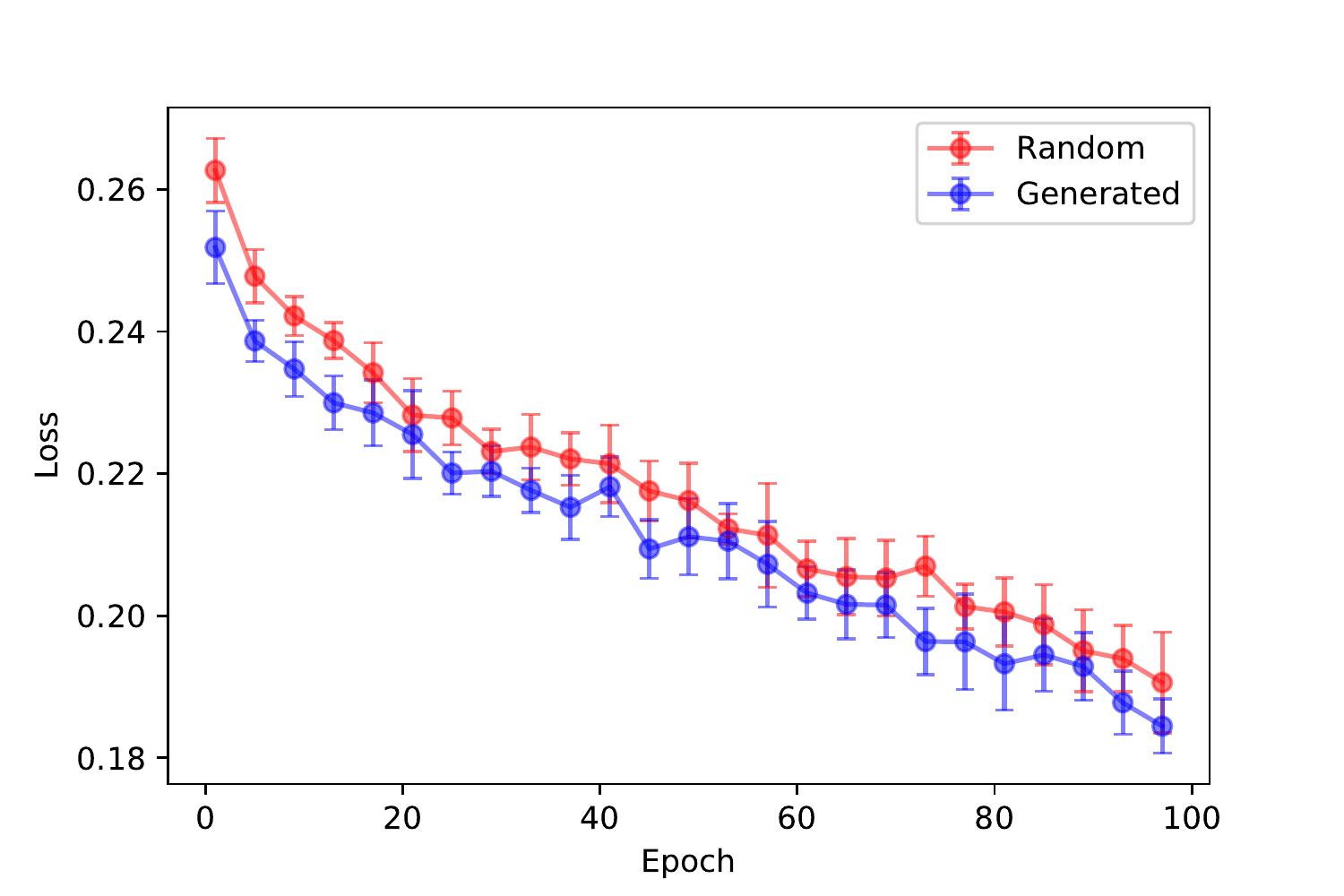}
		\includegraphics[width=0.4\linewidth]{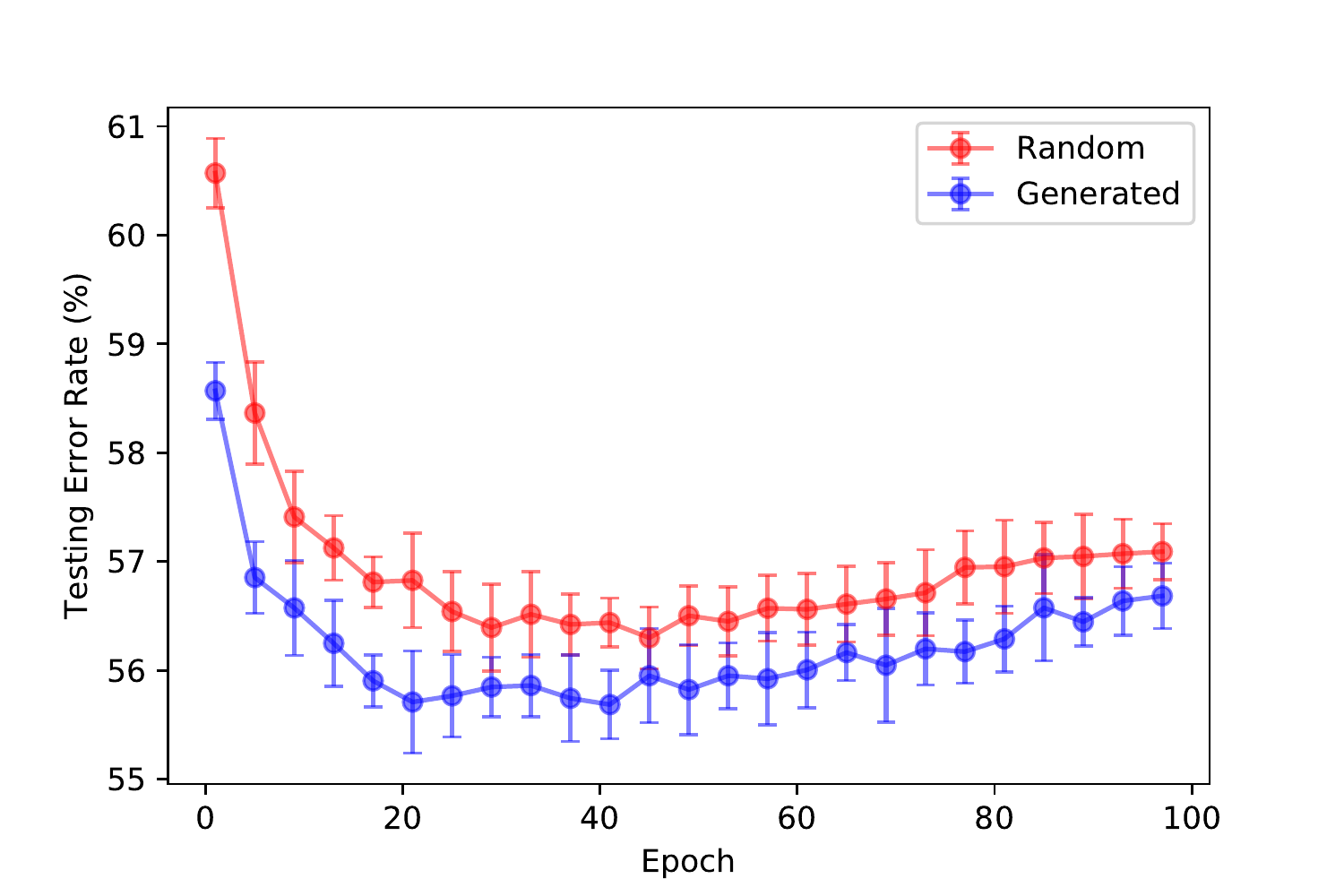}
	\end{center}
	
	\caption{Comparisons between Gaussian initialization (Random) and weight distribution initialization (Generated) on MNIST (The $1_{st}$ row), CIFAR-10 (The $2_{nd}$ row). The $1_{st}$ and the $2_{nd}$ columns plot the training loss and the test error rate respectively. The optimization process considers both $\theta$ AND $u$. The 2-layer net contains 100 hidden nodes. The error bars represent standard deviation.}
	\label{fig:empirical_W}
\end{figure}

Since the initialization strategy by sampling weights from the learned distribution is preferable, according to Figure  \ref{fig:empirical_W},
we can find that, in every epoch, both the training loss and the test error of the sampling approach is smaller than that of random weight initialization. 
As a result, the generalization ability of sampled distribution is reliable.

\subsection{Feature Visualization}\label{feat_visu}
In order to illustrate the superiority of the learned weights, we visualize the output of the first layer. Figure \ref{fig:visualization_feature} shows the results of t-SNE dimensionality reduction \citep{maaten2008visualizing}, it is clear that features of (b) and (c) lead to better separation than that of (a), which means features of (b) and (c) are more discriminative than those of (a).

\begin{figure}[H]
	\begin{center}
		\begin{tabular}{ccc} 
			\includegraphics[width=0.31\linewidth]{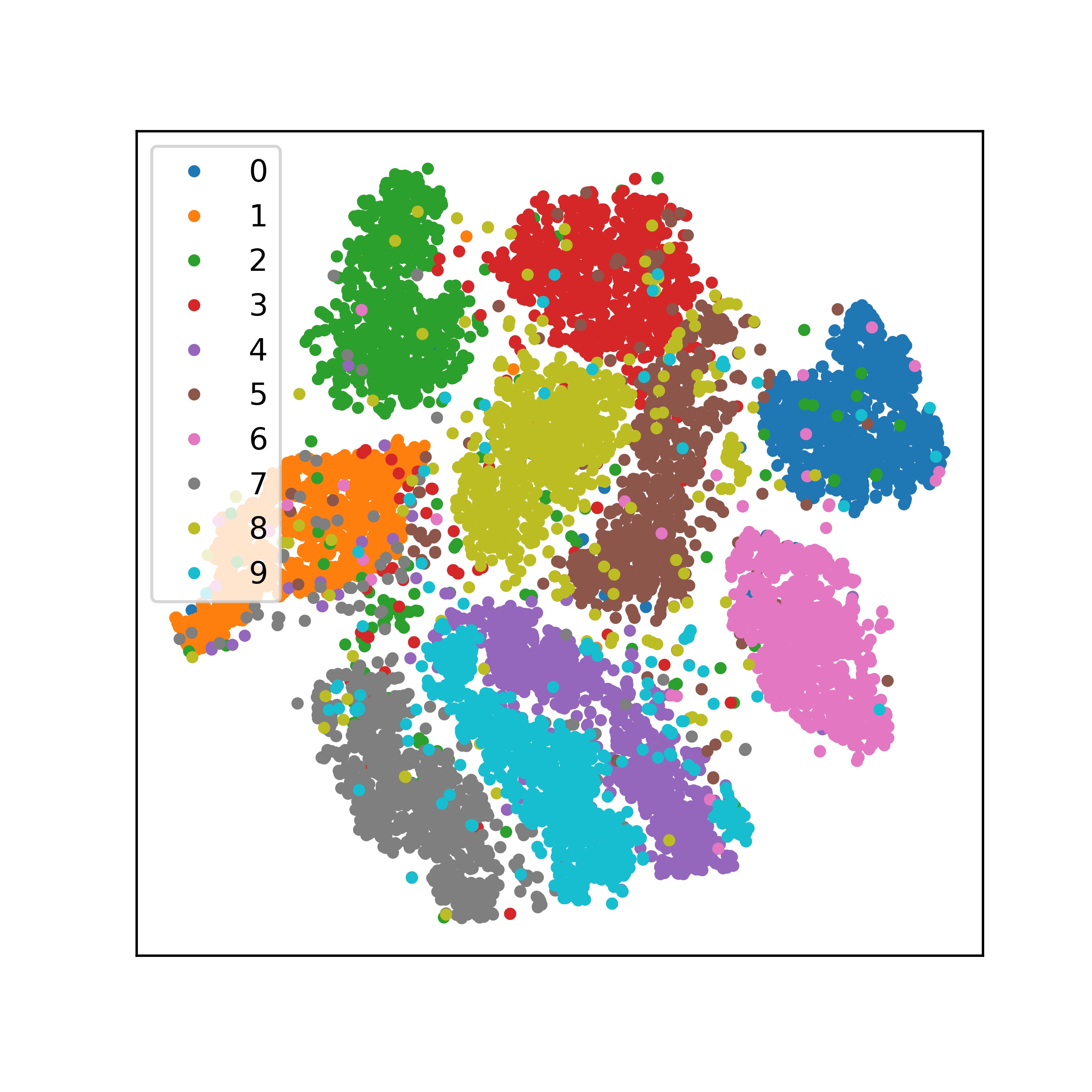}&
			\includegraphics[width=0.31\linewidth]{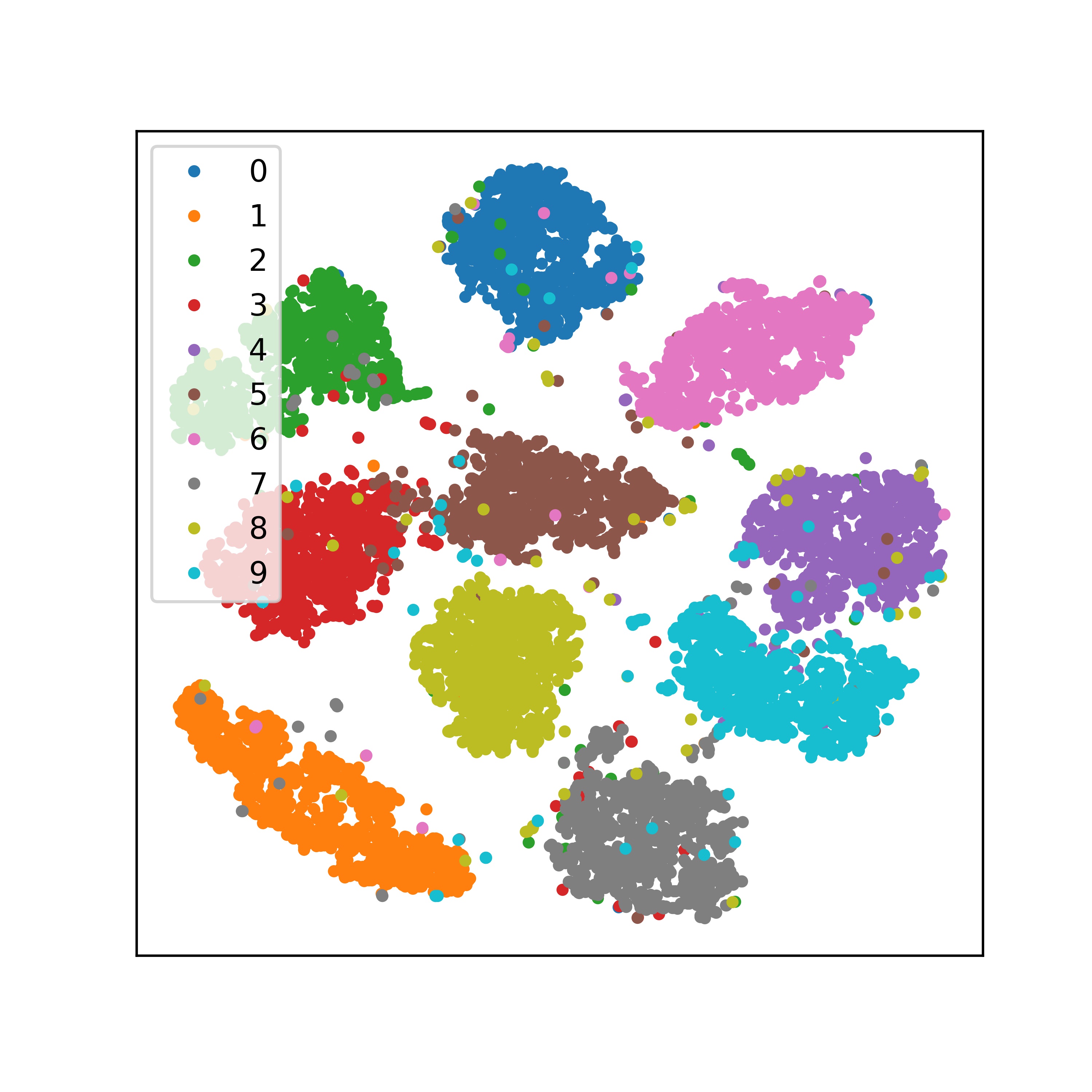}&
			\includegraphics[width=0.31\linewidth]{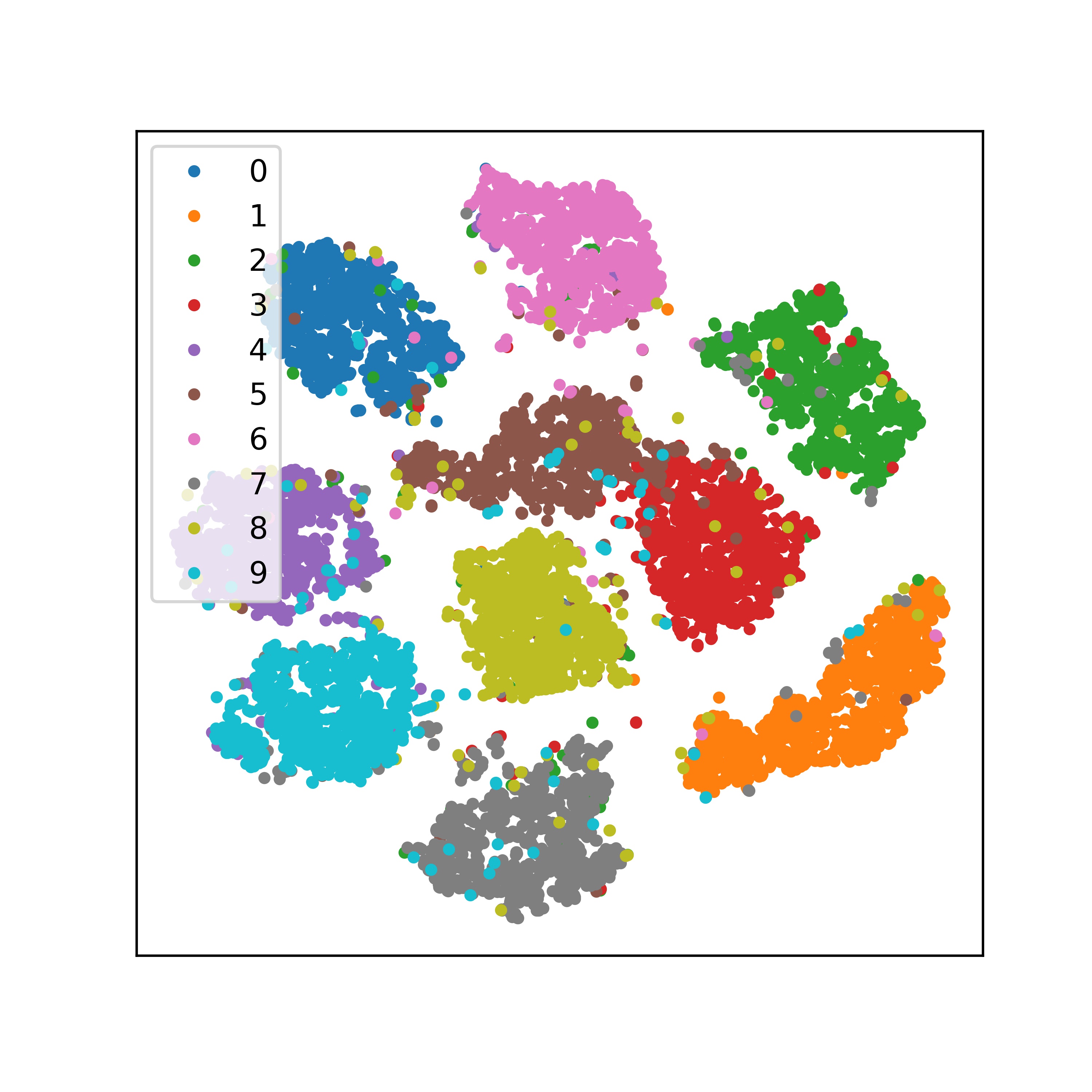}\\
			(a) Random weights & (b) Optimized weights & (c) Generated weights
		\end{tabular}
	\end{center}
	\caption{t-SNE Visualization of MNIST testing features}
	\label{fig:visualization_feature}
\end{figure}

Besides, Figure \ref{fig:visualization_} plots two dimensions of the first layer output and we can notice that the learned weights make some classes linearly separable on the plane. As a result, the second-layer linear classifier leads to a smaller error.

\begin{figure}[H]
	\begin{center}
		\begin{tabular}{ccc} 
			\includegraphics[width=0.31\linewidth]{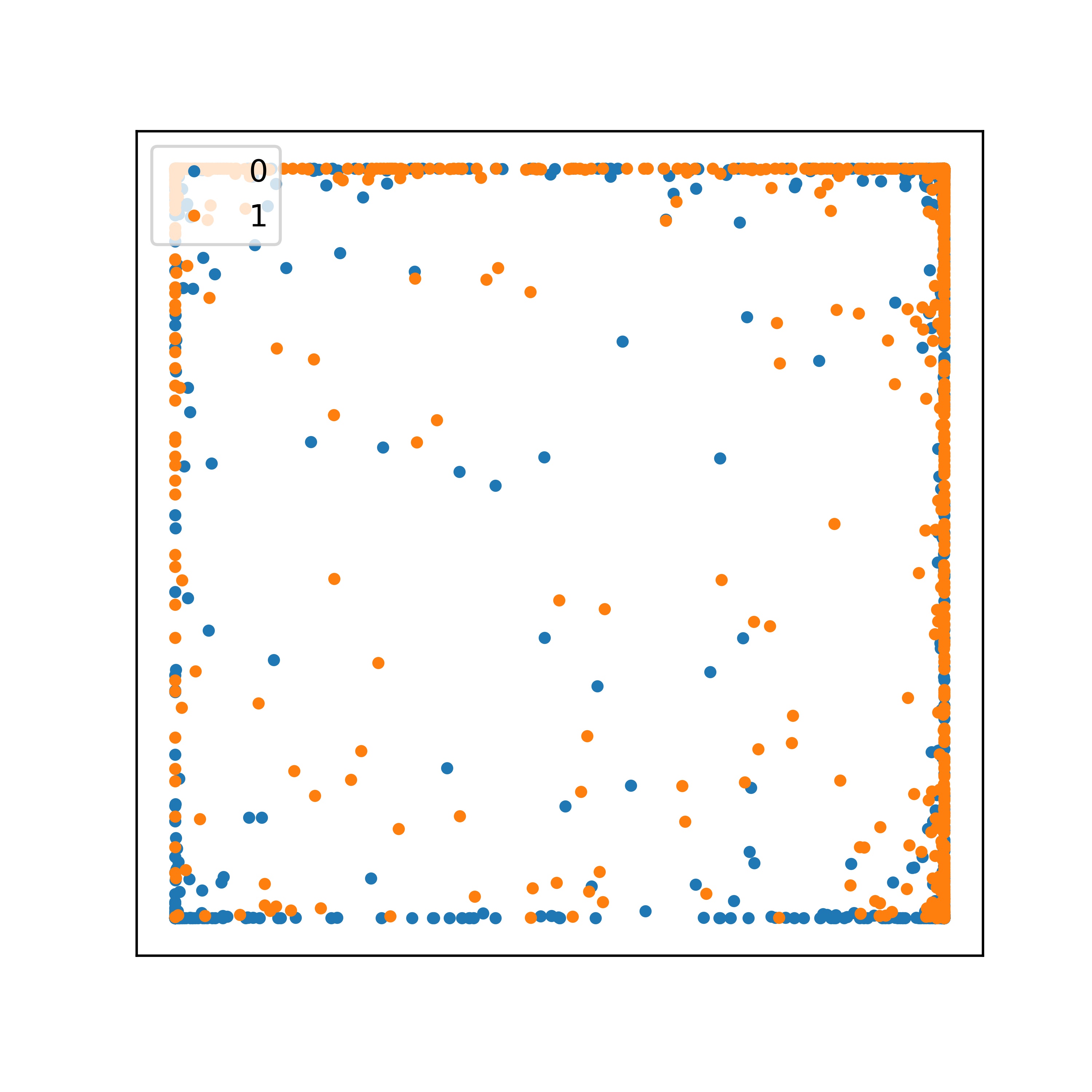}&
			\includegraphics[width=0.31\linewidth]{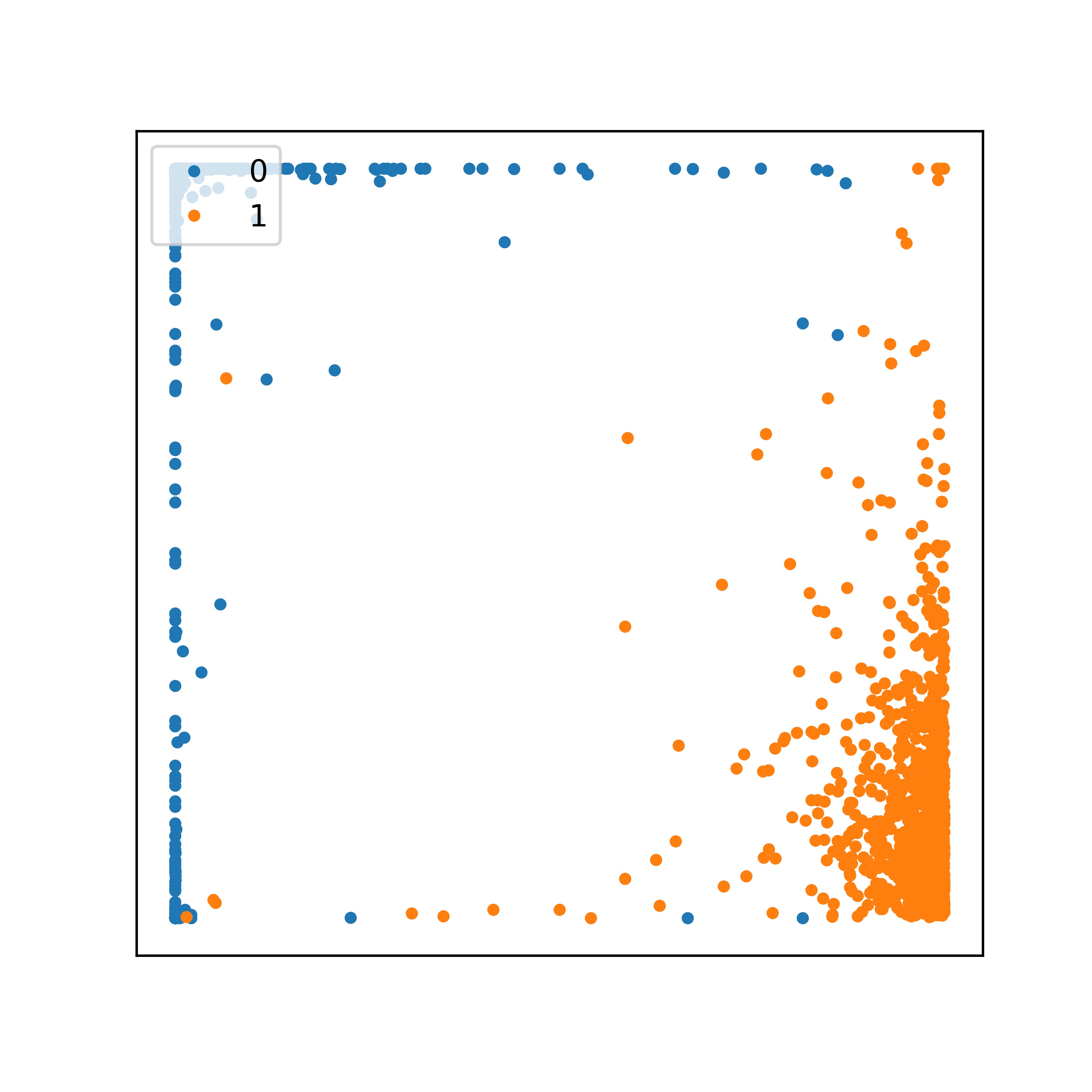}&
			\includegraphics[width=0.31\linewidth]{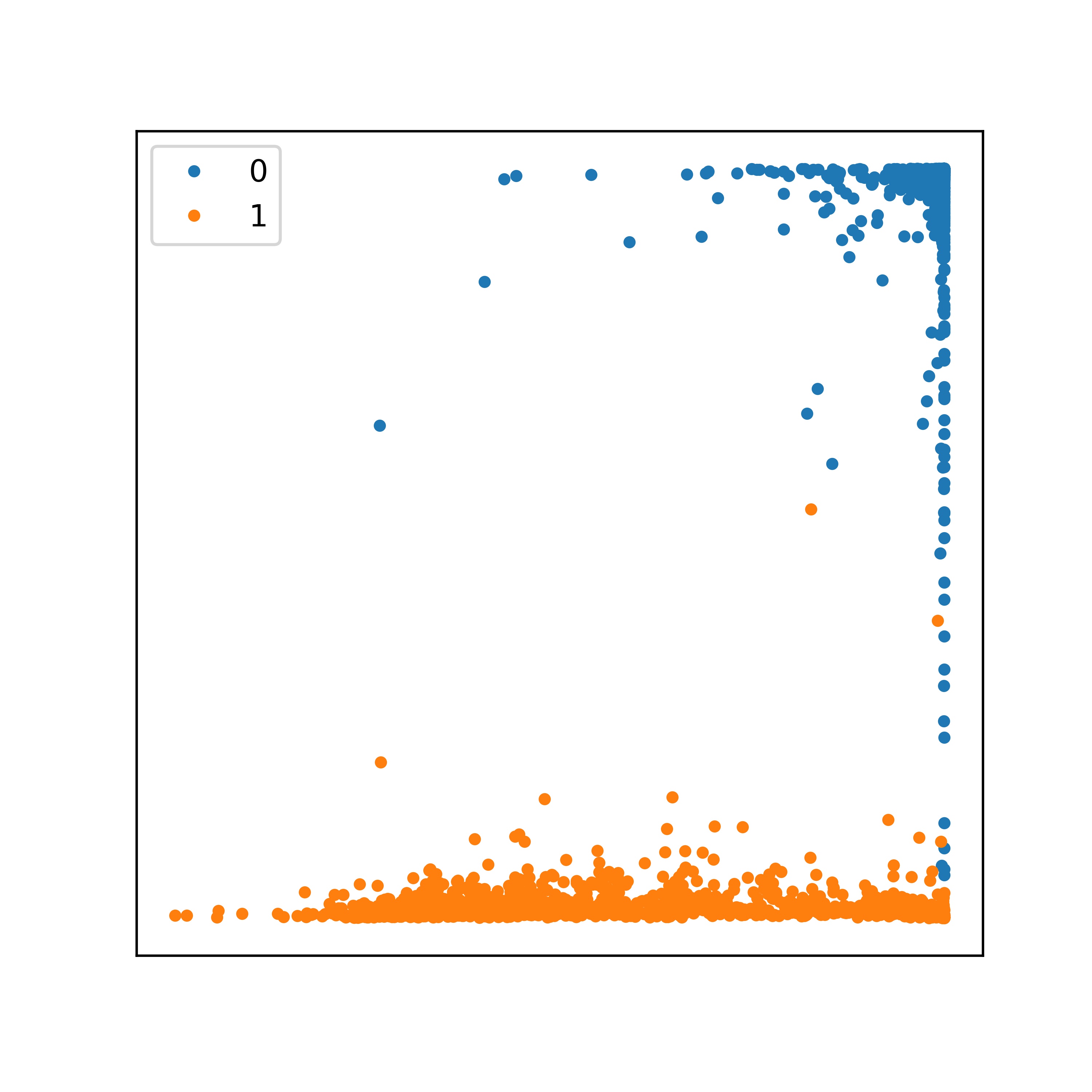}\\
			(a) Random weights & (b) Optimized weights & (c) Sampled weights
		\end{tabular}
	\end{center}
	\caption{Two of MNIST testing features}
	\label{fig:visualization_}
\end{figure}


\end{document}